\documentclass[11pt,reqno]{amsart}
\usepackage{graphicx,verbatim,lineno,titletoc}
\usepackage{amssymb,mathrsfs}
\usepackage{amsmath}
\usepackage{adjustbox}
\usepackage{calc}
\usepackage[numbers, sort]{natbib}
\usepackage{array}
\usepackage{tikz}
\usetikzlibrary{shapes.multipart,patterns,arrows}
\usepackage[font=footnotesize]{caption}
\usepackage[T1]{fontenc} 
\usepackage{fourier} 
\usepackage[english]{babel} 
\usepackage{appendix}
\usepackage[all]{xy}
\usepackage{enumerate}
\usepackage[english]{babel}
\usepackage{multirow}
\usepackage{float}
\usepackage{enumitem}
\usepackage{accents}
\usepackage{hyperref}
\hypersetup{colorlinks}
\usepackage{algorithm}
\usepackage[noend]{algpseudocode}
\usepackage{mathtools}
\usepackage{wrapfig}
\usepackage{caption}
\usepackage{cleveref}  

\usepackage[top=1.1in, left=1.2in, right=1.1in, bottom=1.2in]{geometry}
\hypersetup{colorlinks,linkcolor={red},citecolor={olive},urlcolor={red}}

\newcolumntype{M}[1]{>{\centering\arraybackslash}m{#1}}
\newcolumntype{N}{@{}m{0pt}@{}}

\newtheorem{theorem}{Theorem}[section]
\newtheorem{prop}[theorem]{Proposition}

\newtheorem{lemma}[theorem]{Lemma}

\newtheorem{remark}[theorem]{Remark}

\def\liminf{\mathop{\rm lim\,inf}\limits}

\def\R{\mathbb{R}}

\def\E{\mathbb{E}}
\def\P{\mathbb{P}}

\def\eps{\varepsilon}

\def\X{\mathcal{X}}

\def\A{\mathbf{A}}
\def\B{\mathbf{B}}

\def\param{\mathcal{D}}
\def\Param{\mathcal{C}^{\textup{dict}}}

\newcommand{\dn}[1]{{\textcolor{magenta}{#1 -dn}}}

\newcommand{\tr}{\textup{tr}}

\newcommand{\had}{\odot}

\newcommand{\kr}{\otimes_{kr}}

\DeclareMathOperator{\Out}{\texttt{Out}}

\DeclareMathOperator*{\argmin}{arg\,min}

\DeclareMathOperator{\mat}{\textup{MAT}}

\newlength\myindent
\newtheorem{innercustomassumption}{}
\newenvironment{customassumption}[1]
{\renewcommand\theinnercustomassumption{#1}\innercustomassumption}
{\endinnercustomassumption}

\theoremstyle{definition}

\definecolor{hancolor}{rgb}{0.1, 0.0, 0.9}
\newcommand{\commHL}[1]{{\textcolor{black}{#1}}} 

\allowdisplaybreaks

\makeatletter
\newcommand{\addresseshere}{%
	\enddoc@text\let\enddoc@text\relax
}
\makeatother

\begin{document}
	
	\title[Online nonnegative CP-dictionary learning ]{Online Nonnegative CP-dictionary Learning\\ for Markovian Data}

	\author{Hanbaek Lyu}
	\address{Hanbaek Lyu, Department of Mathematics, University of Wisconsin - Madison WI, 53706, USA}
	\email{\texttt{hlyu@math.wisc.edu}}

	\author{Christopher Strohmeier}
	\address{Christopher Strohmeier, Department of Mathematics, University of California, Los Angeles, CA 90095, USA}
	\email{\texttt{c.strohmeier@math.ucla.edu}}

	\author{Deanna Needell}
	\address{Deanna Needell, Department of Mathematics, University of California, Los Angeles, CA 90095, USA}
	\email{\texttt{deanna@math.ucla.edu}}

	\thanks{The codes for the main algorithm and simulations are provided in \url{https://github.com/HanbaekLyu/OnlineCPDL}}


		\begin{abstract}
		Online Tensor Factorization (OTF) is a  fundamental tool in learning low-dimensional interpretable features from streaming multi-modal data. While various algorithmic and theoretical aspects of OTF have been investigated recently, a general convergence guarantee to stationary points of the objective function without any incoherence or sparsity assumptions is still lacking even for the i.i.d. case. In this work, we introduce a novel algorithm that learns a CANDECOMP/PARAFAC (CP) basis from a given stream of tensor-valued data under general constraints, including nonnegativity constraints that induce interpretability of the learned CP basis. We prove that our algorithm converges almost surely to the set of stationary points of the objective function under the hypothesis that the sequence of data tensors is generated by an underlying Markov chain. Our setting covers the classical i.i.d. case as well as a wide range of application contexts including data streams generated by independent or MCMC sampling. Our result closes a gap between OTF and Online Matrix Factorization in global convergence analysis \commHL{for CP-decompositions}. Experimentally, we show that our algorithm converges much faster than standard algorithms for nonnegative tensor factorization tasks on both synthetic and real-world data. Also, we demonstrate the utility of our algorithm on a diverse set of examples from image, video, and time-series data, illustrating how one may learn qualitatively different CP-dictionaries from the same tensor data by exploiting the tensor structure in multiple ways.
	\end{abstract}
	
	${}$
	\vspace{-0.5cm}
	${}$
	\maketitle
	\section{Introduction}\label{sec:intro}

In modern signal processing applications, there is often a critical need to analyze and understand data that is high-dimensional (many variables), large-scale (many samples), and multi-modal (many attributes). For unimodal (vector-valued) data, \textit{matrix factorization} provides a powerful tool for one to describe data in terms of a linear combination of factors or atoms. In this setting, we have a data matrix $X \in \R^{d \times n}$, and we seek a factorization of $X$ into the product $WH$ for $W \in \R^{d \times R}$ and $H \in \R^{R \times n}$. Including two classical matrix factorization algorithms of Principal Component Analysis (PCA) \cite{wold1987principal} and Nonnegative Matrix Factorization (NMF) \cite{lee1999learning}, this problem has gone by many names over the decades, each with different constraints: dictionary learning, factor analysis, topic modeling, component analysis. It has applications in text analysis, image reconstruction, medical imaging, bioinformatics, and many other scientific fields more generally \cite{sitek2002correction, berry2005email, berry2007algorithms, chen2011phoenix, taslaman2012framework, boutchko2015clustering, ren2018non}.

A \textit{tensor} is a multi-way array that is a natural generalization of a matrix (which is itself a 2-mode tensor) and is suitable for representing multi-modal data. As matrix factorization is for unimodal data, \textit{tensor factorization} (TF) provides a powerful and versatile tool that can extract useful latent information out of multi-modal data tensors. As a result, tensor factorization methods have witnessed increasing popularity and adoption in modern data science. One of the standard tensor factorization paradigms is CANDECOMP/PARAFAC (CP) decomposition \cite{tucker1966some, harshman1970foundations, carroll1970analysis}. In this setting, given a $n$-mode data tensor $\X$, one seeks $n$ \textit{loading matrices} $U^{(1)},\dots,U^{(n)}$, each with $R$ columns, such that $\X$ is approximated by the sum of the outer products of the respective columns of $U_{i}$'s. \commHL{In other words, regarding the $n$-mode tensor $\mathcal{X}$ as the joint probability distribution of $n$ random variables, the CP-decomposition approximates such a joint distribution as the sum of $R$ product distributions, where the columns of the loading matrices give one-dimensional marginal distributions used to form the product distributions.} A particular instance of CP-decomposition is when the data tensor and all of its loading matrices are required to have nonnegative entries. As pointed out in the seminal work of Lee and Seung \cite{lee1999learning} \commHL{(in the matrix case)}, imposing a nonnegativity constraint in \commHL{the} decomposition problem helps one to learn interpretable features from multi-modal data.


\begin{wrapfigure}[14]{r}{0.5\textwidth}
	\vspace{-0.6cm}
	\includegraphics[width=0.5\textwidth]{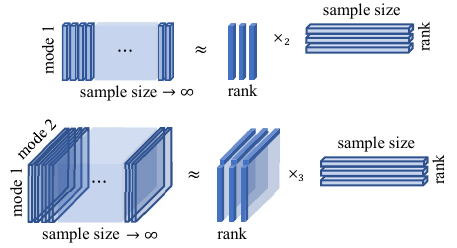}
	\vspace{-0.5cm}
	\caption{Illustration of online MF (top) and online CP-decomposition (bottom). $n$-mode tensors arrive sequentially and past data are not stored. One seeks $n$ loading matrices that give approximate decomposition of all past data.}
	\label{fig:NTF_illustration}
\end{wrapfigure}

Besides being multi-modal, another unavoidable characteristic of modern data is its enormous volume and the rate at which new data are generated. 
\textit{Online learning} algorithms permit incremental processing that overcomes the sample complexity bottleneck inherent to batch processing, which is especially important when storing the entire data set is cumbersome. Not only do online algorithms address capacity and accessibility, but they also have the ability to learn qualitatively different information than offline algorithms for data that admit such a ``sequential" structure (see e.g. \cite{lyu2020applications}). In the literature, many ``online" variants of more classical ``offline" algorithms have been extensively studied --- NMF \cite{mairal2010online, guan2012online, lyu2020online}, TF \cite{zhou2016accelerating, huang2015online, zhou2018online, du2018probabilistic, smith2018streaming}, and dictionary learning \cite{rambhatla2019noodl, arora2015simple, arora2014new, koppel2017d4l}. Online Tensor Factorization (OTF) algorithms with suitable constraints (e.g., nonnegativity) can serve as valuable tools that can extract interpretable features from multi-modal data.

\subsection{Contribution} 

In this work, we develop a novel algorithm and theory for the problem of \textit{online CP-dictionary learning}, \commHL{where the goal is to progressively learn a dictionary of rank-1 tensors (CP-dictionary) from a stream of tensor data.}
\commHL{Namely, given $n$-mode nonnegative tensors $(\X_{t})_{t\ge 0}$, we seek to find an adaptively changing sequence of nonnegative CP-dictionaries such that the current CP-dictionary can approximate all tensor-valued signals in the past as a suitable nonnegative linear combination of its CP-dictionary atoms (see Figure \ref{fig:NTF_illustration} \commHL{in Section \ref{fig:NTF_illustration}}).}  Our framework is flexible enough to handle general situations of an arbitrary number of modes in the tensor data, arbitrary convex constraints in place of the nonnegativity constraint, and a sparse representation of the data using the learned rank-1 tensors. \commHL{In particular, our problem setting includes online nonnegative CP-decomposition.}

Furthermore, we rigorously establish that under mild conditions, our online algorithm produces a sequence of loading matrices that converge almost surely to the set of stationary points of the objective function. In particular, our convergence results hold not only when the sequence of input tensors $(\X_{t})_{t\ge 0}$ are independent and identically distributed (i.i.d.), but also when they form a Markov chain or functions of some underlying Markov chain. Such a theoretical convergence guarantee for online NTF algorithms has not been available even under the i.i.d. assumption on the data sequences. The relaxation to the Markovian setting is particularly useful in practice since often the signals have to be sampled from some complicated or unknown distribution, and obtaining even approximately independent samples is difficult. In this case, the Markov Chain Monte Carlo (MCMC) approach provides a powerful sampling technique (e.g., sampling from the posterior in Bayesian methods \cite{van2018simple} or from the Gibbs measure for Cellular Potts models \cite{voss2012multi}, or motif sampling from sparse graphs \cite{lyu2019sampling}), where consecutive signals can be highly correlated. 

\subsection{Approach}	

Our algorithm combines the Stochastic Majorization-Minimization (SMM) framework \cite{mairal2013stochastic}, which has been used for online NMF algorithms \cite{mairal2010online,guan2012online,zhao2017online,lyu2020online}, and a recent work on block coordinate descent with diminishing radius (BCD-DR)  \cite{lyu2020convergence}. In SMM, one iteratively minimizes a recursively defined surrogate loss function $\hat{f}_{t}$ that majorizes the empirical loss function $f_{t}$. A premise of SMM is that $\hat{f}_{t}$ is convex so that it is easy to minimize, which is the case for online matrix factorization problems in the aforementioned references. However, in the setting of factorizing $n$-mode tensors, $\hat{f}_{t}$ is only convex in each of the $n$ loading matrices and nonconvex jointly in all loading matrices. Our main algorithm (Algorithm \ref{alg:online NTF_highlevel}) only approximately minimizes $\hat{f}_{t}$ by a single round of cyclic block coordinate descent (BCD) in the $n$ loading matrices. This additional layer of relaxation causes a number of technical difficulties in convergence analysis. One of our crucial innovations to handle them is to use a search radius restriction during this process \cite{lyu2020convergence}, which is reminiscent of restricting step sizes in stochastic gradient descent algorithms and is in some sense `dual' to proximal modifications of BCD \cite{grippo2000convergence, xu2013block}. 

Our convergence analysis on dependent data sequences uses the technique of ``conditioning on a distant past", which leverages the fact that while the \commHL{one-step conditional distribution of a Markov chain may be a constant distance away from the stationary} distribution $\pi$, the $N$-step conditional distribution is exponentially close to $\pi$ in $N$. This technique has been developed in \cite{lyu2020online} recently to handle dependence in data streams for online NMF algorithms. 

\subsection{Related work} 
We roughly divide the literature on TF into two classes depending on \textit{structured} or \textit{unstructured} TF problems. The \textit{structured TF problem} concerns recovering exact loading matrices of a  tensor, where a structured tensor decomposition with loading matrices satisfying some incoherence or sparsity conditions is assumed. A number of works address this problem in the offline setting \cite{tang2015guaranteed, anandkumar2015learning, sharan2017orthogonalized, sun2017provable, barak2015dictionary, ma2016polynomial, schramm2017fast}. Recently,  \cite{rambhatla2020provable} addresses an online structured TF problem by reducing it to an online MF problem using sparsity constraints on all but one loading matrices.

On the other hand, \commHL{in the \textit{unstructured TF problem}, one does not make any modeling assumption on the tensor subject to a decomposition so there are no true factors to be discovered. Instead, given an arbitrary tensor, one tries to find a set of factors (matrices or tensors) that gives the best fit of a chosen tensor decomposition model}. In this case, convergence to a globally optimal solution cannot be expected, and global convergence to stationary points of the objective function is desired. For offline problems, global convergence to stationary points of the block coordinate descent method is known to hold under some regularity assumptions on the objective function \cite{grippo2000convergence, grippof1999globally, bertsekas1999nonlinear}. The recent works \cite{zhou2016accelerating, huang2015online, zhou2018online, du2018probabilistic, smith2018streaming} on online TF focus on computational considerations and do not provide a convergence guarantee. For online NMF, almost sure convergence to stationary points of a stochastic majorization-minimization (SMM) algorithm under i.i.d. data assumption is well-known \cite{mairal2010online}, \commHL{which has been recently extended} to the Markovian case in \cite{lyu2020online}. Similar global convergence for online TF is not known even under the i.i.d. assumption. The main difficulty of extending a similar approach to online TF is that the recursively constructed surrogate loss functions are nonconvex and cannot be jointly minimized in all $n$ loading matrices when $n\ge 2$. 

\commHL{There are several recent works improving standard CP-decomposition algorithms such as the alternating least squares (ALS) (see, e.g., \cite{kolda2009tensor}).  \cite{battaglino2018practical} proposes a randomized ALS algorithm, that subsamples rows from each factor matrix update, which is an overdetermined least squares problem. A similar technique of row subsampling was used in the context of high-dimensional online matrix factorization \cite{mensch2017stochastic}. \cite{ma2018randomized} proposed a randomized algorithm for online CP-decomposition but no theoretical analysis was provided. Also, CP-decomposition with structured factor matrices has been investigated in \cite{goulart2015tensor}. On the other hand, \cite{vannieuwenhoven2015computing} considers a more efficient version of gradient descent type algorithms for CP-decomposition. 
}

\commHL{In the context of dictionary learning, there is an interesting body of work considering tensor-dictionary learning based on the Tucker-decomposition \cite{shakeri2016minimax, ghassemi2017stark, shakeri2018identifiability, ghassemi2019learning, shakeri2019sample}. When learning a reconstructive tensor dictionary for tensor-valued data, one can impose additional structural assumptions on the tensor dictionary in order to better exploit the tensor structure of the data and to gain reduced computational complexity. While in this work we consider the CP-decomposition model for the tensor-dictionary (also in an online setting), the aforementioned works consider the Tucker-decomposition instead and obtain various results on sample complexity, identifiability of a Tucker dictionary, and local convergence. }

While our approach largely belongs to the SMM framework, there are related works using stochastic gradient descent (SGD). In \cite{zhao2017online}, an online NMF algorithm based on projected SGD with general divergence in place of the squared $\ell_{2}$-loss is proposed, and convergence to stationary points to the expected loss function for i.i.d. data samples is shown. In \cite{sun2018markov}, a similar convergence result for stochastic gradient descent algorithms for \textit{unconstrained} nonconvex optimization problems with Markovian data samples is shown. While \commHL{none of these} results can be directly applied to our setting of online NTF for Markovian data, it may be possible to develop an SGD based approach for our setting, and it will be interesting to compare the performance of the algorithms based on SMM and SGD. 

\subsection{Organization}

In Section \ref{sec:PF} we first give a background discussion on NTF and CP-decomposition and then state the main optimization problem we address in this paper (see \eqref{eq:def_OCPDL}). In Section \ref{section:main_alg}, we provide the main algorithm (Algorithm \ref{alg:online NTF_highlevel}) and give an overview of the main idea. Section \ref{sec:theory} states the main convergence result in this paper, Theorem \ref{thm:online NTF_convergence}, together with a discussion on necessary assumptions and key lemmas used for the proof. In Section \ref{section:proof} we give the proof of the main result, Theorem \ref{thm:online NTF_convergence}. In Section \ref{section:experimanl_validation}, we compare the performance of our main algorithm on the offline nonnegative CP-decomposition problem against other baseline algorithms -- Alternating Least Squares and Multiplicative Update. We then illustrate our approach on a diverse set of applications in Section \ref{sec:zebras}; these applications are chosen to showcase the advantage of being able to flexibly reshape multi-modal tensor data and learn CP-dictionary atoms for any desired group of modes jointly. 

\commHL{In Appendix \ref{sec:MC_intro}, we provide some additional background on Markov chains and Markov chain Monter Carlo sampling. Appendix \ref{sec:auxiliary_lemmas} contains some auxiliary lemmas. In Appendix \ref{section:statement_alg_bounded_memeory}, we provide a memory-efficient implementation (Algorithm \ref{algorithm:online NTF}) of Algorithm \ref{alg:online NTF_highlevel} that uses bounded memory regardless of the length of the data stream.  }

\subsection{Notation}

For each integer $k\ge 1$, denote $[k]=\{1,2,\dots, k\}$. Fix integers $n, I_{1},\dots,I_{n}\ge 1$. An $n$-mode tensor $\mathbf{X}$ of shape $I_{1}\times \dots \times I_{n}$ is a map $(i_{1}, \dots, i_{n})\mapsto \mathbf{X}(i_{1}, \dots, i_{n})\in \R$ from the multi-index set $[I_{1}]\times \dots \times  [I_{n}]$ into the real line $\R$. We identify 2-mode tensors with matrices and 1-mode tensors with vectors, respectively. We do not distinguish between vectors and columns of matrices. For two real matrices $A$ and $B$, we denote their Frobenius inner product as $\langle A,B \rangle:=\tr(B^{T}A)$ whenever the sizes match. If we have $N$ $n$-mode tensors $\mathbf{X}_{1},\dots,\mathbf{X}_{N}$ of the same shape $I_{1}\times \dots \times I_{n}$, we identify the tuple $[\mathbf{X}_{1},\dots, \mathbf{X}_{N}]$ as the $(n+1)$-mode tensor $\mathcal{X}$ of shape $I_{1}\times \dots \times I_{n}\times N$, whose the $i^{\textup{th}}$ slice along the $(n+1)^{\textup{st}}$ mode equals $\mathbf{X}_{i}$. For given $n$-mode tensors $\A$ and $\B$, denote by $\A\had \B$ and $\A \kr \B$ their Hadamard (pointwise) product and Katri-Rao product, respectively. When $\B$ is a matrix, for each $1\le j \le n$, we also denote their mode-$j$ product by $\A\times_{j} \B$. (See \cite{kolda2009tensor} for an excellent survey of tensor algorithms, albeit with notation that differs from our own). 

\section{Background and problem formulation}\label{sec:PF}

\subsection{CP-dictionary learning and nonnegative tensor factorization}
\label{subsection:CPD_NTF}

Assume that we are given $N$ observed vector-valued signals $x_{1},\dots,x_{N}\in \R_{\ge 0}^{d}$. Fix an integer $R \ge 1$ and consider the following approximate factorization problem (see \eqref{eq:NMF_intro_2} for a precise statement)
\begin{align}\label{eq:NMF_intro_1}
	[x_{1},\dots,x_{N}] \approx [u_{1},\dots,u_{R}] \times_{2} H \qquad \Longleftrightarrow \qquad X\approx UH,
\end{align} 
where $\times_{2}$ denotes the mode-2 product, $X=[x_{1},\dots,x_{N}]\in \R_{\ge 0}^{d\times N}$, $U=[u_{1},\dots,u_{R}]\in \R_{\ge 0}^{d\times R}$, and $H\in \R_{\ge 0}^{R\times N}$. The right hand side \eqref{eq:NMF_intro_1} is the well-known \textit{nonnegative matrix factorization} (NMF) problem, where the use of nonnegativity constraint is crucial in obtaining a ``parts-based" representation of the input signals \cite{lee1999learning}.  Such an approximate factorization learns $R$ \textit{dictionary atoms} $u_{1},\dots,u_{R}$ that together can approximate each observed signal $x_{j}$ by using the nonnegative linear coefficients in the $j^{\text{th}}$ column of $H$. The factors $U$ and $H$ in \eqref{eq:NMF_intro_1} above are called the \textit{dictionary} and \textit{code} of the data matrix $X$, respectively. They can be learned by solving the following optimization problem 
\begin{align}\label{eq:NMF_intro_2}
	\argmin_{U'\in \R_{\ge 0}^{d\times R}, H'\in \R_{\ge 0}^{R\times N} } \left( \lVert X - U'H' \rVert_{F}^{2} + \lambda \lVert H' \rVert_{1} \right),
\end{align}
where $\lambda\ge 0$ is a \commHL{regularization parameter} that encourages a sparse representation of the columns of $X$ over the columns of $U$. Note that  \eqref{eq:NMF_intro_2} is also known as a \textit{dictionary learning problem}
\cite{olshausen1997sparse, engan1999frame, lewicki2000learning, elad2006image, lee2005acquiring}, especially when $R \ge d$.  

Next, suppose we have $N$ observed $n$-mode tensor-valued signals $\mathbf{X}_{1},\dots,\mathbf{X}_{N}\in \R_{\ge 0}^{I_{1}\times \dots \times I_{n}}$. A direct tensor analogue of the NMF problem \eqref{eq:NMF_intro_1} would be the following:
\begin{align}\label{eq:NTF_intro_1}
	[\mathbf{X}_{1},\dots, \mathbf{X}_{N}] \approx [\mathbf{D}_{1},\dots, \mathbf{D}_{R}]\times_{n+1} H \qquad \Longleftrightarrow \qquad \mathcal{X}\approx \mathcal{D}\times_{n+1} H,
\end{align}
where $\mathcal{X}=[\mathbf{X}_{1},\dots, \mathbf{X}_{N}]\in \mathbb{R}_{\ge 0}^{I_{1}\times \dots \times I_{n} \times N}$, $\mathcal{D}=[\mathbf{D}_{1},\dots, \mathbf{D}_{R}]\in \mathbb{R}_{\ge 0}^{I_{1}\times \dots \times I_{n} \times R}$, and $H\in \mathbb{R}_{\ge 0}^{R\times N}$. As before, we call $\mathcal{D}$ and $H$ above the dictionary and code of the data tensor $\mathcal{X}$, respectively. Note that this problem is equivalent to \eqref{eq:NMF_intro_1} since 
\begin{align}
	\lVert \mathcal{X}- \mathcal{D}\times_{n+1} H \rVert_{F}^{2} =  \lVert \mat(\mathcal{X})- \mat(\mathcal{D})\times_{\commHL{2}} H \rVert_{F}^{2},
\end{align}
where $\mat(\cdot)$ is the matricization operator that vectorizes  (using lexicographic ordering of entries) each slice with respect to the last mode. For instance, $\mat([\mathbf{X}_{1},\dots, \mathbf{X}_{N}])$ is a $(I_{1}\cdots I_{n})\times N$ matrix whose $i^{\textup{th}}$ column is the vectorization of $\mathbf{X}_{i}$. 

Now, consider imposing an additional structural constraint on the dictionary atoms $\mathbf{D}_{1},\dots, \mathbf{D}_{N}$ in  \eqref{eq:NTF_intro_1}. Specifically, suppose we want each $\mathbf{D}_{i}$ to be \commHL{the sum of $R$ rank 1 tensors}. Equivalently, we assume that there exist \textit{loading matrices} $[U^{(1)},\dots,U^{(n)}]\in \R_{\ge 0}^{I_{1}\times R} \times \dots \times \R_{\ge 0}^{I_{n}\times R}$ such that 
\begin{align}\label{eq:def_OUT}
	[\mathbf{D}_{1},\dots, \mathbf{D}_{R}] &= \Out(U^{(1)},\dots,U^{(n)}) \\
	&:= \left[\bigotimes_{k=1}^{n} U^{(k)}(:,1),\, \bigotimes_{k=1}^{n} U^{(k)}(:,2),\, \dots \,, \bigotimes_{k=1}^{n} U^{(k)}(:,R)  \right] \in \R_{\ge 0}^{I_{1}\times \dots \times I_{n} \times R},
\end{align} 
where $U^{(k)}(:,j)$ denotes the $i^{\textup{th}}$ column of the $I_{k}\times R$ matrix $U^{(k)}$ and $\otimes$ denotes the outer product. Note that we are also defining the operator $\Out(\cdot)$ here, which will be used throughout this paper. 
In this case, the tensor factorization problem in \eqref{eq:NTF_intro_1} becomes (a more precise statement is given in \eqref{eq:def_OCPDL})
\begin{align}\label{eq:NTF_intro_2}
	[\mathbf{X}_{1},\dots, \mathbf{X}_{N}] \approx \Out(U^{(1)},\dots,U^{(n)})\times_{n+1} H.
\end{align}
When $N=1$ and $\lambda=0$, then $H\in \mathbb{R}_{\ge 0}^{R\times 1}$, so by absorbing the $i^{\textup{th}}$ entry of $H$ into the $\mathbf{D}_{i}$, we see that the above problem \eqref{eq:NTF_intro_2} reduces to 
\begin{align}\label{eq:NTF_intro_4}
	\mathbf{X}\approx \sum \Out(U^{(1)},\dots,U^{(n)}) := \sum_{i=1}^{R} \bigotimes_{k=1}^{n} U^{(k)}(:,i),
\end{align}
which is the nonnegative CANDECOMP/PARAFAC (CP) decomposition  problem \cite{tucker1966some, harshman1970foundations, carroll1970analysis}.
On the other hand, if $n=1$ so that $\mathbf{X}_{i}$ are vector-valued signals, then \eqref{eq:NTF_intro_2} reduces to the classical dictionary learning problem \eqref{eq:NMF_intro_2}. For these reasons, we refer to  \eqref{eq:NTF_intro_2} as the \textit{CP-dictionary learning} (CPDL) problem. We call the $(n+1)$-mode tensor  $\Out(U^{(1)},\dots,U^{(n)})=[\mathbf{D}_{1},\dots,\mathbf{D}_{R}]$ in $\R^{(I_{1}\times \dots \times I_{n}\times R)}$ a \textit{CP-dictionary} and the matrix $H\in \R_{\ge 0}^{R\times N}$ the \textit{code} of the dataset $\mathcal{X}=[\mathbf{X}_{1},\dots,\mathbf{X}_{N}]$, respectively. Here we call the rank-1 tensors $\mathbf{D}_{i}$ the \textit{atoms} of the CP-dictionary. 

\subsection{Online CP-dictionary learning}\label{sec:22}

Next, we consider an \textit{online} version of the CPDL problem we considered in \eqref{eq:NTF_intro_2}. \commHL{Given a continuously arriving sequence of data tensors $(\mathbf{X}_{t})_{t\ge 0}$, can we find an adaptively changing sequence of CP-dictionaries such that the current CP-dictionary can approximate all tensor-valued signals in the past as a suitable nonnegative linear combination of its CP-dictionary atoms (see Figure \ref{fig:NTF_illustration} {in Section \ref{fig:NTF_illustration}})? This online problem can be explicitly formulated as an empirical loss minimization framework, and we will also state an equivalent stochastic program (under some modeling assumption) that we rigorously address.}  

\commHL{
	Fix constraint sets for code and loading matrices $\mathcal{C}^{\textup{code}}\subseteq \R^{R\times b}$ and $\mathcal{C}^{(i)}\subseteq \R^{I_{i}\times R}$, $i=1,\dots,n$, respectively (generalizing the nonnegativity constraints in Subsection \ref{subsection:CPD_NTF}). Write  $\mathcal{C}^{\textup{dict}}:=\mathcal{C}^{(1)}\times \cdots \times \mathcal{C}^{(n)}$. For each $\mathcal{X}\in\R_{\ge 0}^{I_{1}\times \dots\times I_{n}\times b}$, $\mathcal{D}:=[U^{(1)},\dots,U^{(n)}]\in \R^{I_{1}\times R}\times \dots \times \R^{I_{n}\times R}$, $H\in \R^{R\times b}$, define  
	\begin{align}
		\ell(\mathcal{X},\mathcal{D},H) &:= \lVert 
		\mathcal{X} - \Out(\mathcal{D}) \times_{n+1} H
		\rVert_{F}^{2} + \lambda \lVert H \rVert_{1},  \label{eq:def_loss_full} \\
		\ell(\mathcal{X},\mathcal{D}) &:= \inf_{H\in \mathcal{C}^{\textup{code}}}\,\, \ell(\mathcal{X},\mathcal{D},H),\label{eq:def_loss}
	\end{align}
	where $\lambda\ge 0$  is a regularization parameter. Fix a sequence of non-increasing weights $(w_{t})_{t\ge 0}$ in $(0,1]$. Here $\mathcal{X}$ denotes a minibatch of $b$ tensors in $\R^{I_{1}\times \dots \times I_{n}}$, so minimizing $\ell(\mathcal{X}, \mathcal{D})$ with respect to $\mathcal{D}$ amounts to fitting the CP-dictionary $\mathcal{D}$ to the minibatch of $b$ tensors in $\mathcal{X}$. }

\commHL{The \textit{online CP-dictionary learning} (online CPDL) problem is the following empirical loss minimization problem: 
	\begin{align}\label{eq:def_OCPDL_ELM}
		\hspace{-1cm} \textup{Upon arrival of $\mathcal{X}_{t}$:}\quad	\mathcal{D}_{t}\in \argmin_{ \mathcal{D} \in \mathcal{C}^{\textup{dict}}} \big( 	f_{t}(\mathcal{D}) := (1-w_{t}) f_{t-1}(\mathcal{D}) + w_{t} \, \ell(\mathcal{X}_{t},\mathcal{D}) \big),
	\end{align}
	where $f_{t}$ is the \textit{empirical loss function} recursively defined by the weighted average  in \eqref{eq:def_OCPDL_ELM} with $f_{0}\equiv 0$.  One can solve the recursion in \eqref{eq:def_OCPDL_ELM} and obtain the more explicit formula for the empirical loss:  
	\begin{align}\label{eq:def_OCPDL_ELM_explicit}
		f_{t}(\mathcal{D}) = \sum_{k=1}^{t} \ell(\X_{k},\mathcal{D}) \, w^{t}_{k}, \qquad w^{t}_{k}:=w_{k}\prod_{i=k+1}^{t} (1-w_{i}).
\end{align}}

\commHL{\noindent The weight $w_{t}$  in \eqref{eq:def_OCPDL_ELM} controls how much we want our new loading matrices in $\mathcal{D}_{t}$ to deviate from minimizing the previous empirical loss $f_{t-1}$ to adapting to the newly observed tensor data $\mathcal{X}_{t}$. In one extreme case of $w_{t}\equiv 1$, $\mathcal{D}_{t}$ is a minimizer of the time-$t$ loss $\ell(\mathcal{X}_{t},\cdot)$ and ignores the past $f_{t-1}$. If $w_{t}\equiv \alpha\in (0,1)$ then the history is forgotten exponentially fast, that is, $f_{t}(\cdot) = \sum_{s=1}^{t} \alpha(1-\alpha)^{t-s}\, \ell(\mathcal{X}_{s},\cdot)$. On the other hand, the `balanced weight' $w_{t}=1/t$ makes the empirical loss to be the arithmetic mean: $f_{t}(\cdot) = \frac{1}{t}\sum_{s=1}^{t}  \ell(\mathcal{X}_{s},\cdot)$, which is the choice made for the online NMF problem in \cite{mairal2010online}. Therefore, one can choose the sequence of weights $(w_{t})_{t\ge 1}$ in Algorithm \ref{alg:online NTF_highlevel} in a desired way to control the sensitivity of the algorithm to the newly observed data. That is, make the weights decay fast for learning average features and decay slow (or keep it constant) for learning trending features. We mention that our theoretical convergence analysis covers only the former case. }

\commHL{We note that the online CPDL problem \eqref{eq:def_OCPDL_ELM} involves solving a constrained optimization problem for each $t$, which is practically infeasible. Hence we may compute a sub-optimal sequence $(\mathcal{D}_{t})_{t\ge 0}$ of tuples of loading matrices (see Algorithm \ref{alg:online NTF_highlevel}) and assess its asymptotic fitness to the original problem \eqref{eq:def_OCPDL_ELM}. We seek to perform some rigorous theoretical analysis at the expense of some suitable but non-restrictive assumption on the data sequence $\mathcal{X}_{t}$ as well as the weight sequence $(w_{t})_{t\ge 1}$. A standard modeling assumption in the literature is to assume the data sequence $\mathcal{X}_{t}$ are independent and identically distributed (i.i.d.) according to some distribution $\pi$ \cite{mairal2010online,mairal2013stochastic,mensch2017stochastic,zhao2017online}. We consider a more relaxed setting where $\mathcal{X}_{t}$ is given as a function of some underlying Markov chain (see \ref{assumption:A1}) and $\pi$ is the stationary distribution of $(\mathcal{X}_{t})_{t\ge 1}$ viewed as a stochastic process. Under this assumption, consider the following stochastic program 
	\begin{align}\label{eq:def_OCPDL}
		\argmin_{\mathcal{D} \in \mathcal{C}^{\textup{dict}}} \big(f(\mathcal{D}):=\mathbb{E}_{\mathcal{X}\sim \pi}\left[ \ell(\mathcal{X}, \mathcal{D})
		\right]\big),
	\end{align}
	where the \textit{random} tensor $\mathcal{X}$ is sampled from the distribution $\pi$, and we call the function $f$ defined in \eqref{eq:def_OCPDL} the \textit{expected loss function}. The connection between the online  \eqref{eq:def_OCPDL_ELM} and the stochastic  \eqref{eq:def_OCPDL} formulation of the CPDL problem  is that, if the parameter space $\mathcal{C}^{\textup{dict}}$ is compact and the weights $w_{t}$ satisfy some condition, then $\sup_{\mathcal{D}} |f_{t}(\mathcal{D}) - f(\mathcal{D})|\rightarrow 0$ almost surely as $t\rightarrow \infty$. (see Lemma \ref{lem:uniform_convergence_asymmetric_weights} \commHL{in Appendix \ref{sec:auxiliary_lemmas}}). Hence, under this setting, we seek to find a sequence $(\mathcal{D}_{t})_{t\ge 1}$ that converges to a solution to \eqref{eq:def_OCPDL}. In other words, fitness to the \textit{single} expected loss function $f$ is enough to deduce the asymptotic fitness to \textit{all} empirical loss functions $f_{t}$. }

\commHL{Once we find an optimal CP-dictionary $\mathcal{D}^{*}=\Out(U^{(1)},\dots,U^{(n)})$ for \eqref{eq:def_OCPDL}, then we can obtain an optimal code matrix $H=H(\mathcal{X})\in \R^{R\times b}$ for each realization of the random tensor $\mathcal{X}$ by solving the convex problem in \eqref{eq:def_loss}. Demanding optimality of the CP-dictionary $\mathcal{D}^{*}$ and leaving the code matrix $H=H(\mathcal{X})$ adjustable in this way is a more flexible framework for CP-decomposition of a random (as well as online) tensor than seeking a pair of jointly optimal CP-dictionary $\mathcal{D}^{*}$ and code matrix $H^{*}$, especially when the variation of the random tensor is large. However, if we have a single deterministic tensor $\mathcal{X}$ to be factorized, then these two formulations are equivalent since $\min_{\mathcal{D},\mathcal{H}}\ell(\mathcal{X},\mathcal{D},H) = \min_{\mathcal{D}}\min_{H} \ell(\mathcal{X},\mathcal{D},H)$. 
}

\section{The Online CP-Dictionary Learning algorithm}\label{section:main_alg}

In this section, we state our main algorithm (Algorithm \ref{alg:online NTF_highlevel}). For simplicity, we first give a preliminary version for the case of $n=2$ modes, minibatch size $b=1$, with nonnegativity constraints. Suppose we have learned $n=2$ loading matrices $\mathcal{D}_{t-1}:=[U_{t-1}^{(1)}, U_{t-1}^{(2)}]$ from the sequence $\X_{1},\dots,\X_{t-1}$ of data tensors, $\X_{i}\in \R^{I_{1}\times I_{2}\times 1}$. Then we compute the updated loading matrices $\mathcal{D}_{t}=[ U^{(1)}_{t}, U^{(2)}_{t} ]$ by
\begin{align}\label{alg:online NTF_special}
	\begin{cases}
		H_t &\leftarrow {\color{black}\underset{H\in \R_{\ge 0}^{R\times 1}}{\argmin} \ell(\X_{t}, \mathcal{D}_{t-1}, H) } \\
		\hat{f}_{t}(\mathcal{D}) &\leftarrow  (1-w_{t}) \hat{f}_{t-1}(\mathcal{D}) + w_{t}  \ell(\X_{t}, \mathcal{D}, H_{t}) \\[5pt]
		U_{t}^{(1)} &\leftarrow  \underset{U\in \R_{\ge 0}^{I_{i}\times R},\, \lVert U - U_{t-1}^{(1)}\rVert_{F}\le c'w_{t}}{\argmin} \,\,  \hat{f}_{t}(U,U_{t-1}^{(2)})\\
		U_{t}^{(2)} &\leftarrow  \underset{U\in \R_{\ge 0}^{I_{i}\times R},\, \lVert U - U_{t-1}^{(2)}\rVert_{F}\le c'w_{t}}{\argmin} \,\,  \hat{f}_{t}(U_{t}^{(1)}, U),
	\end{cases}
\end{align}
where $\lambda\ge 0$ is an \commHL{absolute} constant and $(w_{t})_{t\ge 1}$ \commHL{is a} non-increasing sequence of weights in $(0,1]$. The recursively defined function $\hat{f}_{t}:\mathcal{D}=[U^{(1)},\dots,U^{(n)}]\mapsto [0,\infty)$ is called the \textit{surrogate loss function}, which is quadratic in each factor $U^{(i)}$ but is not jointly convex. Namely, when the new tensor data $\X_{t}$ arrives, one computes the code $H_{t}\in \R^{R\times 1}_{\ge 0}$ for $\X_{t}$ with respect to the previous loading matrices in $\mathcal{D}_{t-1}$,  updates the surrogate loss function $\hat{f}_{t}$, and then \textit{sequentially} minimizes it to find updated loading matrices within diminishing search radius $c'w_{t}$.


\commHL{Note that the surrogate loss function $\hat{f}_{t}$ in \eqref{eq:Alg1_def_surrogate_loss} is defined by the same recursion that defines the empirical loss function in \eqref{eq:def_OCPDL_ELM}. However, notice that the loss term  $\ell(\mathcal{X}_{t},\mathcal{D})$ in the definition of the empirical loss function $f_{t}$ in \eqref{eq:def_OCPDL} involves optimizing over the code matrices $H$ in \eqref{eq:def_loss}, which should be done for every $\mathcal{X}_{s}$, $1\le s \le t$, in order to evaluate $f_{t}$. On the contrary, in the definition of the surrogate loss function $\hat{f}_{t}$ in \eqref{eq:Alg1_def_surrogate_loss}, this term is replaced with the sub-optimal loss $\ell(\mathcal{X}_{t},\mathcal{D},H_{t})$, which is sub-optimal since $H_{t}$ was found by decomposing $\mathcal{X}_{t}$ using the previous CP-dictionary $\Out(\mathcal{D}_{t-1})$. From this, it is clear that $\hat{f}_{t}\ge f_{t}$ for all $t\ge 0$. In other words, $\hat{f}_{t}$ is a majorizing surrogate of $f_{t}$.} 

Now we state our algorithm in the general mode case in  \commHL{Algorithm} \ref{alg:online NTF_highlevel}. Our algorithm combines two key elements: stochastic majorization-minimization (SMM) 
\cite{mairal2013stochastic} and block coordinate descent with diminishing radius (BCD-DR) \cite{lyu2020convergence}. SMM amounts to iterating the following steps: sampling new data points, constructing a strongly convex surrogate loss, and then minimizing the surrogate loss to update the current estimate. This framework has been successfully applied to online matrix factorization problems \cite{mairal2010online, mensch2017stochastic}. However, the biggest bottleneck in using a similar approach in the tensor case is that the surrogate loss function $\hat{f}_{t}$ in \eqref{eq:Alg1_def_surrogate_loss} is only block multi-convex, meaning that it is convex in each block of coordinates but nonconvex jointly. Hence we cannot find an exact minimizer for $\hat{f}_{t}$ to update all $n$ loading matrices \commHL{at the same time}.

\begin{algorithm}[H]
	\caption{Online CP-Dictionary Learning (online CPDL)}
	\label{alg:online NTF_highlevel}
	\begin{algorithmic}[1]
		\State \textbf{Input:} $(\mathcal{X}_{t})_{1\le t\le T}$ (minibatches of data tensors in $\R^{I_{1}\times \dots \times I_{n}\times b}$);  $[U_{0}^{(1)},\dots,U_{0}^{(n)}]\in \R^{I_{1}\times R} \times \dots \times \R^{I_{n}\times R}$ (initial loading matrices); \, $c'>0$ (search radius constant);
		\vspace{0.1cm} 
		\State \textbf{Constraints:}  $\mathcal{C}^{(i)}\subseteq \R^{I_{i}\times R}$, $1\le i\le n$, $\mathcal{C}^{\textup{code}}\subseteq \R^{R\times b}$ (e.g., nonnegativity constraints)

		\State \textbf{Parameters:} $R\in \mathbb{N}$ ($\#$ of dictionary atoms);\, $\lambda\ge 0$ ($\ell_{1}$-\commHL{regularization parameter}); \, $(w_{t})_{t\ge 1}$ (weights in $(0,1]$);
		\vspace{0.1cm}
		\State \quad Initialize surrogate loss $\hat{f}_{0}\equiv 0$;
		
		\State \quad \textbf{For $t=1,\ldots,T$ do:}
		\State  \quad \quad \textit{Coding}: Compute the optimal code matrix
		\begin{align}\label{eq:coding_OCPDL}
			\hspace{2cm} H_{t}\leftarrow \argmin_{H \in \mathcal{C}^{\textup{code}} \subseteq \R^{R \times b} } \,\, \ell(\mathcal{X}_{t}, U_{t-1}^{(1)},\dots,U_{t-1}^{(n)},H); \quad \text{(using Algorithm \ref{algorithm:spaser_coding})}
		\end{align}

		\State  \quad \quad \textit{Update surrogate loss function}:
		\vspace{-0.1cm}
		\begin{align}\label{eq:Alg1_def_surrogate_loss}
			\hat{f}_{t}(U^{(1)},\dots,U^{(n)})\leftarrow (1-w_{t}) \hat{f}_{t-1}(U^{(1)},\dots,U^{(n)}) + w_{t} \,\ell(\mathcal{X}_{t}, U^{(1)},\dots,U^{(n)},H_{t})  
		\end{align}

		\State  \quad \quad  \textit{Update loading matrices by restricted cyclic block coordinate descent:}
		\vspace{0.2cm}
		\State \quad \quad \quad \textbf{For $i=1,\dots,n$}  \textbf{do:}  \label{eq:Alg1_begin_BCD}
		\State 
		\vspace{-0.5cm}
		\begin{align}
			&\mathcal{C}_{t}^{(i)}\leftarrow \left\{U \in \mathcal{C}^{(i)}\,\bigg|\,   \lVert U-U_{t-1}^{(i)}\rVert_{F}\le c'w_{t}\right\}; \label{eq:Alg1_search_restriction} \commHL{\textup{($\triangleright$ \textit{Restrict the search radius by $c'w_{t}$})}}  \\
			&U_{t}^{(i)} \leftarrow \argmin_{U\in \mathcal{C}_{t}^{(i)}}\,\, \hat{f}_{t}\left(U_{t}^{(1)},\dots,U_{t}^{(i-1)},U,U_{t-1}^{(i+1)},\dots,U_{t-1}^{(n)}\right) ; \label{eq:Alg1_loading_update}\\[-0.4cm]
			&\hspace{7cm} \commHL{\textup{($\triangleright$\textit{Update the $i^{\textup{th}}$ loading matrix})}} \nonumber
		\end{align}
		\State  \quad \quad \quad \textbf{End for} 
		\State \quad\textbf{End for} 
		\State \textbf{Return:}  $[U_{T}^{(1)},\dots, U_{T}^{(n)}] \in \mathcal{C}^{(1)}\times \dots \times \mathcal{C}^{(n)}$;
	\end{algorithmic}
\end{algorithm}


In order to circumvent this issue, one could try to perform \commHL{a few rounds of} block coordinate descent (BCD) on the surrogate loss function $\hat{f}_{t}$, which can be easily done since $\hat{f}_{t}$ is convex in each loading matrix. However, this results in sub-optimal loading matrices in each iteration, causing a number of difficulties in convergence analysis. Moreover, global convergence of BCD to stationary points is not guaranteed in general  even for the deterministic tensor CP-decomposition problems \commHL{without constraints} \cite{ kolda2009tensor}, and such a guarantee is known only with some additional regularity conditions \cite{grippof1999globally, grippo2000convergence, bertsekas1999nonlinear}. There are other popular strategies of using proximal \cite{grippo2000convergence} or prox-linear \cite{xu2013block} modifications of BCD to improve convergence properties. While these methods only ensure square-summability $ \sum_{t=1}^{\infty} \lVert \mathcal{D}_{t} - \mathcal{D}_{t-1} \rVert_{F}^{2}<\infty$ of changes (see, e.g., \cite[Lem 2.2]{xu2013block}), we find it crucial for our stochastic analysis that we are able to control the individual changes $\lVert \mathcal{D}_{t}- \mathcal{D}_{t-1} \rVert_{F}$ of the loading matrices in each iteration. This motivates to use BCD with diminishing radius in \cite{lyu2020convergence}. More discussions on technical points in convergence analysis are given in  Subsection \ref{subsection:proof_overview}.


The coding step in \eqref{eq:coding_OCPDL} is a convex problem and can be easily solved by a number of known algorithms (e.g., LARS \cite{efron2004least}, LASSO \cite{tibshirani1996regression}, and feature-sign search \cite{lee2007efficient}). As we have noted before, the surrogate loss function $\hat{f}_{t}$ in \eqref{eq:Alg1_def_surrogate_loss} is quadratic in each block coordinate, so each of the subproblems in the factor matrix update step in \eqref{eq:Alg1_loading_update} is a constrained quadratic problem and can be solved by projected gradient descent (see \cite{mairal2010online, lyu2020online}).

Notice that implementing Algorithm \ref{alg:online NTF_highlevel} may seem to require unbounded memory as one needs to store all past data $\mathcal{X}_{1},\dots,\mathcal{X}_{t}$ to compute the surrogate loss function $\hat{f}_{t}$ in \eqref{eq:Alg1_def_surrogate_loss}. However, it turns out that there are certain bounded-sized statistics that aggregates the past information that are sufficient to parameterize $\hat{f}_{t}$ and also to \commHL{update} the loading matrices. This bounded memory implementation of Algorithm \ref{alg:online NTF_highlevel} is given in Algorithm \ref{algorithm:online NTF}, and a detailed discussion on the memory efficiency is relegated to Appendix \ref{section:statement_alg_bounded_memeory}.

\section{Convergence results}\label{sec:theory}

In this section, we state our main convergence result of Algorithm \ref{alg:online NTF_highlevel}. Note that all results that we state here also apply to Algorithm \ref{algorithm:online NTF}, which is a bounded-memory implementation of Algorithm \ref{alg:online NTF_highlevel}.

\subsection{Statement of main results}\label{subsection:statement_results}

We first layout all technical assumptions required for our convergence results to hold. 
\begin{customassumption}{(A1)}\label{assumption:A1}
	The observed minibatch of data tensors $\mathcal{X}_{t}=[\mathbf{X}_{t;1},\dots, \mathbf{X}_{t;b}]$ are given by $\mathcal{X}_{t}=\varphi(Y_{t})$, where $Y_{t}$ is an irreducible and aperiodic Markov chain defined on a finite state space $\Omega$ and $\varphi:\Omega\rightarrow \mathbb{R}^{I_{1}\times \dots \times I_{n}\times b}$ is a bounded function. \commHL{Denote the transition matrix and the unique stationary distribution of $Y_{t}$ by $P$ and $\pi$, respectively. }
\end{customassumption}
\vspace{-0.3cm}
\begin{customassumption}{(A2)}\label{assumption:A2}
	For each $1\le i \le n$, the $i^{\textup{th}}$ loading matrix is constrained to a compact and convex subset $\mathcal{C}^{(i)}\subseteq \R^{I_{i}\times R}$ that contains at least two points. Furthermore, the code matrices $H_{t}$ belong to a compact and convex subset $\mathcal{C}^{\textup{code}}\subseteq \R^{R\times b}$.
\end{customassumption}
\begin{customassumption}{(A3)}\label{assumption:A3}
	\commHL{The sequence of non-increasing weights $w_{t}\in (0,1]$ in Algorithm \ref{alg:online NTF_highlevel} $\sum_{t=1}^{\infty} w_{t}=\infty$, and $\sum_{t=1}^{\infty} w_{t}^{2}\sqrt{t}<\infty$. Furthermore, $w_{t+1}^{-1}-w_{t}^{-1}\le 1$ for all sufficiently large $t$.} 
\end{customassumption}
\vspace{-0.3cm}
\begin{customassumption}{(A4)}\label{assumption:A4}
	The expected loss function $f$ defined in \eqref{eq:def_OCPDL} is continuously differentiable and has Lipschitz gradient.  
\end{customassumption}

It is standard to assume that the sequence of signals is drawn from a data distribution of compact support in an independent fashion \cite{mairal2010online, mairal2013optimization}, which enables the processing of large data by using i.i.d. subsampling of minibatches. However, when the signals have to be sampled from some complicated or unknown distribution, obtaining even approximately independent samples is difficult. In this case, Markov Chain Monte Carlo (MCMC) provides a powerful sampling technique (e.g., sampling from the posterior in Bayesian methods \cite{van2018simple} or from the Gibbs measure for Cellular Potts models \cite{voss2012multi}, or motif sampling from sparse graphs \cite{lyu2019sampling}), where consecutive signals could be highly correlated. (\commHL{See Appendix \ref{sec:MC_intro}  for more background on Markov chains and MCMC.}) 

\commHL{An important notion in MCMC sampling is ``exponential mixing'' of the Markov chain\footnote{For our analysis, it is in fact sufficient to have a sufficiently fast polynomial mixing of the Markov chain. See \ref{assumption:A1'} in Appendix \ref{sec:MC_intro} for a relaxed assumption using countable state-space.}. For a simplified discussion, suppose in \ref{assumption:A1} that our data tensors $\X_{t}$ themselves form a Markov chain with unique stationary distribution $\pi$. Under the assumption of finite state space, irreducibility, and aperiodicity in \ref{assumption:A1}, the Markov chain $\X_{t}$ ``mixes'' to the stationary distribution $\pi$ at an exponential rate. Namely, for any $\eps>0$, one can find a constant $\tau=\tau(\eps)=O(\log \eps^{-1})$, called the ``mixing time'' of $\X_{t}$,  such that the conditional distribution of
	$\X_{t+\tau}$ given $\X_{t}$ is within total variation distance $\eps$ from $\pi$ regardless of the distribution of $\X_{t}$ (see \eqref{eq:def_TV} for the definition of total variation distance).  This mixing property of Markov chains is crucial both for practical applications of MCMC sampling  as well as our theoretical analysis. For instance, a common practice of using MCMC sampling to obtain approximate i.i.d. samples is to first obtaining a long Markov chain trajectory $(\X_{t})_{t\ge 1}$ and then thinning it to the subsequence $(\X_{k\tau})_{k\ge 1}$ \cite[Sec. 1.11]{brooks2011handbook}. Due to the choice of mixing time $\tau$, this forms an $\eps$-approximate i.i.d. samples from $\pi$. }

\commHL{
	However, thinning a Markov chain trajectory does not necessarily produce truly independent samples, so classical stochastic analysis that relies crucially on independence between data samples is not directly applicable. For instance, if $\X_{t}$ is a reversible Markov chain then the correlation within the subsequence is nonzero and at least of order $\eps$ (see Appendix \ref{subsection:MCMC} for the definition of reversibility and detailed discussion). In order to obtain truly independent samples, one may independently re-initialize a Markov chain trajectory, run it for $\tau$ iterations, and keep the last samples in each run (e.g., see the discussion in \cite{sun2018markov}). However, in both approaches, only one out of $\tau$ Markov chain samples are used for optimization, which could be extremely wasteful if the Markov chain converges to the stationary distribution slowly so that the implied constant in $\tau(\eps) = O(\log \eps^{-1})$ is huge. }

Instead, our assumption on input signals in \ref{assumption:A1} allows us to use every single sample in the same MCMC trajectory without having to "burn" lots of samples. Such Markovian extension of the classical OMF algorithm in \cite{mairal2010online} has recently been achieved in \cite{lyu2020online}, which has applications in dictionary learning, denoising, and edge inference problems for network data \cite{lyu2021learning}. 

Assumption \ref{assumption:A2} is also standard in the literature of dictionary learning (see \cite{mairal2010online, mairal2013stochastic}). A particular instance of interest is when they are confined to having nonnegative entries, in which case the learned dictionary components give a ``parts-based'' representation of the input signals \cite{lee1999learning}. 

\commHL{ Assumption \ref{assumption:A3} states that the sequence of weights $w_{t}\in (0,1]$ we use to recursively define the empirical loss \eqref{eq:def_OCPDL_ELM} and surrogate loss  \eqref{eq:Alg1_def_surrogate_loss} does not decay too fast so that $\sum_{t=1}^{\infty} w_{t}=\infty$ but  decay fast enough so that $\sum_{t=1}^{\infty} w_{t}^{2}\sqrt{t}<\infty$. This is analogous to requirements for stepsizes in stochastic gradient descent algorithms, where the stepsizes are usually required to be non-summable but square-summable (see, e.g., \cite{sun2018markov}). The additioanl factor $\sqrt{t}$ is used in our analysis to deduce the uniform convergence of the empirical loss $f_{t}$ to the expected loss $f$ (see Lemma \ref{lem:uniform_convergence_asymmetric_weights} \commHL{in Appendix \ref{sec:auxiliary_lemmas}}), which was also the case in the literature \cite{mairal2010online, mairal2013stochastic, mensch2017stochastic,lyu2020online}. Also, the condition $w_{t}^{-1}-w_{t-1}^{-1}\le 1$ for all suficiently large $t$ is equivalent to saying the recursively defined weights $w^{t}_{k}$ in \eqref{eq:def_OCPDL_ELM_explicit} are non-decreasing in $k$ for all sufficiently large $k$, which is required to use Lemma \ref{lem:uniform_convergence_asymmetric_weights} in Appendix \ref{sec:auxiliary_lemmas}. We also remark that \ref{assumption:A3} is implied by the following simpler condition: 
	\begin{customassumption}{(A3')}\label{assumption:A3'}
		The sequence of non-increasing weights $w_{t}\in (0,1]$ in Algorithm \ref{alg:online NTF_highlevel} satisfy either $w_{t}=t^{-1}$ for $t\ge 1$ or $w_{t}=\Theta(t^{-\beta}(\log t)^{-\delta})$ for some $\delta \ge 1$ and $\beta\in [3/4, 1)$. 
	\end{customassumption}
}

For Assumption \ref{assumption:A4}, we remark that it follows from the uniqueness of the solution of \eqref{eq:coding_OCPDL} (see \cite[Prop. 1]{mairal2010online}). This can be enforced for example by the elastic net penalization \cite{zou2005regularization}. Namely, we may add a quadratic regularizer $\lambda' \lVert H \rVert_{F}^{2}$ to the loss function $\ell$ in 	\eqref{eq:def_loss_full} for some $\lambda'>0$. Then the resulting quadratic function is strictly convex and hence it has a unique minimizer in the convex constraint set $\mathcal{C}^{\textup{code}}$. (See \cite[Sec. 4.1]{mairal2010online} and \cite[Sec. 4.1]{lyu2020online} for more detailed discussion on this assumption).

The main result in this paper, which is stated below in Theorem \ref{thm:online NTF_convergence}, states that the sequence $\mathcal{D}_{t}$ of CP-dictionaries produced by Algorithm \ref{alg:online NTF_highlevel} converges to the set of stationary points of the expected loss function $f$ defined in \eqref{eq:def_OCPDL}. 	To the best of our knowledge, Theorem \ref{thm:online NTF_convergence} is the first convergence guarantee for any online \textit{constrained} dictionary learning algorithm for tensor-valued signals or as an online \textit{unconstrained} CP-factorization algorithm, which have not been available even under the classical i.i.d. assumption on input signals. Recall that $f_{t}$ and $\hat{f}_{t}$ denote the empirical and surrogate loss function defined in \eqref{eq:def_OCPDL_ELM} and \eqref{eq:Alg1_def_surrogate_loss}, respectively.   

\begin{theorem}\label{thm:online NTF_convergence}
	Suppose \ref{assumption:A1}-\ref{assumption:A3} hold. Let $(\mathcal{D}_{t})_{t\ge 1}$ be an output of Algorithm \ref{alg:online NTF_highlevel}. Then the following hold. 
	\begin{description}
		\item[(i)] $\lim_{t\rightarrow \infty} \E[f_{t}(\mathcal{D}_{t})]  =\lim_{t\rightarrow \infty} \E[\hat{f}_{t}(\mathcal{D}_{t})] <\infty$.

		\item[(ii)] $f_{t}(\mathcal{D}_{t})-\hat{f}_{t}(\mathcal{D}_{t})\rightarrow 0$ and $f(\mathcal{D}_{t})-\hat{f}_{t}(\mathcal{D}_{t})\rightarrow 0$ as $t\rightarrow \infty$ almost surely.

		\item[(iii)] Further assume \ref{assumption:A4}. Then the distance (measured by block-wise Frobenius distance) between $\mathcal{D}_{t}$ and the set of all stationary points of $f$ in $\mathcal{C}^{\textup{dict}}$ converges to zero almost surely. 
		
	\end{description}
\end{theorem}

We acknowledge that a similar asymptotic convergence result under Markovian dependency in data samples has been recently obtained in \cite[Thm. 2]{sun2018markov} in the context of stochastic gradient descent for nonconvex optimization problems. The results are not directly comparable since \cite[Thm. 2]{sun2018markov} only handles unconstrained problems.

\subsection{Key lemmas and overview of the proof of Theorem \ref{thm:online NTF_convergence}.} 
\label{subsection:proof_overview}

In this subsection, we state the key lemmas we use to prove Theorem \ref{thm:online NTF_convergence} and illustrate our contribution to techniques for convergence analysis. 

As we mentioned in Section \ref{section:main_alg}, there is a major difficulty in convergence analysis in the multi-modal case $n\ge 2$ as the surrogate loss function $\hat{f}_{t}$ (see \eqref{eq:Alg1_def_surrogate_loss}) to be minimized for updating the loading matrices is only multi-convex in $n$ blocks. Note that we can view our algorithm (Algorithm \ref{alg:online NTF_highlevel}) as a multi-modal extension of stochastic majorization-minimization (SMM) in the sense that it reduces to standard SMM in the \commHL{case of vector-valued signals ($n=1$)}. We first list the properties of SMM that have been critically used in convergence analysis in related works \cite{mairal2010online, mairal2013stochastic,mensch2017stochastic, lyu2020online}:
\begin{enumerate}[label=\textbf{{\arabic*}}, leftmargin=0.6cm, itemsep=0.1cm]
	\item (Surrogate Optimality) \quad $\mathcal{D}_{t}$ is a minimizer of $\hat{f}_{t}$ over $\mathcal{C}^{\textup{dict}}$. \label{eq:cond_Surrogate Optimality}
	\item (Forward Monotonicity) \quad  $\hat{f}_{t}(\mathcal{D}_{t-1})\ge \hat{f}_{t}(\mathcal{D}_{t})$.  \label{eq:cond_forward_monotonicity}
	\item (Backward Monotonicity) \quad  $\hat{f}_{t-1}(\mathcal{D}_{t-1})\le \hat{f}_{t-1}(\mathcal{D}_{t})$.  \label{eq:cond_backward_monotonicity}
	\item (Second-Order Growth Property) \quad $\hat{f}_{t}(\mathcal{D}_{t-1})- \hat{f}_{t}(\mathcal{D}_{t}) \ge c \lVert \mathcal{D}_{t} - \mathcal{D}_{t-1} \rVert_{F}^{2}$ for some constant $c>0$.
	\label{eq:cond_2nd_order_growth}
	\item (Stability of Estimates) \quad $\lVert \mathcal{D}_{t} - \mathcal{D}_{t-1} \rVert_{F} = O(w_{t})$.\label{eq:cond_stability_estimates}
	\item (Stability of Errors) \quad For $h_{t}:=\hat{f}_{t}-f_{t}\ge 0$, $|h_{t}(\mathcal{D}_{t}) - h_{t-1}(\mathcal{D}_{t-1})|=O(w_{t})$. \label{eq:cond_stability_errors}
\end{enumerate}
For $n=1$, it is crucial that $\hat{f}_{t}$ is convex so that $\mathcal{D}_{t}$ is a minimizer of $\hat{f}_{t}$ in the convex constraint set $\mathcal{C}^{\textup{dict}}$, as stated in \labelcref{eq:cond_Surrogate Optimality}. From this the  monotonicity properties \labelcref{eq:cond_forward_monotonicity} and \labelcref{eq:cond_backward_monotonicity} follow immediately.  The second-order growth property \labelcref{eq:cond_2nd_order_growth} requires additional assumption that the surrogates $\hat{f}_{t}$ are strongly convex uniformly in $t$. Then \labelcref{eq:cond_backward_monotonicity} and \labelcref{eq:cond_2nd_order_growth} imply  \labelcref{eq:cond_stability_estimates}, which then implies  \labelcref{eq:cond_stability_errors}. Lastly, \labelcref{eq:cond_Surrogate Optimality} is also crucially used to conclude that every limit point of $(\mathcal{D}_{t})_{t\ge 1}$ is a stationary point of $f$ over $\mathcal{C}^{\textup{dict}}$. Now in the multi-modal case $n\ge 2$, we do not have \labelcref{eq:cond_Surrogate Optimality} so all of the implications mentioned above are not guaranteed. Hence the analysis in the multi-modal case requires a significant amount of technical innovation.     

Now we state our key lemma that handles the nonconvexity of the surrogate loss $\hat{f}_{t}$ in the general multi-modal case $n\ge 1$. 

\begin{lemma}[Key Lemma]\label{lem:main_analytic_lemma}
	Assume \ref{assumption:A1}-\ref{assumption:A3}. Let $(\mathcal{D}_{t})_{t\ge 1}$ be an output of Algorithm \ref{alg:online NTF_highlevel}. For all $t\ge 1$, the following hold:
	\begin{description}[itemsep=0.1cm]
		\item[(i)] (Forward Monotonicity) \quad $\hat{f}_{t}(\mathcal{D}_{t-1})\ge \hat{f}_{t}(\mathcal{D}_{t}) $;
		\item[(ii)] (Stability of Estimates) \quad  $\lVert \mathcal{D}_{t}-\mathcal{D}_{t-1}\rVert_{F} = O( w_{t})$;
		\item[(iii)] (Stability of Errors) \quad $|h_{t}(\mathcal{D}_{t}) - h_{t-1}(\mathcal{D}_{t-1})|=O(w_{t})$, where $h_{t}:=\hat{f}_{t}-f_{t}$. 
		
		\item[(iv)] (Asymptotic Surrogate Stationarity)  \quad Let $(t_{k})_{k\ge 1}$ be any sequence such that $\mathcal{D}_{t_{k}}$ and $\hat{f}_{t_{k}}$ converges almost surely. Then $\mathcal{D}_{\infty}:=\lim_{k\rightarrow \infty} \mathcal{D}_{t_{k}}$ is almost surely a stationary point of $\hat{f}_{\infty}:=\lim_{k\rightarrow\infty} \hat{f}_{t_{k}}$ over $\mathcal{C}^{\textup{dict}}$. 
	\end{description}
\end{lemma}

We show Lemma \ref{lem:main_analytic_lemma} \textbf{(i)} using a monotonicity property of block coordinate descent. One of our key observations is that we can directly ensure the stability properties \labelcref{eq:cond_stability_estimates} and \labelcref{eq:cond_stability_errors} (Lemma \ref{lem:main_analytic_lemma} \textbf{(ii)} and \textbf{(iii)}) by using a search radius restriction (see  \labelcref{eq:Alg1_search_restriction} in Algorithm \ref{alg:online NTF_highlevel}). In turn, we do not need the properties \labelcref{eq:cond_backward_monotonicity} and \labelcref{eq:cond_2nd_order_growth}. In particular, our analysis does not require strong convexity of the surrogate loss $\hat{f}_{t}$ in each loading matrices as opposed to the existing analysis (see, e.g., \cite[Assumption \textbf{B}]{mairal2010online} and \cite[Def. 2.1]{mairal2013stochastic}). Lastly, our analysis requires that estimates $\mathcal{D}_{t}$ are only asymptotically stationary to the limiting surrogate loss function along convergent subsequences, as stated in Lemma \ref{lem:main_analytic_lemma} \textbf{(iv)}. The proof of this statement is nontrivial and requires a substantial work. On a high level, the argument consists of demonstrating that the effect of search radius restriction by $O(w_{t})$ vanishes in the limit, and the negative gradient $-\nabla \hat{f}_{\infty}(\mathcal{D}_{\infty})$ is in the normal cone of $\mathcal{C}^{\textup{dict}}$ at $\mathcal{D}_{\infty}$.

The second technical challenge is to handle dependence on input signals, as stated in \ref{assumption:A1}. The theory of quasi-martingales \cite{fisk1965quasi, rao1969quasi} is a key ingredient in convergence analysis under i.i.d input in \cite{mairal2010online,mairal2013stochastic, agarwal2019online}. However, dependent signals do not induce quasi-martingale since conditional on the information $\mathcal{F}_{t}$ at time $t$, the following signal $\mathcal{X}_{t+1}$ could be heavily biased. We use the recently developed technique in \cite{lyu2020online} to overcome this issue of data dependence. \commHL{The essential fact is that for irreducible and aperiodic Markov chains on finite state space, the $N$-step conditional distribution converges exponentially in $N$ to the unique stationary distribution regardless of the initial distribution (exponential mixing). The key insight in \cite{lyu2020online} was that in the analysis, one can condition on ``distant past'' $\mathcal{F}_{t-\sqrt{t}}$, not on the present $\mathcal{F}_{t}$, in order to allow the underlying Markov chain to mix  close enough to the stationary distribution $\pi$ for $\sqrt{t}$ iterations. This is opposed to a common practice of thinning Markov chain samples in order to reduce the dependence between sample points we mentioned earlier in Subsection \ref{subsection:statement_results}. We provide the estimate based on this technique in Lemma \ref{lemma:increment_bd}. }

\commHL{For the statement of Lemma \ref{lemma:increment_bd}, recall that under \ref{assumption:A1}, the data tensor at time $t$ is given by $\X_{t}=\varphi(Y_{t})$, where $Y_{t}$ is an irreducible and aperiodic Markov chain on a finite state space $\Omega$ with transition matrix $P$. For $y,y'\in \Omega$ and $k\in \mathbb{N}$, $P^{k}(y,y')$ equals the $k$-step transition probability of $Y_{t}$ from $y$ to $y'$, and $P^{k}(y,\cdot)$ equals the distribution of $Y_{k}$ conditional on $Y_{0}=y$. We also use the notation $a^{+}=\max(0,a)$ for $a\in \R$. }

\begin{lemma}[Convergence of Positive Variation] \label{lemma:increment_bd}
	Let $(\mathcal{D}_{t})_{t\ge 1}$ be an output of Algorithm \ref{alg:online NTF_highlevel}. Suppose \ref{assumption:A1'}, \ref{assumption:A2}, and \ref{assumption:A3} hold.
	\begin{description}
		\item[(i)] Let $(a_{t})_{t\ge 0}$ be a sequence of non-decreasing non-negative integers such that $a_{t}= o(t)$. Then there exists absolute constants $C_{1},C_{2},C_{3}>0$ such that for all sufficiently large $t\ge 0$, 
		\begin{align}
			& \mathbb{E}\left[ \E\left[ w_{t+1}\big(\ell(\mathcal{X}_{t+1},\mathcal{D}_{t}) - f_{t}(\mathcal{D}_{t}) \big)\,\bigg|\, \mathcal{F}_{t-a_{t}} \right]^{+} \right] \\
			&\qquad\le C_{1}w_{t-a_{t}}^{2}\sqrt{t} + C_{2}w_{t}^{2}a_{t} + C_{3}w_{t} \sup_{\mathbf{y}\in \Omega} \lVert P^{a_{t}+1}(\mathbf{y},\cdot) - \pi \rVert_{TV}.
		\end{align}	
		
		\item[(ii)] $\displaystyle \sum_{t=0}^{\infty}  \left( \E\left[ \hat{f}_{t+1}(\mathcal{D}_{t+1}) - \hat{f}_{t}(\mathcal{D}_{t}) \right]\right)^{+}  \le \sum_{t=0}^{\infty}  w_{t+1}\left( \E\left[ \left(  \ell(\mathcal{X}_{t+1},\mathcal{D}_{t}) - f_{t}(\mathcal{D}_{t}) \right) \right]\right)^{+} <\infty.$
		
	\end{description}
\end{lemma}

\commHL{ We give some remarks on the statement of Lemma \ref{lemma:increment_bd}. According to Proposition \ref{prop:Wt_bd} in Section \ref{section:proof}, one of the main quantities we would like to bound is $\E[\ell(\X_{t+1},\mathcal{D}_{t})-\hat{f}_{t}(\mathcal{D})]^{+}$, which is the expected positive variation of the one-step difference between the one-point error $\ell(\X_{t+1},\mathcal{D}_{t})$ of factorizing the new data $\X_{t+1}$ using the current dictionary $\mathcal{D}_{t}$ and the empirical error $\hat{f}_{t}(\mathcal{D}_{t})$. According to the recursive update of the empirical and surrogate losses in \eqref{eq:def_OCPDL_ELM} and \eqref{eq:Alg1_def_surrogate_loss}, we want the weighted sum of such expected positive variations in Lemma \ref{lemma:increment_bd} \textbf{(ii)} is finite. This follows from the bound in Lemma \ref{lemma:increment_bd} \textbf{(i)}, as long as $\sum_{t=1}^{\infty} w_{t}^{2}\sqrt{t}<\infty$, $a_{t}=O(\sqrt{t})$, and the Markov chain $Y_{t}$ modulating the data tensor $\X_{t}$ mixes fast enough  (see \ref{assumption:A1}). Such conditions are satisfied from the assumptions \ref{assumption:A1} and \ref{assumption:A3}.
}

\section{Proof of the main result}
\label{section:proof}


In this section, we prove our main convergence result, Theorem \ref{thm:online NTF_convergence}. Throughout this section, we assume the code matrices $H_{t}$ and loading matrices $U_{t}^{(i)}$ belong to convex and compact constraint sets $H_{t}\in \mathcal{C}^{\textup{code}}\subseteq \R^{R\times b}$, $U_{t}^{(i)}\in \mathcal{C}^{(i)} \subseteq \R^{I_{i}\times R}$ as in \ref{assumption:A2} and denote  $\mathcal{C}^{\textup{dict}}=\mathcal{C}^{(1)}\times \dots \times \mathcal{C}^{(n)}\subseteq \R^{I_{1}\times R}\times \dots \times \R^{I_{n}\times R}$.

\subsection{Deterministic analysis}

In this subsection, we provide some deterministic analysis of our online algorithm (Algorithm \ref{alg:online NTF_highlevel}), which will be used in the forthcoming stochastic analysis. 

\commHL{First, we derive a parameterized form of Algorithm \ref{alg:online NTF_highlevel}, where the surrogate loss function $\hat{f}_{t}$ is replaced with $\hat{g}_{t}$, which is a block-wise quadradtic function with recursively updating parameters. This will be critical in our proof of Lemma \ref{lem:main_analytic_lemma} \textbf{(iv)} as well as deriving the bounded-memory implementation of Algorithm \ref{alg:online NTF_highlevel} stated in Algorithm \ref{algorithm:online NTF} in Appendix \ref{section:statement_alg_bounded_memeory}. } Consider the following block optimization problem 
\begin{align}\label{eq:scheme_online NTF_surrogate_full}
	&\text{Upon arrival of $\,\,\mathcal{X}_{t}$:} 
	\qquad
	\begin{cases}
		H_{t}= \argmin_{H\in \mathcal{C}^{\textup{code}}\subseteq \R^{R\times b}} \,\, \ell(\mathcal{X}_{t}, \mathcal{D}_{t-1}, H)\\
		A_{t} = (1-w_{t}) A_{t-1} + w_{t} H_{t}H_{t}^{T} \\
		\commHL{\B}_{t} = (1-w_{t}) \commHL{\B}_{t-1} + w_{t} (\mathcal{X}_{t}\times_{n+1} H_{t}^{T}) \\
		\mathcal{D}_{t}= \underset{\substack{\mathcal{D}=[U^{(1)},\dots,U^{(n)}]\in \mathcal{C}^{\textup{dict}} \\  \lVert U^{(i)} - U_{t-1}^{(i)}\rVert_{F}\le c'w_{t}\,\, \forall i} }{\argmin}\, \hat{g}_{t}(\mathcal{D})
	\end{cases},
\end{align}
where for each $\mathcal{D}=[U_{1},\dots,U_{n}]\in \mathcal{C}^{\textup{dict}}$  (here we use subscripts to denote modes for taking their transpose) \commHL{and $\hat{g}_{t}$ in \eqref{eq:scheme_online NTF_surrogate_full} is defined as }
\begin{align}\label{eq:def_g_hat}
	\hat{g}_{t}(\mathcal{D}) := \tr(A_{t} \: (U_n^T U_n \had \dots \had U_1^T U_1)) - 2\tr\left( \commHL{\B}_{t}^{(n+1)} (U_n \kr \dots \kr U_1)^{T}  \right),
\end{align}
where $\commHL{\B}_{t}^{(n+1)}\in \R^{I_{1}\cdots I_{n}\times R}$ denotes the mode-$(n+1)$ unfolding of $\commHL{\B}_{t}\in \R^{I_{1}\times \dots \times I_{n} \times R }$, and $A_{0}\in \R^{R\times R}$ and $\commHL{\B}_{0}\in \R^{I_{1}\times \dots \times I_{n}\times \commHL{R}}$ are tensors of all zero entries. \commHL{The following proposition shows that minimizing $\hat{f}_{t}$ in \eqref{eq:Alg1_def_surrogate_loss} is equivalent to minimizing $\hat{g}_{t}$ in  \eqref{eq:scheme_online NTF_surrogate_full}. This shows that $\hat{f}_{t}$, which requires the storage of all past tensors $\X_{1},\dots,\X_{t}$ for its definition, can in fact be parameterized by an aggregate matrix $A_{t}\in \R^{R\times R}$ and an aggregate tensor $\commHL{\B}_{t}\in \R^{I_{1}\times \cdots \times I_{n}\times R}$, whose dimensions are independent of $t$. This implies that Algorithm \ref{alg:online NTF_highlevel} can be implemented using a bounded memory, without needing to store a growing number of full-dimensional data tensors $\X_{1},\dots,\X_{t}$. See Appendix \ref{section:statement_alg_bounded_memeory} for more details. }

\begin{prop}\label{prop:g_t_derivation}
	The following hold: 
	\begin{description}
		\item[(i)] Let $\hat{f}_{t}$ be as in \eqref{eq:Alg1_def_surrogate_loss} and $\hat{g}_{t}$ be in \eqref{eq:def_g_hat}. Then \begin{align}
			\hat{f}_{t}(\commHL{\mathcal{D}}) &= \hat{g}_{t}(\commHL{\mathcal{D}}) + \sum_{s=1}^{t}\tr\left( \mat(\mathcal{X}_{s}) \mat(\mathcal{X}_{s})^{T}\right) + \lambda\sum_{s=1}^{t} \lVert H_{s} \rVert_{1},
		\end{align}
		
		\item[(ii)] For each $1\le j \le n$ and $t\ge 1$,  let $\overline{A}_{t;j}\in \R^{R\times R}$, $\overline{B}_{t;j}\in \R^{I_{\commHL{j}}\times R}$ be the output of Algorithm \ref{algorithm:inter_aggregation} with input $A_{t},\commHL{\B}_{t},U_{1},\dots,U_{n}$, and $j$. Then can rewrite $\hat{g}_{t}(\mathcal{D})=\hat{g}_{t}(U_{1},\dots,U_{n})$ in \eqref{eq:def_g_hat} as 
		\begin{align}
			\hat{g}_{t}(U_{1},\dots,U_{n}) =	\tr\left( U_{i} \overline{A}_{t;j} U_{i}^{T} \right)  - 2\tr\left( U_{i}\overline{B}_{t;j}^{T}\right).
		\end{align}

	\end{description}
\end{prop}

\begin{proof}
	Let $\mat(\mathcal{X}_{s})=[\textup{vec}(\X_{s;1}),\dots,  \textup{vec}(\X_{s;b})]\in \R^{(I_{1}\dots I_{n})\times b}$ denote the matrix whose $i^{\textup{th}}$ column is the vectorization $\textup{vec}(\X_{s;j})$ of the tensor $\X_{s;j}\in \R^{I_{1}\times \dots \times I_{n}}$. The first assertion follows easily from observing that, for each $[U_{1},\dots,U_{n}]\in \mathcal{C}^{\textup{dict}}$ and $H\in \R^{R\times b}$,
	\begin{align*}
		&\left\lVert \mathcal{X}_{s} - \mathtt{Out}(U_{1},\dots, U_{n}) \times_{n+1} H \right\rVert_{F}^{2} \\
		&\qquad = \left\lVert \mat(\mathcal{X}_{s})  - (U_n \kr \dots \kr U_1 ) H  \right\rVert_{F}^{2} \\
		&\qquad= \tr\left( (U_n \kr \dots \kr U_1) H H^{T} (U_n \kr \dots \kr U_1)^{T}  \right) \\ 
		&\qquad\qquad - 2\tr\left( \mat(\mathcal{X}_{s}) H^{T} (U_n \kr \dots \kr U_1)^{T}  \right) + \tr\left( \mat(\mathcal{X}_{s}) \mat(\mathcal{X}_{s})^{T}\right),
	\end{align*}
	and also note that 
	\begin{align*}
		&\tr\left( (U_n \kr \dots \kr U_1) H H^{T} (U_n \kr \dots \kr U_1)^{T}  \right) \\
		&\qquad = \tr(H H^{T}(U_n \kr \dots \kr U_1)^T(U_n \kr \dots \kr U_1))\\
		&\qquad = \tr(H H^{T} \: (U_n^T U_n \had \dots \had U_1^T U_1)).
	\end{align*}
	Then the linearity of trace show
	\begin{align}\label{eq:f_hat_g_hat}
		\hat{f}_{t}(U_{1},\dots,U_{n}) &= \tr(A_{t} \: (U_n^T U_n \had \dots \had U_1^T U_1)) - 2\tr\left( \widetilde{\B}_{t}(U_n \kr \dots \kr U_1)^{T}  \right) \\ 	
		& \qquad + \sum_{s=1}^{t}\tr\left( \mat(\mathcal{X}_{s}) \mat(\mathcal{X}_{s})^{T}\right) + \lambda\sum_{s=1}^{t} \lVert H_{s} \rVert_{1}, \nonumber
	\end{align}
	where \commHL{$A_{t}$ is recursively defined in \eqref{eq:scheme_online NTF_surrogate_full}} and $\widetilde{\B}_{t}\in \R^{(I_{1}\times \dots \times I_{n})\times b}$ is defined recursively by 
	\begin{align*}
		\widetilde{\B}_{s} = (1-w_{t})\widetilde{\B}_{s-1} + w_{t}\mat(\mathcal{X}_{s}) H_{s}^{T}.
	\end{align*}
	By a simple induction argument, one can show that $\widetilde{B}_{t}$ equals the mode-$(n+1)$ unfolding $\commHL{\B}_{t}^{(n+1)}$ of $\commHL{\B}_{t}$ \commHL{defined recursively in \eqref{eq:scheme_online NTF_surrogate_full}}, as desired.  
	
	For \textbf{(ii)}, first note that 
	\begin{align*}
		&\tr(A \: (U_n^T U_n \had \dots \had U_1^T U_1)) \\
		&\qquad = \tr((A \had U_1^TU_1 \had \dots U_{i - 1}^TU_{i - 1} \had U_{i + 1}^T U_{i + 1} \had \dots \had U_n^TU_n)\:U_j^TU_j) \\
		&\qquad= \tr(U_j \: (A \had U_1^TU_1 \had \dots U_{i - 1}^TU_{i - 1} \had U_{i + 1}^T U_{i + 1} \had \dots \had U_n^TU_n)\:U_j^T) \\
		&\qquad = \tr(U_j \: \overline{A}_{t;j}\:U_j^T).
	\end{align*}
	Also, recall that $\commHL{\B}_{t}^{(n+1)}$ and $U_n \kr \dots \kr U_1$ are  $\left(\prod_{i = 1}^n I_j \right) \times R$ matrices. Let $\commHL{\B}_{t}(,r)\in \R^{I_{1}\times \dots \times I_{n}}$  denote the $r^{\textup{th}}$ mode-$(n+1)$ slice of $\commHL{\B}_{t}$. We note that  
	{\small
		\begin{align*}
			&\tr\left( \commHL{\B}_{t}^{(n+1)}(U_n \kr \dots \kr U_1)^{T}  \right)\\
			& =  \sum_{r = 1}^{R}  \tr\left( \commHL{\B}_{t}(,r) \times_1 U_1(:, r) \times_2 \dots \times_{i - 1} U_{i - 1}(:, r) \times_{i} U_{i}(:,r) \times_{i + 1}U_{i + 1}(:, r) \times_{i + 2} \dots \times_{n} U_n(:, r)\right)   \\
			& =  \tr\left( \sum_{r = 1}^{R}  \left[ \commHL{\B}_{t}(,r) \times_1 U_1(:, r) \times_2 \dots \times_{i - 1} U_{i - 1}(:, r) \times_{i + 1}U_{i + 1}(:, r) \times_{i + 2} \dots \times_{n} U_n(:, r)\right] U_{i}(:,r)^{T} \right) \\
			&= \tr\left( U_{i} \overline{B}_{t;j}^{T} \right),
		\end{align*}
	}
	where $\overline{B}_{t;j}^{T}$ is as in the assertion. Then the assertion follows. 
\end{proof}

\begin{proof}[\textbf{Proof of Lemma \ref{lem:main_analytic_lemma} \textup{\textbf{(i)}-\textbf{(iii)}}}.] First, we show \textbf{(i)}. Write $\mathcal{D}_{t-1}=[U_{1},\dots,U_{n}]$ and $\mathcal{D}_{t}=[U_{1}',\dots,U_{n}']$ (here we use subscripts to denote modes). Using Proposition \ref{prop:g_t_derivation} \textbf{(i)}, we write 
	\begin{align}
		&\hat{f}_{t}(\mathcal{D}_{t-1}) - \hat{f}_{t}(\mathcal{D}_{t}) \\
		& \qquad =\hat{f}_{t}([U_{1},\dots, U_{n}]) - \hat{f}_{t}([U_{1}',\dots, U_{n}']) \\
		&\qquad = \sum_{i=1}^{n} \hat{f}_{t}([U_{1}',\dots, U_{i-1}',U_{i},U_{i+1}, \dots, U_{n}]) - \hat{f}_{t}([U_{1}',\dots, U_{i-1}',U_{i}', U_{i+1}, \dots, U_{n}]).
	\end{align}
	Recall that $U_{i}'$ is a minimizer of the function $U\mapsto \hat{f}_{t}([U_{1}',\dots, U_{i-1}',U,U_{i+1}, \dots, U_{n}])$ (which is convex by Proposition \ref{prop:g_t_derivation}) over the convex set $\mathcal{C}_{i}$ defined in Algorithm \ref{alg:online NTF_highlevel}. Also, $U_{i}'$ belongs to $\mathcal{C}_{i}$. Hence each summand in the last expression above is nonnegative. This shows $\hat{f}_{t}(\mathcal{D}_{t-1}) - \hat{f}_{t}(\mathcal{D}_{t})\ge 0$, as desired. Also note that \textbf{(ii)} is trivial by the search radius restriction in Algorithm \ref{alg:online NTF_highlevel}.

	Lastly, we show \textbf{(iii)}. Both $\hat{f}_{t}$ and $f_{t}$ are uniformly bounded and Lipschitz by Lemma \ref{lem:loss_Lipschitz} \commHL{in Appendix \ref{sec:auxiliary_lemmas}}. Hence $h_{t}=\hat{f}_{t}-f_{t}$ is also Lipschitz with some constant $C>0$ independent of $t$. Then by the recursive definitions of $\hat{f}_{t}$ and $f_{t}$ (see \eqref{eq:Alg1_def_surrogate_loss} and \eqref{eq:def_OCPDL_ELM}) and noting that $\ell(\mathcal{X}_{t}, \mathcal{D}_{t-1},H_{t})=\ell(\mathcal{X}_{t}, \mathcal{D}_{t-1})$, we have 
	\begin{align}\label{eq:pf_thm_1_f_fhat_same_hypothesis}
		&|h_{t}(\mathcal{D}_{t}) - h_{t-1}(\mathcal{D}_{t-1})|  \\
		&\qquad \le |h_{t}(\mathcal{D}_{t})-h_{t}(\mathcal{D}_{t-1})|+|h_{t}(\mathcal{D}_{t-1})-h_{t-1}(\mathcal{D}_{t-1})| \\
		&\qquad\le C \lVert \mathcal{D}_{t}-\mathcal{D}_{t-1} \rVert_{F} + \left|\left(\hat{f}_{t}(\mathcal{D}_{t-1}) - \hat{f}_{t-1}(\mathcal{D}_{t-1})\right) -  \left( f_{t}(\mathcal{D}_{t-1}) - f_{t-1}(\mathcal{D}_{t-1})\right) \right| \\
		&\qquad= C \lVert \mathcal{D}_{t}-\mathcal{D}_{t-1} \rVert_{F} + w_{t} |\hat{f}_{t-1}(\mathcal{D}_{t-1}) - f_{t-1}(\mathcal{D}_{t-1})|.
	\end{align}
	Hence this and \textbf{(ii)} show $|h_{t}(\mathcal{D}_{t}) - h_{t-1}(\mathcal{D}_{t-1})|=O(w_{t})$, as desired.
\end{proof}

Next, we establish two elementary yet important inequalities connecting the empirical and surrogate loss functions. This is trivial in the case of vector-valued signals, in which case we can directly minimize $\hat{f}_{t}$ over a convex constraint set $\mathcal{C}^{\textup{dict}}$ to find $\mathcal{D}_{t}$ so we have the `forward monotonicity' $\hat{f}_{t}(\mathcal{D}_{t})\le \hat{f}_{t}(\mathcal{D}_{t-1})$ immediately from the algorithm design. In the tensor case, this still holds since we use block coordinate descent to progressively minimize $\hat{f}_{t}$ in each loading matrix. 

\begin{prop}\label{prop:Wt_bd}
	Let $(\mathcal{D}_{t})_{t\ge 1}$ be an output of Algorithm \ref{alg:online NTF_highlevel}. Then for each $t\ge 0$, the following hold:
	\begin{description}
		\item[(i)] $\hat{f}_{t+1}(\mathcal{D}_{t+1}) - \hat{f}_{t}(\mathcal{D}_{t}) \le w_{t+1}\left(  \ell(\mathcal{X}_{t+1},\mathcal{D}_{t}) - f_{t}(\mathcal{D}_{t})  \right)$.
		
		\item[(ii)] $0\le w_{t+1}\left( \hat{f}_{t}(\mathcal{D}_{t}) - f_{t}(\mathcal{D}_{t}) \right)  \le    w_{t+1}\left( \ell(\mathcal{X}_{t+1},\mathcal{D}_{t}) - f_{t}(\mathcal{D}_{t}) \right)  + \hat{f}_{t}(\mathcal{D}_{t}) - \hat{f}_{t+1}(\mathcal{D}_{t+1})$.
	\end{description}
\end{prop}

\begin{proof}
	We begin by observing that 
	\begin{align}
		\hat{f}_{t+1}(\mathcal{D}_{t}) &=  {\color{black} (1-w_{t+1}) \hat{f}_{t}(\mathcal{D}_{t}) + w_{t+1}\ell_{t+1}(\mathcal{X}_{t+1},\mathcal{D}_{t} , H_{t+1}) }\\
		& = (1-w_{t+1}) \hat{f}_{t}(\mathcal{D}_{t}) + w_{t+1}\ell_{t+1}(\mathcal{X}_{t+1},\mathcal{D}_{t})
	\end{align}
	for all $t\ge 0$. The first equality above uses the definition of $\hat{f}_{t}$ in \eqref{eq:Alg1_def_surrogate_loss} and the second equality uses the fact that $H_{t+1}$ is a minimizer of $ \ell(\mathcal{X}_{t+1},\mathcal{D}_{t},H)$ over $\mathcal{C}^{\textup{code}}$. Hence  
	\begin{align} \label{eq:fhat_expansion}
		&\hat{f}_{t+1}(\mathcal{D}_{t+1}) - \hat{f}_{t}(\mathcal{D}_{t}) \\ \nonumber
		&\quad =  \hat{f}_{t+1}(\mathcal{D}_{t+1}) - \hat{f}_{t+1}(\mathcal{D}_{t}) + \hat{f}_{t+1}(\mathcal{D}_{t}) - \hat{f}_{t}(\mathcal{D}_{t})\\  \nonumber
		&\quad=\hat{f}_{t+1}(\mathcal{D}_{t+1}) - \hat{f}_{t+1}(\mathcal{D}_{t}) + (1-w_{t+1})\hat{f}_{t}(\mathcal{D}_{t}) + w_{t+1}\ell(\mathcal{X}_{t+1},\mathcal{D}_{t}) - \hat{f}_{t}(\mathcal{D}_{t})\\  \nonumber
		&\quad=\hat{f}_{t+1}(\mathcal{D}_{t+1}) - \hat{f}_{t+1}(\mathcal{D}_{t}) + w_{t+1}(\ell(\mathcal{X}_{t+1},\mathcal{D}_{t})-f_{t}(\mathcal{D}_{t})) + w_{t+1}(f_{t}(\mathcal{D}_{t}) - \hat{f}_{t}(\mathcal{D}_{t})).  \nonumber
	\end{align} 
	Now note that $f_{t}\le \hat{f}_{t}$ by definition. Furthermore, $\hat{f}_{t+1}(\mathcal{D}_{t+1}) - \hat{f}_{t+1}(\mathcal{D}_{t})  \le 0$ by Lemma \ref{lem:main_analytic_lemma} \textbf{(i)}, so the inequalities in both \textbf{(i)} and \textbf{(ii)} follow.
\end{proof}

\subsection{Stochastic analysis}
In this subsection, we develop stochastic analysis on our online algorithm, a major portion of which is devoted to handling Markovian dependence in signals as stated in assumption \ref{assumption:A1}. The analysis here is verbatim as the one developed in \cite{lyu2020online} for the vector-valued signal (or matrix factorization) case, which we present some of the important arguments in detail here for the sake of completeness. However, the results in this subsection crucially rely on the deterministic analysis in the previous section that was necessary to handle difficulties arising in the tensor-valued signal case. 

Recall that under our assumption \ref{assumption:A1}, the signals $(\mathcal{X}_{t})_{t\ge 0}$ are given as  $\mathcal{X}_{t}=\varphi(Y_{t})$ for a fixed function $\varphi$ and a Markov chain $(Y_{t})_{t\ge 0}$. Note that Proposition \ref{prop:Wt_bd} gives a bound on the change in surrogate loss $\hat{f}_{t}(\mathcal{D}_{t})$ in one iteration, which allows us  to control its \textit{positive variation} in terms of difference $\ell(\mathcal{X}_{t+1},\mathcal{D}_{t}) - f_{t}(\mathcal{D}_{t})$. The core of the stochastic analysis in this subsection \commHL{is to show that $w_{t+1}\E[\ell(\mathcal{X}_{t+1},\mathcal{D}_{t}) - f_{t}(\mathcal{D}_{t})]^{+}$ is summable.} In the classical setting when $Y_{t}$'s are i.i.d., our signals $\mathcal{X}_{t}=\varphi(Y_{t})$ are also i.i.d., so we can condition on the information $\mathcal{F}_{t}$ up to time $t$ so that 
\begin{align}
	\E\left[ \ell(\X_{t+1},\mathcal{D}_{t}) - f_{t}(\mathcal{D}_{t}) \,\bigg|\, \mathcal{F}_{t} \right] = f(\mathcal{D}_{t}) - f_{t}(\mathcal{D}_{t}).
\end{align} 
Note that for each fixed $\mathcal{D}\in \mathcal{C}^{\textup{dict}}$, $f_{t}(\commHL{\mathcal{D}})\rightarrow f(\commHL{\mathcal{D}})$ almost surely as $t\rightarrow \infty$ by the strong law of large numbers. To handle time dependence of the evolving dictionaries $\mathcal{D}_{t}$, one can instead look that the convergence of the supremum $\lVert f_{t} - f \rVert_{\infty}$ over the compact set $\mathcal{C}^{\textup{dict}}$, which is provided by the classical Glivenko-Cantelli theorem. This is the approach taken in \cite{mairal2010online, mairal2013stochastic} for i.i.d. input.

However, the same approach is not applicable for dependent signals, for instance, when $(Y_{t})_{t\ge 0}$ is a Markov chain. This is because, in this case, conditional on $\mathcal{F}_{t}$, the distribution of $Y_{t+1}$ is not necessarily the stationary distribution $\pi$. In fact, when $Y_{t}$'s form a Markov chain with transition matrix $P$, $Y_{t}$ given $Y_{t-1}$ has distribution $P(Y_{t-1}, \cdot)$,  and this conditional distribution is a constant distance away from the stationary distribution $\pi$. (\commHL{For instance, consider the case when $Y_{t}$ takes two values and it differs from $Y_{t-1}$ with probability $1 - \eps$. Then $\pi=[1/2,1/2]$ and the distribution of $Y_{t}$ converges exponentially fast to $\pi$, but $P(Y_{t-1},\cdot)$ is either $[1-\eps,\eps]$ or $[\eps,1-\eps]$ for all $t\ge 1$.}) 

To handle dependence in data samples, we adopt the strategy developed in \cite{lyu2020online} in order to handle a similar issue for vector-valued signals (or matrix factorization). The key insight in \cite{lyu2020online} is that, while the 1-step conditional distribution $P(X_{t-1}, \cdot)$ may be far from the stationary distribution $\pi$, the $N$-step conditional distribution $P^{N}(X_{t-N}, \cdot)$ is exponentially close to $\pi$ under mild conditions. Hence we can condition much early on -- at time $t-N$ for some suitable $N=N(t)$. Then the Markov chain runs $N+1$ steps up to time $t+1$, so if $N$ is large enough for the chain to mix to its stationary distribution $\pi$, then the distribution of $Y_{t+1}$ conditional on $\mathcal{F}_{t-N}$ is close to $\pi$. The error of approximating the stationary distribution by the $N+1$ step distribution can be controlled using total variation distance and Markov chain mixing bound. This is stated more precisely in the proposition below. 

\begin{prop}\label{prop:increment_bd}
	Suppose \ref{assumption:A1} hold. Fix a CP-dictionary $\mathcal{D}$. Then for each $t\ge 0$ and $0\le N<t$, conditional on the information $\mathcal{F}_{t-N}$ up to time $t-N$,
	\begin{align}
		\left(\E\left[ \ell(\mathcal{X}_{t+1},\mathcal{D}) - f_{t}(\mathcal{D})  \,\bigg|\, \mathcal{F}_{t-N} \right] \right)^{+} &\le \left| f(\mathcal{D}) - f_{t-N}(\mathcal{D}) \right|   + Nw_{t} f_{t-N}(\mathcal{D})  \\
		&\qquad + 2\lVert \ell(\cdot, \mathcal{D}) \rVert_{\infty} \sup_{\mathbf{y}\in \Omega} \lVert P^{N+1}(\mathbf{y},\cdot) - \pi \rVert_{TV}.
	\end{align}	
\end{prop}

\begin{proof}
	The proof is identical to that of \cite[Prop. 7.5]{lyu2020online}.
\end{proof}

\begin{proof}[\textbf{Proof of Lemma \ref{lemma:increment_bd}}.]
	Part \textbf{(i)} can be derived from Proposition \ref{prop:increment_bd} and Lemma \ref{lem:uniform_convergence_asymmetric_weights} \commHL{in Appendix \ref{sec:auxiliary_lemmas}}. See the proof of \cite[Prop. 7.8 \textbf{(i)}]{lyu2020online} for details. Next, part \textbf{(ii)} can be derived from part \textbf{(i)} with Proposition \ref{prop:Wt_bd} \textbf{(i)}. See the proof of \cite[Prop. 7.8 \textbf{(ii)}]{lyu2020online} for details. 
\end{proof}


\begin{lemma}\label{lemma:main_finite_sum}
	Let $(\mathcal{D}_{t})_{t\ge 1}$ be the output of Algorithm \ref{alg:online NTF_highlevel}.  Suppose \ref{assumption:A1}-\ref{assumption:A3} hold. Then the following hold. 
	\begin{description}
		\item[(i)] $\displaystyle \sum_{t=0}^{\infty} \E\left[ w_{t+1}\left(  \ell(\mathcal{X}_{t+1},\mathcal{D}_{t}) - f_{t}(\mathcal{D}_{t})  \right) \right]^{+} <\infty$;
		
		\item[(ii)] $\E[\hat{f}_{t}(\mathcal{D}_{t})]$ converges as $t\rightarrow \infty$;
		
		\item[(iii)] $\displaystyle \E\left[ \sum_{t=0}^{\infty} w_{t+1}\left( \hat{f}_{t}(\mathcal{D}_{t}) - f_{t}(\mathcal{D}_{t}) \right) \right] = \sum_{t=0}^{\infty} w_{t+1}\left(\E[\hat{f}_{t}(\mathcal{D}_{t})] - \E[ f_{t}(\mathcal{D}_{t})]\right)  <\infty$;
		
		\item[(iv)] $\displaystyle  \sum_{t=0}^{\infty} w_{t+1}\left( \hat{f}_{t}(\mathcal{D}_{t}) - f_{t}(\mathcal{D}_{t}) \right)  <\infty$ almost surely. 
	\end{description}
\end{lemma}

\begin{proof}
	Part \textbf{(i)} can be derived from Proposition \ref{prop:increment_bd} and Jensen's inequality. See the proof of \cite[Lem. 12 \textbf{(ii)}]{lyu2020online} for details. Parts \textbf{(ii)}-\textbf{(iv)} can be shown by using Propositions \ref{prop:Wt_bd}, \ref{prop:increment_bd},  and part \textbf{(i)}. See the proof of \cite[Lem. 13]{lyu2020online} for details. 
\end{proof}

\subsection{Asymptotic surrogate stationarity}

In this subsection, we prove Lemma \ref{lem:main_analytic_lemma} \textbf{(iv)}, which requires one of the most nontrivial arguments we give in this work. Throughout this subsection, we will denote by $(\mathcal{D}_{t})_{t\ge 1}$ the output of Algorithm \ref{alg:online NTF_highlevel} and $\Lambda := \{ \mathcal{D}_{t}\,|\, t\ge 1 \}\subseteq \mathcal{C}^{\textup{dict}}$. Note that by Proposition \ref{prop:g_t_derivation}, $\hat{f}_{t_{k}}$ converges almost surely if and only if $A_{t_{k}}, \commHL{\B}_{t_{k}}, \mathcal{X}_{t_{k}}, H_{t_{k}}$ converge a.s. as $k\rightarrow \infty$. In what follows, we say $\mathcal{D}_{\infty}\in \mathcal{C}^{\textup{dict}}$ a \textit{stationary point} of $\Lambda$ if it is a limit point $\mathcal{D}_{\infty}$ of $\Lambda$ and  there exists a sequence $t_{k}\rightarrow \infty$ such that $\mathcal{D}_{t_{k}}\rightarrow \mathcal{D}_{\infty}$ and $\hat{f}_{\infty}:=\lim_{k\rightarrow \infty} \hat{f}_{t_{k}}$ exists almost surely and $\mathcal{D}_{\infty}$ is a stationary point of $\hat{f}_{\infty}$ over $\mathcal{C}^{\textup{dict}}$. Our goal is to show that every limit point of $\Lambda$ is stationary.  

The following observation is key to our argument.

\begin{prop}\label{prop:finite_range_short_points}
	Assume \ref{assumption:A1}-\ref{assumption:A3} hold. Let $(\mathcal{D}_{t})_{t\ge 1}$ be an output of Algorithm \ref{alg:online NTF_highlevel}. Then almost surely,
	\begin{align*}
		\sum_{t=1}^{\infty}  \left| \left( \nabla \hat{f}_{t+1}(\mathcal{D}_{t+1})^{T} (\mathcal{D}_{t} - \mathcal{D}_{t+1})  \right) \right|  <\infty. 
	\end{align*}
\end{prop}

\begin{proof}
	Since $\mathcal{C}^{\textup{dict}}$ is compact by \ref{assumption:A2} and the aggregate tensors $A_{t},\commHL{\B}_{t}$ are uniformly bounded by Lemma \ref{lem:H_bdd} \commHL{in Appendix \ref{sec:auxiliary_lemmas}}, we can see from Proposition \ref{prop:g_t_derivation} that $\nabla \hat{f}_{t+1}$ over $\mathcal{C}^{\textup{dict}}$ is Lipschitz with some uniform constant $L>0$. Hence by Lemma \ref{lem:surrogate_L_gradient} \commHL{in Appendix \ref{sec:auxiliary_lemmas}}, for all $t\ge 1$, 
	\begin{align*}
		\left|\hat{f}_{t+1}(\mathcal{D}_{t})-\hat{f}_{t+1}(\mathcal{D}_{t+1}) -  \tr\left( \nabla \hat{f}_{t+1}(\mathcal{D}_{t+1})^{T} (\mathcal{D}_{t} - \mathcal{D}_{t+1})  \right) \right| \le \frac{L}{2}\lVert \mathcal{D}_{t} - \mathcal{D}_{t+1} \Vert_{F}^{2}.
	\end{align*}
	Also note that $\hat{f}_{t+1}(\mathcal{D}_{t})\ge \hat{f}_{t+1}(\mathcal{D}_{t+1})$ by Lemma \ref{lem:main_analytic_lemma} \textbf{(i)}. Hence it follows that 
	\begin{align}\label{eq:1storder_growth_bd_pf}
		\left| \tr\left( \nabla \hat{f}_{t+1}(\mathcal{D}_{t+1})^{T} (\mathcal{D}_{t} - \mathcal{D}_{t+1})  \right) \right| \le   \frac{L}{2}\lVert \mathcal{D}_{t} - \mathcal{D}_{t+1} \Vert_{F}^{2} + \hat{f}_{t+1}(\mathcal{D}_{t})- \hat{f}_{t+1}(\mathcal{D}_{t+1})
	\end{align}
	On the other hand, \eqref{eq:fhat_expansion} and $\hat{f}_{t}\ge f_{t}$ yields 
	\begin{align*}
		0\le \hat{f}_{t+1}(\mathcal{D}_{t})-\hat{f}_{t+1}(\mathcal{D}_{t+1})\le 
		\hat{f}_{t}(\mathcal{D}_{t}) -\hat{f}_{t+1}(\mathcal{D}_{t+1}) + w_{t+1}(\ell(\mathcal{X}_{t+1},\mathcal{D}_{t})-f_{t}(\mathcal{D}_{t})).
	\end{align*} 
	Hence using Lemma \ref{lemma:main_finite_sum}, we have 
	\begin{align*}
		\sum_{t=1}^{\infty} \E\left[ \hat{f}_{t+1}(\mathcal{D}_{t})-\hat{f}_{t+1}(\mathcal{D}_{t+1}) \right] <\infty.
	\end{align*}
	Then from \eqref{eq:1storder_growth_bd_pf} and noting that $\lVert \mathcal{D}_{t} - \mathcal{D}_{t+1} \Vert_{F}^{2} = O(w_{t+1}^{2})$ and $\sum_{t=1}^{\infty} w_{t}^{2}<\infty$ (see \ref{assumption:A3}), it follows that 
	\begin{align*}
		\sum_{t=1}^{\infty} \E\left[ \left| \tr\left( \nabla \hat{f}_{t+1}(\mathcal{D}_{t+1})^{T} (\mathcal{D}_{t} - \mathcal{D}_{t+1})  \right) \right|\right] &= \frac{L}{2}\sum_{t=1}^{\infty} \E\left[  \lVert \mathcal{D}_{t} - \mathcal{D}_{t+1} \Vert_{F}^{2} \right] \\
		&\qquad + \sum_{t=1}^{\infty} \E\left[ \hat{f}_{t+1}(\mathcal{D}_{t})-\hat{f}_{t+1}(\mathcal{D}_{t+1}) \right] <\infty. 
	\end{align*}
	Then the assertion follows by Fubini's theorem and the fact that $\E[|X|]<\infty$ implies $|X|<\infty$ almost surely for any random variable $X$, where $|\cdot|$ denotes the largest absolute value among the entries of $X$. 
\end{proof}

Next, we show that the block coordinate descent we use to obtain $\param_{t+1}$ should always give the optimal first-order descent up to a small additive error.

\begin{prop}[Asymptotic first-order optimality]\label{prop:first_order_optimality}
	Assume \ref{assumption:A1}-\ref{assumption:A3} and $w_{t}=o(1)$. Then there exists a constant $c_{1}>0$ such that for all $t\ge 1$,
	\begin{align}
		\tr\left( \nabla \hat{f}_{t+1}(\param_{t+1})^{T} \frac{(\param_{t+1} - \param_{t})}{\lVert \param_{t+1} - \param_{t} \rVert_{F}}  \right) &\le  \inf_{\param\in \Param} \tr\left( \nabla \hat{f}_{t+1}(\param_{t})^{T}\frac{(\param - \param_{t}) }{\lVert \param - \param_{t}\rVert_{F}}\right) \\
		&\qquad+ c_{1}\lVert \param_{t+1}-\param_{t} \rVert_{F}^{2}.
	\end{align}
\end{prop}

\begin{proof}
	Fix a sequence $(b_{t})_{t\ge 1}$ such that $0< b_{t} \le c'w_{t}$ for all $t\ge 1$. 
	Write $\param_{t}=[U_{t}^{(1)},\dots,U_{t}^{(n)}]$ for $t\ge 1$ and denote
	\begin{align}
		\hat{f}_{t+1;i}:U\mapsto \hat{f}_{t+1}(U_{t+1}^{(1)},\dots,U_{t+1}^{(i-1)},U,U_{t}^{(i+1)},\dots,U_{t}^{(n)})
	\end{align} 
	for $U\in \R^{I_{i}}$  and $i=1,\dots,n$. Recall that $U_{t+1}^{(i)}$ is a minimizer of $\hat{f}_{t+1;i}$ over the convex set $\mathcal{C}_{t+1}^{(i)}$ defined in \eqref{eq:Alg1_search_restriction}.  Fix arbitrary $\param=[U^{(1)},\dots,U^{(n)}]\in \Param$ such that $\lVert \param - \param_{t} \rVert_{F} \le  b_{t+1}$. 
	Then  $\lVert U^{(i)} - U_{t}^{(i)}  \rVert_{F}\le b_{t+1}$ for all $1\le i \le n$. By convexity of $\mathcal{C}^{(i)}$, note that for each $U^{(i)}\in \mathcal{C}^{(i)}$,  $U_{t}^{(i)}+a(U^{(i)}-U_{t}^{(i)})\in \mathcal{C}^{(i)}$ for all $a\in [0,1]$.  Then by the definition of $\mathcal{C}_{t+1}^{(i)}$ and the choice of $U_{t+1}^{(i)}$, we have that for all $t\ge 1$,
	\begin{align}
		\hat{f}_{t+1;i}(U_{t+1}^{(i)})-\hat{f}_{t+1;i}(U_{t}^{(i)}) \le \hat{f}_{t+1;i}\left( U_{t}^{(i)}+a(U^{(i)}-U_{t}^{(i)})\right)-\hat{f}_{t+1;i}(U_{t}^{(i)}).
	\end{align}
	Recall that $\nabla \hat{f} = [\nabla \hat{f}_{t+1;1},\dots, \nabla \hat{f}_{t+1;n}]$ is Lipschitz with uniform Lipschitz constant $L>0$. Hence by Lemma \ref{lem:surrogate_L_gradient} \commHL{in Appendix \ref{sec:auxiliary_lemmas}}, there exists a constant $c_{1}>0$ such that for all $t\ge 1$,
	\begin{align}
		&\tr\left( \nabla \hat{f}_{t+1;i}(U_{t}^{(i)})^{T} (U_{t+1}^{(i)}-U_{t}^{(i)})  \right) - \frac{L}{2} \lVert U^{(i)}_{t+1}-U^{(i)}_{t} \rVert^{2} \\
		&\qquad \le  a\,  \tr\left( \nabla \hat{f}_{t+1;i}(U_{t}^{(i)})^{T} (U^{(i)}-U_{t}^{(i)})  \right) + \frac{L a^{2} \lVert U^{(i)}-U^{(i)}_{t} \rVert}{2}. 
	\end{align}
	Adding up these inequalities for $i=1,\dots,n$, we get 
	\begin{align}
		&\tr\left( \left[ \nabla \hat{f}_{t+1;1}(U_{t}^{(1)}),\dots,\nabla \hat{f}_{t+1;n}(U_{t}^{(n)}) \right]^{T} (\param_{t+1} - \param_{t}) \right) \\
		&\qquad \le a\,  \tr\left( \left[ \nabla \hat{f}_{t+1;1}(U_{t}^{(1)}),\dots,\nabla \hat{f}_{t+1;n}(U_{t}^{(n)}) \right]^{T} (\param - \param_{t}) \right)\\
		&\hspace{2cm} + \frac{L}{2}\lVert \param_{t+1}+\param_{t} \rVert_{F}^{2} + \frac{La^{2}}{2} \lVert \param -\param_{t} \rVert_{F}^{2}.
	\end{align}
	Since $\nabla \hat{f}_{t+1}$ is $L$-Lipschitz, using Cauchy-Schwarz inequality, 
	\begin{align}\label{eq:optimality1}
		&\tr\left( \nabla \hat{f}_{t+1}(\param_{t+1})^{T} (\param_{t+1} - \param_{t}) \right) \\
		&\qquad \le a\,  \tr\left( \nabla \hat{f}_{t+1}(\param_{t})^{T} (\param - \param_{t}) \right) + a \lVert \param_{t+1}-\param_{t} \rVert_{F} \, \lVert \param - \param_{t} \rVert_{F} \\ 
		&\hspace{2cm} + \frac{3L}{2}\lVert \param_{t+1}-\param_{t} \rVert_{F}^{2} + \frac{La^{2}}{2} \lVert \param -\param_{t} \rVert_{F}^{2} \\ 
		&\qquad  \le a\,  \tr\left( \nabla \hat{f}_{t+1}(\param_{t})^{T} (\param - \param_{t}) \right) +  a \lVert \param_{t+1}-\param_{t} \rVert_{F} \, \lVert \param - \param_{t} \rVert_{F} \\ 
		&\hspace{2cm} + \frac{3L}{2}\lVert \param_{t+1}-\param_{t} \rVert_{F}^{2} + ca^{2} \lVert \param -\param_{t} \rVert_{F}^{2} 
	\end{align}
	for some constant $c>0$ for all $t\ge 1$. Recall that the above holds for all $a\in [0,1]$. Note that since $\lVert \nabla \hat{f}_{t} \rVert$ is uniformly bounded and $\mathcal{D}^{\textup{dict}}$ is compact (see \ref{assumption:A2}), the last expression above, viewed as a quadratic function in $a$, is strictly increasing in $a$ for all $t\ge 1$ when $c>0$ is sufficiently large. We make such choice for $c_{3}$. Hence, the above holds for all $a\ge 0$. Now we may choose $a=b_{t+1}/\lVert \param-\param_{t}\rVert$ and bound the last expression by its first term plus $c_{1}(\lVert \param_{t+1}-\param_{t} \rVert_{F} + b_{t+1})^{2}$ for some constant $c_{1}>0$. Finally, by the radius restriction $\lVert \param_{t+1}-\param_{t} \rVert_{F}\le c'w_{t+1}$, we may choose $b_{t+1}=\lVert \param_{t+1}-\param_{t} \rVert_{F}$. Then the assertion follows by dividing both sides of the resulting inequality by $\lVert \param_{t+1}-\param_{t} \rVert$. 
\end{proof}

\begin{prop}\label{prop:stationary_conditions}
	Assume \ref{assumption:A1}-\ref{assumption:A3}. Suppose there exists a subsequence $(\mathcal{D}_{t_{k}})_{k\ge 1}$ such that either 
	\begin{align}\label{eq:stationary_conditions}
		\sum_{k=1}^{\infty} \lVert \mathcal{D}_{t_{k}} - \mathcal{D}_{t_{k}+1} \rVert_{F}= \infty  \quad \text{or} \quad \liminf_{k\rightarrow \infty} \,\, \left| \tr\left( \nabla \hat{f}_{t_{k}+1}(\mathcal{D}_{t_{k}+1})^{T} \frac{\mathcal{D}_{t_{k}} - \mathcal{D}_{t_{k}+1}}{\lVert \mathcal{D}_{t_{k}} - \mathcal{D}_{t_{k}+1} \rVert_{F}}  \right) \right| = 0.
	\end{align}
	There exists a further subsequence $(s_{k})_{k\ge 1}$ of $(t_{k})_{k\ge 1}$ such that  $\mathcal{D}_{\infty}:=\lim_{k\rightarrow \infty} \mathcal{D}_{s_{k}}$ exists and is a stationary point of $\Lambda$. 
\end{prop}

\begin{proof}
	By Proposition \ref{prop:finite_range_short_points}, we have 
	\begin{align}
		\sum_{k=1}^{\infty} \lVert \mathcal{D}_{t_{k}} - \mathcal{D}_{t_{k}+1} \rVert_{F} \, \left| \tr\left( \nabla \hat{f}_{t_{k}+1}(\mathcal{D}_{t_{k}+1})^{T} \frac{\mathcal{D}_{t_{k}} - \mathcal{D}_{t_{k}+1}}{\lVert \mathcal{D}_{t_{k}} - \mathcal{D}_{t_{k}+1} \rVert_{F}}  \right) \right| <\infty.
	\end{align}
	Hence the former condition implies the latter condition in \eqref{eq:stationary_conditions}. Thus it suffices to show that this latter condition implies the assertion. Assume this condition, and
	let $(s_{k})_{k\ge 1}$ be a subsequence of $(t_{k})_{k\ge 1}$ for which the  liminf in \eqref{eq:stationary_conditions} is achieved. By taking a subsequence, we may assume that $\mathcal{D}_{\infty}'=\lim_{k\rightarrow \infty} \mathcal{D}_{s_{k}}$ and $\hat{f}_{\infty}:=\lim_{k\rightarrow\infty} \hat{f}_{s_{k}}$ exist. 
	
	Now suppose for contradiction that $\mathcal{D}_{\infty}$ is not a stationary point of $\hat{f}_{\infty}$ over $\mathcal{C}^{\textup{dict}}$. Then there exists $\mathcal{D}^{\star}\in \mathcal{C}^{\textup{dict}}$ and $\delta>0$ such that 
	\begin{align}
		\tr\left( \nabla \hat{f}_{\infty}(\mathcal{D}_{\infty})^{T}(\mathcal{D}^{\star}-\mathcal{D}_{\infty}) \right)<-\delta<0. 
	\end{align}
	By triangle inequality, write 
	\begin{align}
		& \lVert  \nabla \hat{f}_{s_{k}+1}(\mathcal{D}_{s_{k}})^{T}(\mathcal{D}^{\star}-\mathcal{D}_{s_{k}}) - \nabla \hat{f}_{\infty}(\mathcal{D}_{\infty})^{T}(\mathcal{D}^{\star}-\mathcal{D}_{\infty}) \rVert_{F} \\
		&\qquad \le \lVert \nabla \hat{f}_{s_{k}+1}(\mathcal{D}_{s_{k}}) - \nabla \hat{f}_{\infty}(\mathcal{D}_{\infty}) \rVert_{F}\cdot \lVert \mathcal{D}^{\star}-\mathcal{D}_{s_{k}} \rVert_{F} +  \lVert \nabla \hat{f}_{\infty}(\mathcal{D}_{\infty}) \rVert_{F} \cdot \lVert \mathcal{D}_{\infty} - \mathcal{D}_{s_{k}} \rVert_{F}.
	\end{align}
	Noting that $\lVert \mathcal{D}_{t}-\mathcal{D}_{t-1} \rVert_{F} = O(w_{t}) = o(1)$, we see that the right hand side goes to zero as $k\rightarrow \infty$. Hence for all sufficiently large $k\ge 1$, we have 
	\begin{align}
		\tr\left( \nabla \hat{f}_{s_{k}+1}(\mathcal{D}_{s_{k}})^{T}(\mathcal{D}^{\star}-\mathcal{D}_{s_{k}}) \right)< -\delta/2.
	\end{align}
	Then by Proposition \ref{prop:first_order_optimality}, denoting $\lVert \mathcal{C}^{\textup{dict}} \rVert_{F}:=\sup_{\mathcal{D},\mathcal{D}'\in \mathcal{C}^{\textup{dict}}} \lVert \mathcal{D}-\mathcal{D}' \rVert_{F}<\infty$,
	\begin{align}
		\liminf_{k\rightarrow \infty}\,\,  \tr\left( \nabla \hat{f}_{s_{k}+1}(\mathcal{D}_{s_{k}+1})^{T} \frac{\mathcal{D}_{s_{k}} - \mathcal{D}_{s_{k}+1}}{\lVert \mathcal{D}_{s_{k}} - \mathcal{D}_{s_{k}+1} \rVert_{F}}  \right) \le -\frac{c_{1} \delta}{2 \lVert \mathcal{C}^{\textup{dict}} \rVert_{F}} < 0,
	\end{align}
	which  contradicts the choice of the subsequence $(\mathcal{D}_{s_{k}})_{k\ge 1}$. This shows the assertion. 
\end{proof}


Recall that during the update $\mathcal{D}_{t-1}\mapsto \mathcal{D}_{t}$  in \eqref{eq:Alg1_loading_update} each factor matrix of $\mathcal{D}_{t-1}$ changes by at most $w_{t}$ in Frobenius norm. For each $t\ge 1$, we say $\mathcal{D}_{t}$ is a \textit{long point} if none of the factor matrices of $\mathcal{D}_{t-1}$ change by $w_{t}$ in Frobenius norm and \textit{short point} otherwise. Observe that if $\mathcal{D}_{t}$ is a long point, then imposing the search radius restriction in \ref{eq:Alg1_search_restriction} has no effect and $\mathcal{D}_{t}$ is obtained from $\mathcal{D}_{t-1}$ by a single cycle of block coordinate descent on $\hat{f}_{t}$ over $\mathcal{C}^{\textup{dict}}$. 

\begin{prop}\label{prop:long_points_stationary}
	Assume \ref{assumption:A1}-\ref{assumption:A3} hold. If $(\mathcal{D}_{t_{k}})_{k\ge 1}$ is a convergent subsequence of $(\mathcal{D}_{t})_{t\ge 1}$  consisting of long points, then the $\mathcal{D}_{\infty}=\lim_{k\rightarrow \infty} \mathcal{D}_{s_{k}}$ is stationary.
\end{prop}

\begin{proof}
	For each $A\in \R^{R\times R}$, $B\in \R^{I_{1}\times \dots \times I_{n}\times b}$, $\mathcal{D}=[U^{(1)},\dots,U^{(n)}]\in \R^{I_{1}\times R}\times \dots \times \R^{I_{n}\times R}$, define 
	\begin{align}\label{eq:f_hat_pf_thm}
		\hat{g}(A,B,\mathcal{D}) &= \tr(A \: ((U^{(n)})^T U^{(n)} \had \dots \had (U^{(1)})^T U^{(1)}))\\ 
		&\qquad - 2\tr\left( B^{(n+1)} (U^{(n)} \kr \dots \kr U^{(1)})^{T}  \right),
	\end{align}
	where  $B^{(n+1)}$ denotes the mode-$(n+1)$ unfolding of $B$ (see also \eqref{eq:def_g_hat}). By taking a subsequence of $(t_{k})_{k\ge 1}$, we may assume that $A_{\infty}:=\lim_{k\rightarrow \infty} A_{t_{k}}$ and $\commHL{\B}_{\infty}:=\lim_{k\rightarrow \infty} \commHL{\B}_{t_{k}}$ exist. Hence the function $\hat{g}_{\infty}:=\lim_{k\rightarrow \infty}\hat{g}_{t_{k}} = \hat{g}(A_{\infty}, \commHL{\B}_{\infty}, \cdot)$ is well-defined. Noting that since $\nabla \hat{f}_{t} = \nabla \hat{g}_{t}$ for all $t\ge 1$ by Proposition \ref{prop:g_t_derivation}, it suffices to show that $\mathcal{D}_{\infty}$ is a stationary point of $\hat{g}_{\infty}$ over $\mathcal{C}^{\textup{dict}}$ almost surely. 
	
	The argument is similar to that of \cite[Prop. 2.7.1]{bertsekas1997nonlinear}. However, here we do not need to assume uniqueness of solutions to minimization problems of $\hat{f}_{t}$ in each block coordinate due to the added search radius restriction. Namely, write $\mathcal{D}_{\infty}=[U_{\infty}^{(1)},\dots,U_{\infty}^{(n)}]$. Then for each $k\ge 1$, 
	\begin{align}\label{eq:pf_stationarity_long_point1}
		\hat{g}_{t_{k}+1}(U_{t_{k}+1}^{(1)},U_{t_{k}}^{(2)},\dots,U_{t_{k}}^{(n)}) \le \hat{g}_{t_{k}+1}(U^{(1)},U_{t_{k}}^{(2)},\dots,U_{t_{k}}^{(n)})
	\end{align}
	for all $U^{(1)}\in \mathcal{C}^{(1)}\cap \{ U\colon \lVert U - U_{t_{k}}^{(1)}  \rVert_{F}\le c'w_{t_{k}+1} \}$. In fact, since $\mathcal{D}_{t_{k}}$ is a long point by the assumption, \eqref{eq:pf_stationarity_long_point1} holds for all $U^{(1)}\in \mathcal{C}^{(1)}$. Taking $k\rightarrow \infty$ and using the fact that $\lVert U_{t_{k}+1}^{(1)}-U_{t_{k}}^{(1)}\rVert_{F}\le c'w_{t_{k}+1}=o(1)$, 
	\begin{align}\label{eq:pf_stationarity_long_point2}
		\hat{g}_{\infty}(U_{\infty}^{(1)},U_{\infty}^{(2)},\dots,U_{\infty}^{(n)}) \le \hat{g}_{\infty}(U^{(1)},U_{\infty}^{(2)},\dots,U_{\infty}^{(n)}) \quad \text{for all $U_{1}\in \mathcal{C}^{(1)}$}.
	\end{align}
	Since $\mathcal{C}^{(1)}$ is convex, it follows that 
	\begin{align}
		\nabla_{1} \hat{g}_{\infty}(\mathcal{D}_{\infty})^{T}(U_{1}-U_{1}^{(\infty)}) \ge 0 \quad \quad \text{for all $U_{1}\in \mathcal{C}^{(1)}$},
	\end{align}
	\commHL{where $\nabla_{1}$ denotes the partial gradient with respect to the first block $U^{(1)}$.} By using a similar argument for other coordinates of $\mathcal{D}_{\infty}$, it follows that $\nabla \hat{g}_{\infty}(\mathcal{D}_{\infty})^{T}(\mathcal{D}-\mathcal{D}_{\infty}) \ge 0$ for all $\mathcal{D}\in \mathcal{C}^{\textup{dict}}$. This shows the assertion. 
\end{proof}

\begin{prop}\label{prop:non-stationary_nbh}
	Assume \ref{assumption:A1}-\ref{assumption:A3} hold. Suppose there exists a non-stationary limit point $\mathcal{D}_{\infty}$ of $\Lambda$. Then there exists $\eps>0$ such that the $\eps$-neighborhood $B_{\eps}(\mathcal{D}_{\infty}):=\{ \mathcal{D}\in \mathcal{C}^{\textup{dict}}\,|\, \lVert \mathcal{D}-\mathcal{D}_{\infty} \rVert_{F}<\eps\}$ with the following properties: 
	\begin{description}
		\item[(a)] $B_{\eps}(\mathcal{D}_{\infty})$ does not contain any stationary points of $\Lambda$. 
		
		\item[(b)] There exists infinitely many $\mathcal{D}_{t}$'s outside of $B_{\eps}(\mathcal{D}_{\infty})$.
	\end{description}
\end{prop}

\begin{proof}
	We will first show that there exists an $\eps$-neighborhood $B_{\eps}(\mathcal{D}_{\infty})$ of $\mathcal{D}_{\infty}$ that does not contain any long points of $\Lambda$. Suppose for contradiction that for each $\eps>0$, there exists a long point $\Lambda$ in $B_{\eps}(\mathcal{D}_{\infty})$. Then one can construct a sequence of long points converging to $\mathcal{D}_{\infty}$. But then by Proposition \ref{prop:long_points_stationary}, $\mathcal{D}_{\infty}$ is a stationary point, a contradiction.

	Next, we show that there exists $\eps>0$ such that $B_{\eps}(\mathcal{D}_{\infty})$ satisfies  \textbf{(a)}. Suppose for contradiction that there exists no such $\eps>0$. Then we have a sequence $(\mathcal{D}_{\infty;k})_{k\ge 1}$ of stationary points of $\Lambda$ that converges to $\mathcal{D}_{\infty}$. Denote the limiting surrogate loss function associated with $\mathcal{D}_{\infty;k}$ by $\hat{f}_{\infty;k}$. Recall that each $\hat{f}_{\infty;k}$ is parameterized by elements in a compact set (see \ref{assumption:A1}, Proposition \ref{prop:g_t_derivation}, and Lemma \ref{lem:loss_Lipschitz}) \commHL{in Appendix \ref{sec:auxiliary_lemmas}}. Hence by choosing a subsequence, we may assume that $\hat{f}_{\infty}:=\lim_{k\rightarrow \infty} \hat{f}_{\infty;k}$ is well-defined. Fix $\mathcal{D}\in \mathcal{C}^{\textup{dict}}$ and note that by Cauchy-Schwarz inequality,
	\begin{align}
		\nabla \hat{f}_{\infty}(\mathcal{D}_{\infty})^{T}(\mathcal{D}-\mathcal{D}_{\infty}) &\ge -\lVert \nabla\hat{f}_{\infty}(\mathcal{D}_{\infty}) -\nabla \hat{f}_{\infty;k}(\mathcal{D}_{\infty;k})\rVert_{F}  \cdot \lVert \mathcal{D}-\mathcal{D}_{\infty} \rVert_{F} \\
		&\qquad -  \lVert \nabla \hat{f}_{\infty;k}(\mathcal{D}_{\infty;k})\rVert_{F}  \cdot \lVert \mathcal{D}_{\infty}-\mathcal{D}_{\infty;k}  \rVert_{F} \\
		&\qquad \qquad + \nabla \hat{f}_{\infty;k}(\mathcal{D}_{\infty;k})^{T}(\mathcal{D}-\mathcal{D}_{\infty;k}).
	\end{align}
	Note that $\nabla \hat{f}_{\infty;k}(\mathcal{D}_{\infty;k})^{T}(\mathcal{D}-\mathcal{D}_{\infty;k})\ge 0$ since $\mathcal{D}_{\infty;k}$ is a stationary point of $\hat{f}_{\infty;k}$ over $\mathcal{C}^{\textup{dict}}$. Hence by taking $k\rightarrow \infty$, this shows $\nabla \hat{f}_{\infty}(\mathcal{D}_{\infty})^{T}(\mathcal{D}-\mathcal{D}_{\infty})\ge 0$. Since $\mathcal{D}\in \mathcal{D}^{\textup{dict}}$ was arbitrary, this shows that $\mathcal{D}_{\infty}$ is a stationary point of $\hat{f}_{\infty}$ over $\mathcal{C}^{\textup{dict}}$, a contradiction.

	Lastly, from the earlier results, we can choose $\eps>0$ such that $B_{\eps}(\mathcal{D}_{\infty})$ has no long points of $\Lambda$ and also satisfies  \textbf{(b)}.  We will show that $B_{\eps/2}(\mathcal{D}_{\infty})$ satisfies \textbf{(c)}. Then $B_{\eps/2}(\mathcal{D}_{\infty})$ satisfies  \textbf{(a)}-\textbf{(b)}, as desired. Suppose for contradiction there are only finitely many $\mathcal{D}_{t}$'s outside of $B_{\eps/2}(\mathcal{D}_{\infty})$. Then there exists an integer $M\ge 1$ such that $\mathcal{D}_{t}\in B_{\eps/2}(\mathcal{D}_{\infty})$ for all $t\ge M$. Then each $\mathcal{D}_{t}$ for $t\ge M$ is a short point of $\Lambda$. By definition, it follows that $\lVert \mathcal{D}_{t}-\mathcal{D}_{t}\lVert_{F} \ge w_{t}$ for all $t\ge M$, so $\sum_{t=1}^{\infty} \lVert \mathcal{D}_{t}-\mathcal{D}_{t}\lVert_{F}  \ge  \sum_{t=1}^{\infty} w_{t}=\infty$. Then by Proposition \ref{prop:stationary_conditions}, there exists a subsequence $(s_{k})_{k\ge 1}$ such that $\mathcal{D}_{\infty}':=\lim_{k\rightarrow \infty} \mathcal{D}_{t_{k}} $ exists and is stationary.  But since $\mathcal{D}'_{\infty}\in B_{\eps}(\mathcal{D})$, this  contradicts \textbf{(a)} for $B_{\eps}(\mathcal{D})$. This shows the assertion. 
\end{proof}

We are now ready to give a proof of Lemma  \ref{lem:main_analytic_lemma} \textbf{\textup{(iv)}}.

\vspace{0.1cm}
\begin{proof}[\textbf{Proof of Lemma \ref{lem:main_analytic_lemma} \textbf{\textup{(iv)}}}.]
	Assume \ref{assumption:A1}-\ref{assumption:A3} hold. Suppose there exists a non-stationary limit point $\mathcal{D}_{\infty}$ of $\Lambda$. By Proposition \ref{prop:non-stationary_nbh}, we may choose $\eps>0$ such that $B_{\eps}(\mathcal{D}_{\infty})$ satisfies the conditions \textbf{(a)}-\textbf{(b)} of Proposition \ref{prop:non-stationary_nbh}. Choose $M\ge 1$ large enough so that $c'w_{t}<\eps/4$ whenever $t\ge M$. We call an integer interval \commHL{$I:=[\ell,\ell')$} a \textit{crossing} if $\mathcal{D}_{\ell}\in B_{\eps/3}(\mathcal{D}_{\infty})$, $\mathcal{D}_{\ell'}\notin B_{2\eps/3}(\mathcal{D}_{\infty})$, and no proper subset of $I$ satisfies both of these conditions. By definition, two distinct crossings have empty intersection. Fix a crossing $I=[\ell,\ell')$. Then it follows that by triangle inequality,
	\begin{align}\label{eq:upcrossing_ineq}
		\sum_{t=\ell}^{\ell'-1} \lVert \mathcal{D}_{t+1}-\mathcal{D}_{t} \rVert_{F} \ge \lVert \mathcal{D}_{\ell'}-\mathcal{D}_{\ell} \rVert_{F} \ge \eps/3. 
	\end{align}
	Note that since $\mathcal{D}_{\infty}$ is a limit point of $\Lambda$, $\mathcal{D}_{t}$ visits $B_{\eps/3}(\mathcal{D}_{\infty})$ infinitely often. Moreover, by condition \textbf{(a)} of Proposition \ref{prop:non-stationary_nbh}, $\mathcal{D}_{t}$ also exits $B_{\eps}(\mathcal{D}_{\infty})$ infinitely often. It follows that there are infinitely many crossings.  Let $t_{k}$ denote the $k^{\textup{th}}$ smallest integer that appears in some crossing. By definition, $\mathcal{D}_{t_{k}}\in B_{2\eps/3}(\mathcal{D}_{\infty})$ for $k\ge 1$. Then $t_{k}\rightarrow \infty$ as $k\rightarrow \infty$, and by \eqref{eq:upcrossing_ineq}, 
	\begin{align}
		\sum_{k=1}^{\infty} \lVert \param_{t_{k}+1}-\param_{t_{k}} \rVert_{F} \ge (\text{$\#$ of crossings}) \, \frac{\eps}{3} = \infty. 
	\end{align}
	Then by Proposition \ref{prop:stationary_conditions}, there is a further subsequence $(s_{k})_{k\ge 1}$ of $(t_{k})_{k\ge 1}$ such that $\mathcal{D}_{\infty}':=\lim_{k\rightarrow \infty} \mathcal{D}_{s_{k}}$ exists and is stationary. \commHL{However, since $\mathcal{D}_{t_{k}}\in B_{2\eps/3}(\mathcal{D}_{\infty})$ for $k\ge 1$}, we have $\mathcal{D}_{\infty}'\in B_{\eps}(\mathcal{D}_{\infty})$. This contradicts the condition \textbf{(b)} of Proposition \ref{prop:non-stationary_nbh} for $B_{\eps}(\mathcal{D}_{\infty})$ that it cannot contain any stationary point of $\Lambda$. This shows the assertion. 
\end{proof}

\subsection{Proof of the main result}

Now we prove the main result in this paper, Theorem \ref{thm:online NTF_convergence}. 

\vspace{0.2cm}

\begin{proof}[\textbf{Proof of Theorem \ref{thm:online NTF_convergence}}.]
	Suppose \ref{assumption:A1}-\ref{assumption:A3}  hold. We first show \textbf{(i)}. Recall that $\E[\hat{f}_{t}(\mathcal{D}_{t})]$ converges by Lemma \ref{lemma:main_finite_sum}. Jensen's inequality and Lemma \ref{lem:main_analytic_lemma} \textbf{(iv)}  imply 
	\begin{align}
		\left|\E[h_{t+1}(\mathcal{D}_{t+1})] - \E[h_{t}(\mathcal{D}_{t})] \right| \le \E\left[ |h_{t+1}(\mathcal{D}_{t+1}) - h_{t}(\mathcal{D}_{t})| \right] =  O(w_{t+1}).
	\end{align}
	Since $\E[\hat{f}_{t}(\mathcal{D}_{t})]\ge \E[f_{t}(\mathcal{D}_{t})]$, Lemma \ref{lemma:main_finite_sum} \textbf{(ii)}-\textbf{(iii)} and Lemma \ref{lem:positive_convergence_lemma} \commHL{in Appendix \ref{sec:auxiliary_lemmas}} give 
	\begin{align}
		\lim_{t\rightarrow \infty} \E[f_{t}(\mathcal{D}_{t})] &= \lim_{t\rightarrow \infty} \E[\hat{f}_{t}(\mathcal{D}_{t})]  +  \lim_{t\rightarrow \infty} \left( \E[f_{t}(\mathcal{D}_{t})]-\E[\hat{f}_{t}(\mathcal{D}_{t})] \right)\\
		& = \lim_{t\rightarrow \infty} \E[\hat{f}_{t}(\mathcal{D}_{t})] \in (1,\infty).
	\end{align}
	This shows \textbf{(i)}. 
	
	Next, we show \textbf{(ii)}. Triangle inequality gives 
	\begin{align}\label{eq:thm_pf_triangle}
		|f(\mathcal{D}_{t}) - \hat{f}_{t}(\mathcal{D}_{t})| \le \left( \sup_{\mathcal{D}\in \mathcal{C}^{\textup{dict}}}| f(\mathcal{D})-f_{t}(\mathcal{D}) |\right) - h_{t}(\mathcal{D}_{t}).
	\end{align}
	Note that    $|h_{t+1}(\mathcal{D}_{t+1})-h_{t}(\mathcal{D}_{t})|= O(w_{t+1})$ by Lemma \ref{lem:main_analytic_lemma} \textbf{(iii)}. Hence Lemma \ref{lemma:main_finite_sum} \textbf{(iv)} and Lemma \ref{lem:positive_convergence_lemma} \commHL{in Appendix \ref{sec:auxiliary_lemmas}} show that $h_{t}(\mathcal{D}_{t})\rightarrow 0$ almost surely. Furthermore, \eqref{eq:thm_pf_triangle} and \commHL{Lemma \ref{lem:uniform_convergence_asymmetric_weights} in Appendix \ref{sec:auxiliary_lemmas}} show that $|f(\mathcal{D}_{t}) - \hat{f}_{t}(\mathcal{D}_{t})|\rightarrow 0$ almost surely. This completes the proof of \textbf{(ii)}.
	
	Lastly, we show \textbf{(iii)}. Further assume \ref{assumption:A4}. Let $\mathcal{D}_{\infty}\in \mathcal{C}^{\textup{dict}}$ be an arbitrary limit point of the sequence $(\mathcal{D}_{t})_{t\ge 1}$. Recall that $\Sigma_{t}:=(\mathcal{D}_{t}, A_{t},\commHL{\B}_{t},r_{t})_{t\ge 0}$ is bounded by Lemma \ref{lem:H_bdd}\commHL{(in Appendix \ref{sec:auxiliary_lemmas})} and \ref{assumption:A1} and \ref{assumption:A2}. Hence we may choose a random subsequence $(t_{k})_{k\ge 1}$ so that $\mathcal{D}_{t_{k}}\rightarrow \mathcal{D}_{\infty}$. By taking a further subsequence, we may also assume that $\Sigma_{t_{k}}$ converges to some random element $(\mathcal{D}_{\infty}, A_{\infty},\commHL{\B}_{\infty},r_{\infty})$ a.s. as $k\rightarrow \infty$. Then $\hat{f}_{\infty}:=\lim_{k\rightarrow\infty} \hat{f}_{t_{k}}$ exists almost surely. It is important to note that $\mathcal{D}_{\infty}$ is a stationary point of $\hat{f}_{\infty}$ over $\mathcal{C}^{\textup{dict}}$ by Lemma \ref{lem:main_analytic_lemma} \textbf{(iv)}.
	
	Recall that $\hat{f}_{t}(\mathcal{D}_{t}) - f_{t}(\mathcal{D}_{t})\rightarrow 0$ as $t\rightarrow \infty$ almost surely by part \textbf{(ii)}. 
	By using continuity of $\hat{f}_{t}$, $f_{t}$, $f$ in parameters (see  \ref{assumption:A4}), it follows that 
	\begin{align}\label{eq:pf_thm3_as_zero}
		\left| \hat{f}_{\infty}(\mathcal{D}_{\infty}) - f(\mathcal{D}_{\infty}) \right| &= \lim_{k\rightarrow \infty} \left| \hat{f}_{t_{k}}(\mathcal{D}_{t_{k}}) -  f_{t_{k}}(\mathcal{D}_{t_{k}})\right| \\
		& \le \lim_{k\rightarrow \infty} \left( \sup_{\mathcal{D}\in \mathcal{C}^{\textup{dict}}} \left|  f - f_{t_{k}}(\mathcal{D})\right|  - h_{t_{k}}(\mathcal{D}_{t_{k}})  \right) = 0,
	\end{align}
	where the last equality also uses Lemma \ref{lem:uniform_convergence_asymmetric_weights} \commHL{in Appendix \ref{sec:auxiliary_lemmas}}.

	Fix $\eps>0$ and $\mathcal{D}\in \R^{I_{1}\times R}\times \dots \times \R^{I_{n}\times R}$. Hence, almost surely,
	\begin{align}
		\hat{f}_{\infty}(\mathcal{D}_{\infty}+\mathcal{D}) = \lim_{k\rightarrow \infty} \hat{f}_{s_{k}}(\mathcal{D}_{s_{k}}+\mathcal{D})\ge \lim_{k\rightarrow \infty} f_{s_{k}}(\mathcal{D}_{s_{k}}+\mathcal{D}) = f(\mathcal{D}_{\infty}+\mathcal{D}),
	\end{align}
	where the last equality follows from Lemma \ref{lem:uniform_convergence_asymmetric_weights}. Since $\nabla \hat{f}$ and $\nabla f$ are both Lipschitz (see \ref{assumption:A4} for the latter), by Lemma \ref{lem:surrogate_L_gradient} \commHL{in Appendix \ref{sec:auxiliary_lemmas}}, we have 
	\begin{align}
		\left| 	\hat{f}_{\infty}(\mathcal{D}_{\infty}+\eps \mathcal{D}) -  \hat{f}_{\infty}(\mathcal{D}_{\infty}) - \tr\left( \nabla \hat{f}_{\infty}(\mathcal{D}_{\infty})^{T}(\eps\mathcal{D}) \right) \right|  \le c_{1}\eps^{2}\lVert \mathcal{D} \rVert_{F}^{2} ,\\
		\left|	f(\mathcal{D}_{\infty}+\eps \mathcal{D}) - f(\mathcal{D}_{\infty}) -  \tr\left( \nabla f(\mathcal{D}_{\infty})^{T}(\eps \mathcal{D}) \right) \right|  \le  c_{1} \eps^{2}\lVert \mathcal{D} \rVert_{F}^{2},
	\end{align}
	for some constant $c_{1}>0$ for all $\eps>0$. Recall that $\hat{f}_{\infty}(\mathcal{D}_{\infty})=f(\mathcal{D}_{\infty})$ a.s. by \eqref{eq:pf_thm3_as_zero}. Hence it follows that there exists some constant $c_{2}>0$ such that almost surely 
	\begin{align}\label{eq:thm3_pf_main_inequality}
		\tr\left( \left( \nabla \hat{f}_{\infty}(\mathcal{D}_{\infty})-\nabla f(\mathcal{D}_{\infty})\right)^{T}(\eps\mathcal{D}) \right) \ge - c_{2}\eps^{2} \lVert \mathcal{D}\rVert_{F}^{2}.
	\end{align}
	After canceling out $\eps>0$ and letting $\eps\searrow 0$ in \eqref{eq:thm3_pf_main_inequality},
	\begin{align}
		\tr\left( \left( \nabla \hat{f}_{\infty}(\mathcal{D}_{\infty})-\nabla f(\mathcal{D}_{\infty})\right)^{T}\mathcal{D} \right)\ge 0 \qquad \text{a.s.}
	\end{align}
	Since this holds for all $\mathcal{D}\in \R^{I_{1}\times R}\dots \times \R^{I_{n}\times R}$, it follows that $\nabla \hat{f}_{\infty}(\mathcal{D}_{\infty})=\nabla f(\mathcal{D}_{\infty})$ almost surely. But since $\mathcal{D}_{\infty}$ is a stationary point of $\hat{f}_{\infty}$ over $\mathcal{C}^{\textup{dict}}$ by Lemma \ref{lem:main_analytic_lemma} \textbf{(iv)}, it follows that $\nabla \hat{f}_{\infty}(\mathcal{D}_{\infty})$ is in the normal cone of $\mathcal{C}^{\textup{dict}}$ at $\mathcal{D}_{\infty}$ (see., e.g., \cite{boyd2004convex}). The same holds for $\nabla f(\mathcal{D}_{\infty})$. This means that $\mathcal{D}_{\infty}$ is a stationary point of $f$ over $\mathcal{C}^{\textup{dict}}$. Since $\mathcal{D}_{\infty}$ is an arbitrary limit point of $\mathcal{D}_{t}$, the desired conclusion follows. 
\end{proof}

\section{Experimental validation} 
\label{section:experimanl_validation}

In this section, we compare the performance of our proposed online CPDL algorithm (Algorithm \ref{alg:online NTF_highlevel}) for the standard (offline) NCPD problem \eqref{eq:NTF_intro_4} against the two most popular algorithms of Alternating Least Squares (ALS), which is a special instance of Block Coordinate Descent, and Multiplicative Update (MU) (see \cite{shashua2005non}) for this task. See Algorithms \ref{algorithm:ALS} and \ref{algorithm:MU} for implementations of ALS and MU.

\commHL{ We give a more precise statement of the NCPD problem we consider here. Given a 3-mode data tensor $\mathbf{X}\in \R^{d_{1}\times d_{2}\times d_{3}}_{\ge 0}$ and an integer $R\ge 1$, we want to find three nonnegative factor matrices $U^{(k)}\in \R^{d_{k}\times R}_{\ge 0}$, $k=1,2,3$, that minimize the following CP-reconstruction error:
	\begin{align}\label{eq:NCPD_experiment}
		\min_{  [U^{(1)},U^{(2)},U^{(3)}] \in \mathcal{C}^{\textup{dict}}_{M} } \left\lVert	\mathbf{X} - \sum_{i=1}^{R} \bigotimes_{k=1}^{3} U^{(k)}(:,i) \right\rVert_{F},
	\end{align}
	where $\mathcal{C}^{\textup{dict}}_{M}$ is the subset of $\R^{d_{1}\times R}_{\ge 0} \times \R^{d_{2}\times R}_{\ge 0}\times \R^{d_{3}\times R}_{\ge 0}$ consisting of factor matrices of Frobenius norm bounded by a fixed constant $M\ge\sqrt{R} \lVert \mathbf{X} \rVert_{F}^{1/3}$. Note that the constraint set $\mathcal{C}^{\textup{dict}}_{M}$ is convex and compact, as required in \ref{assumption:A2} for Theorem \ref{thm:online NTF_convergence} to apply. We claim that the additional bounded norm constraint on the factor matrices does not lose any generality, in the sense that an optimal solution of \eqref{eq:NCPD_experiment} with $M\ge\sqrt{R} \lVert \mathbf{X} \rVert_{F}^{1/3}$ has the same objective value as the optimal solution of \eqref{eq:NCPD_experiment} with $M=\infty$:
	\begin{align}\label{eq:claim_NCPD_bdd_norm}
		\left( \textup{optimal value of \eqref{eq:NCPD_experiment} for $M\ge \sqrt{R} \lVert \mathbf{X} \rVert_{F}^{1/3}$} \right) = \left( \textup{optimal value of \eqref{eq:NCPD_experiment} for $M=\infty$}\right).
	\end{align}
	In order to maintain the flow, we justify this claim at the end of this section. }


\commHL{We consider one synthetic and three real-world tensor data derived from text data and that were used for dynamic topic modeling experiments in \cite{kassab2021detecting}. Each document is encoded as a 5000 or 7000 dimensional word frequency vector using tf-idf vectorizer  \cite{rajaraman2011mining}.}
\begin{enumerate}
	\item $\mathbf{X}_{\text{synth}}\in \R_{\ge 0}^{100\times 100\times 100}$ is generated by $\mathbf{X}_{\text{synth}} = 0.01 * \Out(V_{1},V_{2},V_{3})$, where the loading matrices $V_{1},V_{2},V_{3}\in \R_{\ge 0}^{100\times 50}$ are generated by sampling each of their entries uniformly and independently from the unit interval $[0,1]$.
	\vspace{0.1cm}
	
	\item \commHL{$\mathbf{X}_{\text{20News}}\in \R_{\ge 0}^{40\times 5000\times 26}$ (41.6MB)
		is a tensor representing semi-synthetic text data based on 20 Newsgroups dataset \cite{20news} synthesized in \cite{kassab2021detecting} for dynamic topic modeling, consisting of 40 stacks of 26 documents encoded in 5000 dimensional word space.}
	
	\item $\mathbf{X}_{\textup{Twitter}}\in \R_{\ge 0}^{90\times 5000\times 1000}$ \commHL{(3.6GB)} is an anonymized Twitter text data related to the COVID-19 pandemic from Feb.\ 1 to May 1 of 2020. The three modes correspond to days, words, and tweets, in order. Each day, the top 1000 most retweeted English tweets are collected. The original data was collected in \cite{kassab2021detecting}.\footnote{For code repository, see \url{https://github.com/lara-kassab/dynamic-tensor-topic-modeling}}
	
	\vspace{0.1cm}
	\item $\mathbf{X}_{\textup{Headlines}}\in \R_{\ge 0}^{203\times 7000\times 700}$ \commHL{(8.0GB)} is a tensor derived in \cite{kassab2021detecting} from news headlines published over a period of 17 years sourced from the Australian news source ABC \cite{DVN/SYBGZL_2018}. The three modes correspond to months, words, and headlines, in order. In each month, 700 headlines are chosen uniformly at random.
\end{enumerate}

\begin{figure*}[h]
	\centering
	\includegraphics[width=1 \linewidth]{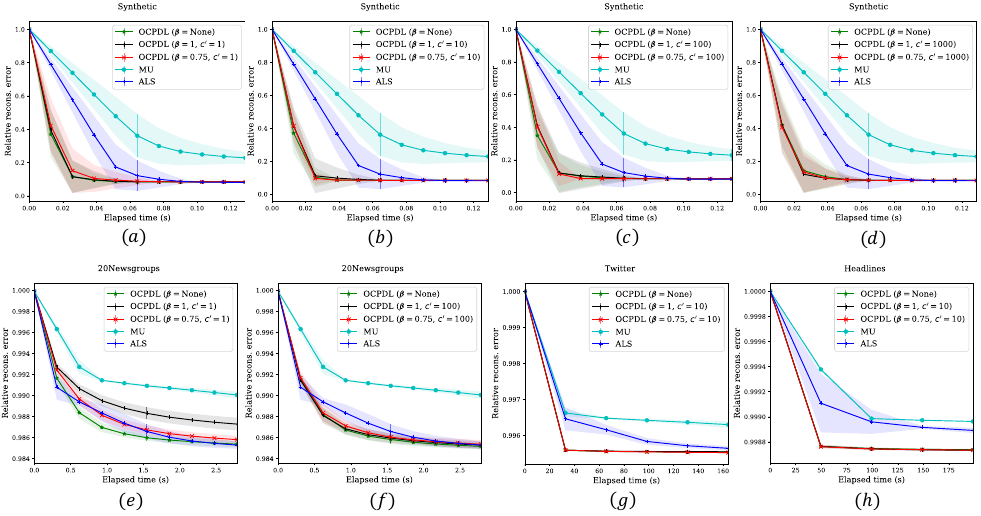}
	\vspace{-0.5cm}
	\caption{ Comparison of performance of online CPDL for the nonnegative tensor factorization problem against Alternating Least Squares (ALS) and Multiplicative Update (MU). For each data tensor, we apply each algorithm to find nonnegative loading matrices $U^{(1)},U^{(2)},U^{(3)}$ of $R=5$ columns. We repeat this multiple times (50 for synthetic, 20 for 20Newsgroups, and 10 for the other two) and the average reconstruction error with 1 standard deviation are shown by the solid lines and shaded regions of respective colors.  
	}
	\label{fig:benchmark}
\end{figure*}

For all datasets, we used all algorithms to learn the loading matrices $U^{(1)},U^{(2)},U^{(3)}$ with $R=5$ columns, that \commHL{evolve} in time as the algorithm proceeds. \commHL{The choice of $R=5$ is arbitrary and is not ideal especially for the real data tensors, but it suffices for the purpose of this experiment as a benchmark of our online CPDL against ALS and MU.} We plot the reconstruction error $\lVert \commHL{\mathbf{X}} - \Out(U^{(1)},U^{(2)},U^{(3)}) \rVert_{F}$ against elapsed time in both cases in Figure \ref{fig:benchmark}. Since these benchmark algorithms are also iterative (see Algorithms \ref{algorithm:ALS} and \ref{algorithm:MU}), we can measure how reconstruction error drops as the three algorithms proceed. In order to make a fair comparison, we compare the reconstruction error against CPU times with the same machine, \commHL{not against iteration counts, since a single iteration may have different computational costs across different algorithms. For ALS and MU, we disregarded the bounded norm constraint in \eqref{eq:NCPD_experiment},} which makes it only favorable to those benchmark methods so it is still a fair comparison of our method. 

We give some implementation details of online CPDL (Algorithm \ref{alg:online NTF_highlevel}) for the offline NCPD problem in \eqref{eq:NCPD_experiment}. From the given data tensor $\mathbf{X}$, we obtain a sequence of tensors $\mathbf{X}_{1},\dots,\mathbf{X}_{T}$ obtained by subsampling \commHL{$1/5$} of the coordinates from the last mode. Hence while ALS and MU require loading the entire tensors into memory, only \commHL{$1/5$} of the data needs to be loaded to execute online CPDL. In Figure \ref{fig:benchmark}, OCPDL ($\beta$) for for  $\beta\in \{ \commHL{0.75},\,1 \}$ denotes Algorithm \ref{alg:online NTF_highlevel} with weights $w_{t}=c't^{-\beta}/\log t$  (with $w_{1}=c'$), \commHL{where $0.75$ and $1$ for $\beta$ correspond to the two extreme values that satisfy the assumption \ref{assumption:A3} (see also \ref{assumption:A3'}) of Theorem \ref{thm:online NTF_convergence}}; \commHL{the case of $\beta=\textup{None}$ uses $c'=\infty$ and $w_{t}\equiv t^{-1}/(\log t)$ (with $w_{1}=1$). In all cases, the weights \commHL{satisfy \ref{assumption:A3} so the algorithm is guaranteed to converge to the stationary points of the objective function almost surely by Theorem \ref{thm:online NTF_convergence}.} The constant $c'$ is chosen from $\{1,10,100,1000\}$ for $\beta\in \{0.75, 1\}$. Initial loading matrices for Algorithm \ref{alg:online NTF_highlevel} are chosen with i.i.d. entries drawn from the uniform distribution on $[0,1]$. }

Note that since the last mode of the full tensors is subsampled, \commHL{the loading matrices we learn from online CPDL have sizes $(d_{1}\times R)$, $(d_{2}\times R)$, and $(d_{3}'\times R)$, where $d_{3}'< d_{3}$ equals the size of the last mode of the subsampled tensors. In order to compute the reconstruction error for the full $d_{1}\times d_{2}\times d_{3}$ tensor, we recompute the last factor matrix of size $d_{3}\times R$} by using the first two factor matrices with the sparse coding algorithm (Algorithm \ref{algorithm:spaser_coding}). This last step of computing a single loading matrix while fixing all the others is equivalent to a single step of ALS in Algorithm \ref{algorithm:ALS}.

\commHL{
	In Figure \ref{fig:benchmark}, each algorithm is used multiple times (50 for Synthetic, 20 for 20Newsgroups, and 10 for Twitter and Headlines)  for the same data, and the plot shows the average reconstruction errors together with their standard deviation (shading). In all cases except the smallest initial radius $c'=1$ on the densest tensor $\mathbf{X}_{\textup{20News}}$, online CPDL is able to obtain significantly lower reconstruction error much more rapidly than the other two algorithms and maintains low average reconstruction accuracy. }

\commHL{
	For $\mathbf{X}_{\textup{20News}}$ with $c'=1$ (Figure \ref{fig:benchmark} (e)), we observe some noticible difference in the performance of online CPDL depending on $\beta$, where larger values of $\beta$ (faster decaying radii) give slower convergence. This seems to be due to the fact that $\mathbf{X}_{\textup{20News}}$ is the densest among the four tensors by orders of magnitude and $c'=1$ gives too small of an initial radius. Namely, the average Frobeinus norm, $\lVert \mathbf{X} \rVert_{F}/(d_{1}d_{2}d_{3} )$ equals $6.26\times 10^{-5}$ for $\mathbf{X}_{\textup{Synthetic}}$, $1.10\times 10^{-3}$ for $\mathbf{X}_{\textup{20News}}$, $6.65\times 10^{-7}$ for $\mathbf{X}_{\textup{Twitter}}$, and $3.78\times 10^{-7}$ and $\mathbf{X}_{\textup{Headlines}}$. However, we did not observe any significant difference in all other cases. In general, when  $c'$ is large enough, it appears that the radius restriction in Algorithm \ref{alg:online NTF_highlevel} enables the theoretical convergence guarantee in Theorem \ref{thm:online NTF_convergence} without any compromise in practical performance, which did not depend significantly on the decay rate paramter $\beta$. In our experiments, $c'=\lVert \mathbf{X} \rVert_{F}$ was sufficiently large, where $\lVert \mathbf{X}_{\textup{Synthetic}} \rVert_{F}=62.61$,  $\lVert \mathbf{X}_{\textup{20News}} \rVert_{F}=32.74$, $\lVert \mathbf{X}_{\textup{Twitter}} \rVert_{F}=299.37$, and $\lVert \mathbf{X}_{\textup{Synthetic}} \rVert_{F}=376.38$. }

\commHL{
	\vspace{0.2cm}
	\begin{proof}[\textbf{Proof of claim \eqref{eq:claim_NCPD_bdd_norm}.}]
		Suppose $\mathcal{D}_{\infty}:=[U^{(1)},U^{(2)},U^{(3)}]$ is an optimal solution of \eqref{eq:NCPD_experiment} without norm restriction (i.e., $M=\infty$). Fix a column index $i\in \{1,\dots,R\}$ and positive scalars $\alpha_{1},\alpha_{2},\alpha_{3}$ such that $\alpha_{1}\alpha_{2}\alpha_{3} = 1$. The objective in \eqref{eq:NCPD_experiment} is invariant under rescaling the three respective columns: $U^{(k)}(:,i)\mapsto \alpha_{k} U^{(k)}(:,i)$ for $k\in \{1,2,3\}$. At the optimal factor matrices, the objective in \eqref{eq:NCPD_experiment} should be at most $\lVert \mathbf{X} \rVert_{F}$. Since all factors are nonnegative, it follows that 
		\begin{align}
			\prod_{k=1}^{3} \lVert \alpha_{k}  U^{(k)}(:,i) \rVert_{F} =  \left\lVert \bigotimes_{k=1}^{3} \alpha_{k} U^{(k)}(:,i) \right\rVert_{F}  \le \lVert \mathbf{X} \rVert_{F}.
		\end{align}
		Then we can choose $\alpha_{k}$'s in a way that $\lVert \alpha_{k}  U^{(k)}(:,i) \rVert_{F}$ is constant in $k$\footnote{e.g., $\alpha_{k}=a_{j}a_{l}/a_{k}^{2}$, where $a_{k}:=\lVert U^{(k)}(:,i) \rVert_{F}$ and $j,k,l\in \{1,2,3\}$ are distinct}, in which case $\lVert \alpha_{k}  U^{(k)}(:,i) \rVert_{F} \le \lVert \mathbf{X} \rVert_{F}^{1/3}$. This argument shows that we can rescale the $i$th columns of the optimal factor matrices in $\mathcal{D}_{\infty}$ in a way that the objective value does not change and the columns have norms bounded by $\lVert \mathbf{X} \rVert_{F}^{1/3}$. This holds for all columns $i$, so we can find a tuple of factor matrices $[V^{(1)},V^{(2)},V^{(3)}]$ in $\mathcal{C}^{\textup{dict}}_{M}$ that has the same objective value as $\mathcal{D}_{\infty}$ as long as $M\ge R \lVert \mathbf{X} \rVert_{F}^{1/3}$. 
	\end{proof}
}

\section{Applications}\label{sec:zebras}

For all our applications in this section, we take the constraint sets $\mathcal{C}^{\textup{code}}$ and $\mathcal{C}^{\textup{dict}}$ in Algorithm \ref{alg:online NTF_highlevel} to consists of \textit{nonnegative} matrices so that the learned CP-dictionary gives a ''parts-based representation" of the subject data as in classical NMF (see \cite{lee1999learning,lee2001algorithms,lee2009semi}). In all our experiments in this section, we used the balanced weight $w_{t}=1/t$, \commHL{which satisfies the assumption \ref{assumption:A3}.}

\subsection{Reshaping tensors before CP-decomposition to preserve joint features}
\label{subsection:reshaping}

\commHL{ Before we discuss our real-world applications of the online CPDL method, we first give some remarks on reshaping tensor data before factorization and why it would be useful in applications. }

\commHL{One may initially think that concatenating some modes of a tensor into a single mode before applying CP-decomposition loses joint features corresponding to the concatenated modes. In fact, if we \textit{undo the unfolding} after the decomposition, it actually \textit{preserves} the joint features. Hence in practice, one can exploit the tensor structure in multiple ways before CP-decomposition to disentangle a select set of features in the desired way, which we demonstrate through analyzing a diverse set of examples from image, video, and time-series in Section \ref{sec:zebras}. }

\commHL{To better illustrate our point, suppose we have three discrete random variables $X_{1},X_{2},X_{3}$, where $X_{i}$ takes $n_{i}$ distinct values for $1\le i \le 3$. Denote their $3$-dimensional joint distribution as a $3$-mode tensor $\mathbf{X}\in \R^{n_{1}\times n_{2}\times n_{3}}$. Suppose we have its CP-decomposition $\mathbf{X}\approx \Out(U^{(1)},U^{(2)}, U^{(3)})$. We can interpret this as the sum of $R$ product distributions of the marginal distributions given by the respective columns in the three factor matrices. In an extreme case of a $R=1$ CP-decomposition, any kind of joint features among multiple random variables will be lost in the single product distribution.  }

\commHL{On the other hand, consider combining the first two random variables $(X_{1},X_{2})$ into a single random variable, say, $Y_{1}$, which takes $n_{1}n_{2}$ distinct values. Then the joint distribution of $(Y_{1},X_{3})$ will be represented as a $2$-dimensional tensor $\mathbf{X}^{(12)}\in \R^{n_{1}n_{2}\times n_{3}}$, which corresponds to the tensor obtained by concatenating the first two modes of $\mathbf{X}$. Suppose we have its CP-decomposition $\mathbf{X}^{(12)}\approx \Out(V^{(1)},V^{(2)})$, where $V^{(1)}\in \R^{n_{1}n_{2}\times R}$ and $V^{(2)}\in \R^{n_{3}\times R}$. Then we can reshape each column $V^{(1)}(:,i)$ to a 2-dimensional tensor $V^{(1)}_{n_{1}\times n_{2}}(:,i)\in \R^{n_{1}\times n_{2}}$ by using the ordering of entries in $[n_{1}]\times [n_{2}]$ we used to concatenate $X_{1}$ and $X_{2}$ into $Y_{1}$. In this way, we have approximated the full joint distribution $\mathbf{X}$ as the sum of the product between two- and one-dimensional distributions $V^{(1)}_{n_{1}\times n_{2}}(:,i)\otimes V^{(2)}(:,i)$. In this factorization, the joint features of $X_{1}$ and $X_{2}$ can still be encoded in the 2-dimensional joint distributions $V^{(1)}_{n_{1}\times n_{2}}(:,i)$, and only the joint features between $(X_{1},X_{2})$ and $X_{3}$ are disentangled. }

\commHL{For instance, the tensor for the mouse brain activity video in Subsection \ref{subsection:brain_sim} has four modes, namely, $(\mathtt{time}, \mathtt{horizontal}, \mathtt{vertical},\mathtt{color})$. There is almost no change in the shape of the brain in the video and only the color changes indicate neuronal activation in time. Hence, we do not want to disentangle the horizontal, vertical, and color modes, but instead, concatenate them to maintain the joint feature of the spatial activation pattern (see Figure \ref{fig:brain}). See also Figures \ref{fig:gogh} and \ref{fig:weather} for the effect of various tensor reshaping before factorization in the context of image and time-series data. 
}

\subsection{Image processing applications}

We first apply our algorithm to patch-based image processing. A workflow for basic patch-based image analysis is to extract small overlapping patches from some large images, vectorize these patches, apply some standard dictionary learning algorithm, and reshape back. Dictionaries obtained from this general procedure have a wide variety of uses, including image compression, denoising, deblurring, and inpainting \cite{elad2010sparse, dong2011sparsity, papyan2017convolutional, ma2013dictionary}.

\begin{figure*}[h]
	\centering
	\includegraphics[width=1 \linewidth]{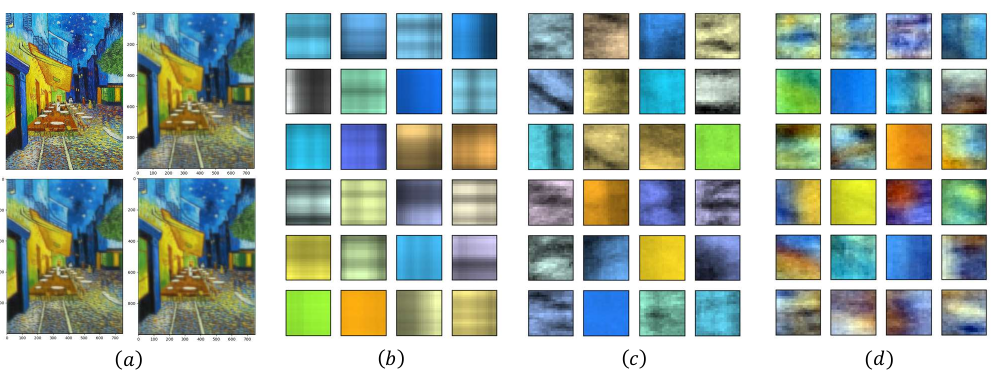}
	\vspace{-0.5cm}
	\caption{\commHL{Color image reconstruction by online CPDL. The original image is shown in the top left of (a). The top right reconstruction in (a) is derived from the dictionary learned from the unmodified tensor decomposition of color image patches, which
			is exemplified in (b). The bottom left reconstruction in (a) uses the dictionary in (c) learned by tensor decomposition of color image patches whose spatial modes are vectorized. The bottom right reconstruction in (a) uses the dictionary learned by a tensor decomposition of fully vectorized color image patches, which is shown in (d). }
	}
	\label{fig:gogh}
\end{figure*}

Although this procedure has produced countless state-of-the-art results, a major drawback to such methods is that vectorizing image patches can greatly slow down the learning process by increasing the effective dimension of the dictionary learning problem. Moreover, by respecting the natural tensor structure of the data, we find that our learned dictionary atoms display a qualitative difference from those trained on reshaped color image patch data. We illustrate this phenomenon in Figure \ref{fig:gogh}. Our experiment is as follows. Figure \ref{fig:gogh} (a) top left is a famous painting (Van Gogh's \textit{Caf\'{e} Terrace at Night}) from which we extracted $1000$ color patches of shape $(\mathtt{vertical}\times \mathtt{horizontal} \times \mathtt{color}) = (20 \times 20 \times 3)$. We applied our online CPDL algorithm (Algorithm \ref{alg:online NTF_highlevel} for 400 iterations with $\lambda=1$) to various reshapings of such patches to learn three separate dictionaries, each consisting of 24 atoms.

The first dictionary, displayed in Figure \ref{fig:gogh} (b), is obtained by applying online CPDL without reshaping the patches. Due to the rank-1 restriction on the atoms as a 3-mode tensor, the spatial features are parallel to the vertical or horizontal axes, and also color variation within each atom is only via scalar multiple (a.k.a. `saturation'). The second dictionary, Figure \ref{fig:gogh} (c), was trained by vectorizing the color image patches along the spatial axes, applying online CPDL to the resulting 2-mode data tensors of shape $(\mathtt{space}\times \mathtt{color} )= (400 \times 3)$ , and reshaping back. Here, the rank-1 restriction on the atoms as a 2-mode tensor separates the spatial and color features, but now the spatial features in the atoms are more `generic' as they do not have to be parallel to the vertical or horizontal axes. Note that the color variation within each atom is still via a scalar multiple. Lastly, the third dictionary, Figure \ref{fig:gogh} (d), is obtained by applying our online CPDL to the fully vectorized image patch data. Here the features in the atoms do not have any rank-1 restriction along with any mode so that they exhibit `fully entangled' spatial and color features. Although dictionary (b) requires much less storage, the reconstructed images from all three dictionaries shown in Figure \ref{fig:gogh} (a) show that it still performs adequately for the task of image reconstruction.

\subsection{Learning spatial and temporal activation patterns in cortex}
\label{subsection:brain_sim}

In this subsection, we demonstrate our method on video data of brain activity across a mouse cortex, and how our online CPDL learns dictionaries for the spatial and temporal activation patterns simultaneously. The original video is due to Barson et al. \cite{barson2020simultaneous} by using genetically encoded calcium indicators to image brain activity transcranially. \commHL{Simultaneous cellular-resolution two-photon calcium imaging of a local microcircuit as well as mesoscopic widefield calcium imaging of the entire cortical mantle in awake mice are used to capture the video (see \cite{barson2020simultaneous} for more details.) }

\begin{figure*}[h]
	\centering
	\includegraphics[width=1 \linewidth]{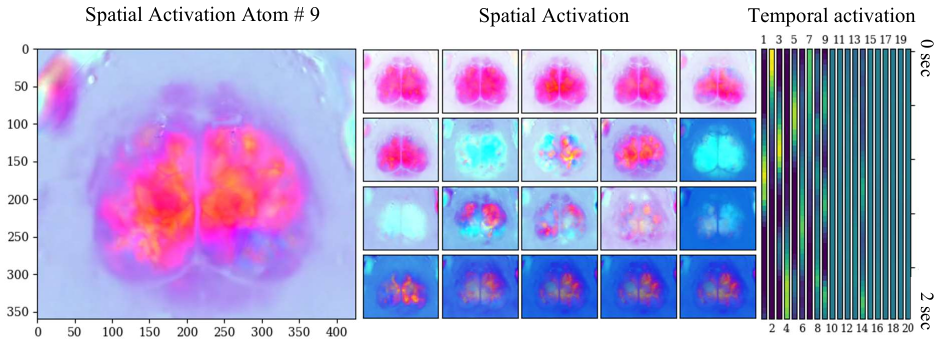}
	\vspace{-0.5cm}
	\caption{Learning 20 CP-dictionary \commHL{atoms} from video frames on brain activity across the mouse cortex. 
	}
	\label{fig:brain}
\end{figure*}


The original video frame is a tensor of shape $((1501, 360, 426, 3)$ corresponding to the four modes $(\mathtt{time}, \mathtt{horizontal}, \mathtt{vertical},\mathtt{color})$, where frames are 0.04 sec apart, which spans total $60.04$ seconds. We intend to learn weakly periodic patterns of spatial and temporal activation patterns of duration at most 2 seconds. To this end, we sample 50-frame (2 sec. long) clips uniformly at random for 200 times. Each sampled tensor is reshaped into $(\mathtt{time}, \, \mathtt{space * color})=(50, \,\,360*426*3)$ matrix, and then sequentially fed into the online CPDL algorithm with $w_{t}=c'/t$, $\lambda=2$, and $c'=10^{5}$.  Note that we vectorize the horizontal, vertical and color modes into a single mode before factorization in order to \textit{preserve} the spatial structure learned in the loading matrix. Namely, for spatial activation patterns, we desire dictionary atoms of the form of Figure \ref{fig:gogh} (d) rather than (b) or (c).

Our algorithm learns a CP-dictionary in the space-color mode that shows spatial activation patterns and the corresponding time mode shows their temporal activation pattern, as seen in Figure \ref{fig:brain}. Due to the nonnegativity constraint, spatial activation atoms representing localized activation regions in the cortex are learned, while the darker ones represent the background brain shape without activation. On the other hand, the activation frequency is simultaneously learned by the temporal activation atoms shown in Figure \ref{fig:brain} (right). For instance, the spacial activation atom \# 9 (numbered lexicographically) activates three times in its corresponding temporal activation atom in the right, so such activation pattern has an approximate \commHL{period} of 2/3 sec.  

\subsection{Joint time series dictionary learning}

A key advantage of online algorithms is that they are well-suited to applications in which data are arriving in real-time. We apply our algorithm to a weather dataset obtained from \cite{Beniaguev2017historical}. Beginning with a $(36 \times 2998 \times 4)$ tensor where the first mode corresponds to cities, the second mode to time in hours, and the third mode to weather data such that the frontal slices correspond to temperature, humidity, pressure, and wind speed. We regularized the data by taking a moving average over up to four hours (in part to impute missing data values), and by applying a separate rescaling of each frontal slice to normalize the magnitudes of the entries. 

\begin{figure*}[h]
	\centering
	\includegraphics[width=1 \linewidth]{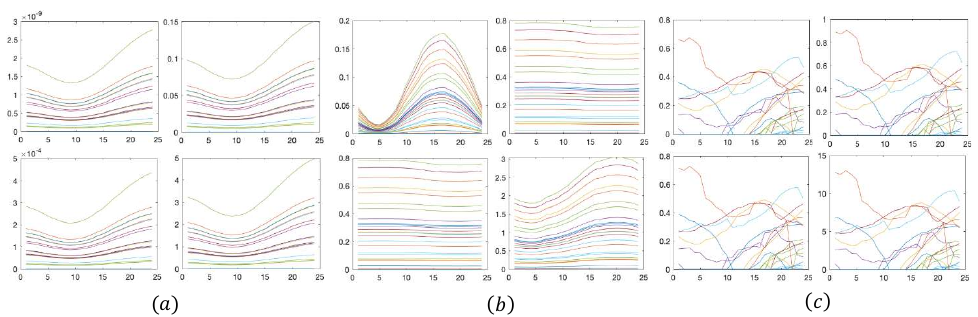}
	\vspace{-0.7cm}
	\caption{Display of one atom from three different dictionaries of $25$ atoms which were obtained from online CPDL on weather data: (a) no reshaping, (b) data which was reshaped to $36 \times (24 \times 4)$, and (c) data which was reshaped to $(36 \times 24) \times 4$. \commHL{For each subplot, the four subplots  represent the evolution of four measurements (temperature (top left), humidity (bottom left), pressure (top right), and wind speed (bottom right)) in time for 24 hours (horizontal axis) in 36 cities (in different colors).}
	}
	\label{fig:weather}
\end{figure*}

\commHL{
	From this large data tensor, we sequentially extracted smaller $(36 \times 24 \times 4)=(\texttt{cities}\times \texttt{time}\times \texttt{measurements})$ tensors  by dividing time into overlapping segments of length 24 hours, with overlap size 4 hours. Our experiment consisted of applying the online CPDL (Algorithm \ref{alg:online NTF_highlevel}) to this dataset to learn a single CP-dictionary atom ($R=1$), say $\mathbf{D}\in \R^{36\times 24\times 4}$, with three different reshaping schemes to preprocess the input tensors of shape $(36 \times 24 \times 4)$: no reshaping $(\texttt{cities}\times \texttt{time}\times \texttt{measurements})$ (Figure \ref{fig:weather} (a));  concatenating time and measurements $(\texttt{cities}\times (\texttt{time} * \texttt{measurements}))$ (Figure \ref{fig:weather} (b)); and concatenating cities and time  $( (\texttt{cities}* \texttt{time} ) \times  \texttt{measurements})$ (Figure \ref{fig:weather} (c)).  (See Subsection \ref{subsection:reshaping} for a discussion on reshaping and CP-dictionary learning). Roughly speaking, the single CP-dictionary atom we learn is a 3-mode tensor of shape $(36 \times 24 \times 4)$ that best approximates the evolution of four weather measurements during a randomly chosen 24-hour period from the original 2998-hour-long data subject to different constraints on the three modes depending how we reshape the input tensors. }

\commHL{In Figure \ref{fig:weather}, for each atom, the top left corner represents the first frontal slice (temperature), the bottom left the second frontal slice (humidity), the top right the third frontal slice (pressure), and the bottom right the fourth frontal slice (wind speed). The horizontal axis corresponds to time (in hours), each individual time series to a ``row" in the first mode, and the vertical axis to the value of the corresponding entry in the CP-dictionary atom. }

\commHL{We emphasize the qualitative difference in the corresponding learned dictionaries. In the first example in Figure \ref{fig:weather} (a), the CP-constraint is applied between all modes, so the single CP-dictionary atom $\mathbf{D}\in \R^{36\times 24\times 4}$ is given by the outer product of three marginal vectors, one for each of the three mode (analogous to Figure \ref{fig:gogh} (a)). Namely, let $\mathbf{D}=u_{1}\otimes u_{2}\otimes u_{3}$, where $u_{1}\in \R^{36}$, $u_{2}\in \R^{24}$, and $u_{3}\in \R^{4}$. Then $u_{2}$ represents a 24-hour long time series, and for example, the humidity of the first city is approximated by $(u_{1}(1) u_{3}(2))u_{2}$. This makes the variability of the time-series across different cities and measurements restrictive, as shown in Figure \ref{fig:weather} (a). }

\commHL{Next, in the second example in Figure \ref{fig:weather} (b), the CP-constraint is applied only between cities and the other two modes combined, so the single CP-dictionary atom $\mathbf{D}\in \R^{36\times 24\times 4}$ is given by $\mathbf{D}=u_{1}\otimes U_{23}$, where $u_{1}\in \R^{36}$ and $U_{23}\in \R^{24\times 4}$. Thus, for each city, the 24-hour evolution of the four measurements need not be some scalar multiple of a single time evolution vector as before, as we can use the full $24\times 4$ entries in $U_{23}$ to encode such information. On the other hand, the variability of the joint 24-hour evolution of the four measurements across the cities should only be given by a scalar multiple, as we can observe in Figure \ref{fig:weather} (b). }

\commHL{Lastly, in the third example in Figure \ref{fig:weather} (c), the CP-constraint is applied only between the measurements (last mode) and the other two modes combined, so the single CP-dictionary atom $\mathbf{D}\in \R^{36\times 24\times 4}$ is given by $\mathbf{D}=U_{12}\otimes u_{3}$, where $U_{12}\in \R^{36\times 24}$ and $u_{3}\in \R^{4}$. Thus, the 24-hour evolution of a latent measurement of the 36 cities can be encoded by the $36\times 24$ matrix $U_{12}$ without rank restriction. For each of the four measurements, this joint evolution pattern encoded in $U_{12}$ is multiplied by a scalar. For example, the temperature evolution across 36 cities is modeled by $u_{3}(1)U_{12}$. 
}

\section*{Acknowledgement}

HL is partially supported by NSF  DMS-2010035. CS and DN are grateful to and were partially supported by NSF BIGDATA $\#1740325$, NSF DMS $\#2011140$ and NSF DMS $\#2108479$.



\vspace{0.3cm}

\small{
	\bibliographystyle{amsalpha}
	\bibliography{mybib}
}

\vspace{1cm}

\appendix
\newpage

\section{Background on Markov chains and MCMC}

\label{sec:MC_intro}

\subsection{Markov chains}\label{subsection:MC}
Here we give a brief account on Markov chains on countable state space (see, e.g., \cite{levin2017markov}). Fix a countable set $\Omega$. A function $P:\Omega^{2} \rightarrow [0,\infty)$ is called a \textit{Markov transition matrix} if every row of $P$ sums to 1. A sequence of $\Omega$-valued random variables $(X_{t})_{t\ge 0}$ is called a \textit{Markov chain} with transition matrix $P$ if for all $x_{0},x_{1},\dots,x_{n}\in \Omega$, 
\begin{align}\label{eq:def_MC}
	\P(X_{n}=x_{n}\,|\, X_{n-1}=x_{n-1}, \dots, X_{0}=x_{0}) = \P(X_{n}=x_{n}\,|\, X_{n-1}=x_{n-1}) = P(x_{n-1},x_{n}). 
\end{align}
We say a probability distribution $\pi$ on $\Omega$ a \textit{stationary distribution} for the chain $(X_{t})_{t\ge 0}$ if $\pi = \pi P$, that is, 
\begin{align}
	\pi(x) = \sum_{y\in \Omega} \pi(y) P(y,x).
\end{align}
We say the chain $(X_{t})_{t\ge 0}$ is \textit{irreducible} if for any two states $x,y\in \Omega$ there exists an integer $t\ge 0$ such that $P^{t}(x,y)>0$. For each state $x\in \Omega$, let $\mathcal{T}(x) = \{ t\ge 1\,|\, P^{t}(x,x)>0 \}$ be the set of times when it is possible for the chain to return to starting state $x$. We define the \textit{period} of $x$ by the greatest common divisor of $\mathcal{T}(x)$. We say the chain $X_{t}$ is \textit{aperiodic} if all states have period 1. Furthermore, the chain is said to be \textit{positive recurrent} if there exists a state $x\in \Omega$
such that the expected return time of the chain to $x$ started from $x$ is finite. Then an irreducible and aperiodic Markov chain has a unique stationary distribution if and only if it is positive recurrent 
\cite[Thm 21.21]{levin2017markov}.

Given two probability distributions $\mu$ and $\nu$ on $\Omega$, we define their \textit{total variation distance} by 
\begin{align}\label{eq:def_TV}
	\lVert \mu - \nu \rVert_{TV} = \sup_{A\subseteq \Omega} |\mu(A)-\nu(A)|.
\end{align}
If a Markov chain $(X_{t})_{t\ge 0}$ with transition matrix $P$ starts at $x_{0}\in \Omega$, then by \eqref{eq:def_MC}, the distribution of $X_{t}$ is given by $P^{t}(x_{0},\cdot)$. If the chain is irreducible and aperiodic with stationary distribution $\pi$, then the convergence theorem (see, e.g., \cite[Thm 21.14]{levin2017markov}) asserts that the distribution of $X_{t}$ converges to $\pi$ in total variation distance: As $t\rightarrow \infty$,
\begin{align}\label{eq:finite_MC_convergence_thm}
	\sup_{x_{0}\in \Omega} \,\lVert P^{t}(x_{0},\cdot) - \pi \rVert_{TV} \rightarrow 0.
\end{align}
See \cite[Thm 13.3.3]{meyn2012markov} for a similar convergence result for the general state space chains. When $\Omega$ is finite, then the above convergence is exponential in $t$ (see., e.g., 
\cite[Thm 4.9]{levin2017markov})). Namely, there exists constants $\lambda\in (0,1)$ and $C>0$ such that for all $t\ge 0$, 
\begin{align}\label{eq:finite_MC_convergence_thm}
	\max_{x_{0}\in \Omega} \,\lVert P^{t}(x_{0},\cdot) - \pi \rVert_{TV} \le C \lambda^{t}.
\end{align}
\textit{Markov chain mixing} refers to the fact that, when the above convergence theorems hold, then one can approximate the distribution of $X_{t}$ by the stationary distribution $\pi$.

\begin{remark}\label{remark:mixing_rate}
	\normalfont 
	
	Our main convergence result in Theorem 3.1 assumes that the underlying Markov chain $Y_{t}$ is irreducible, aperiodic, and defined on a finite state space $\Omega$, as stated in \ref{assumption:A1}. This can be relaxed to countable state space Markov chains. Namely, Theorem 3.1 holds if we replace \ref{assumption:A1} by 
	\begin{customassumption}{(A1)'}\label{assumption:A1'}
		The observed data tensors $\mathbf{X}_{t}$ are given by $\mathbf{X}_{t}=\varphi(Y_{t})$, where $Y_{t}$ is an irreducible, aperiodic, and positive recurrent Markov on a countable and compact state space $\Omega$ and $\varphi:\Omega\rightarrow \mathbb{R}^{d\times n}$ is a bounded function. Furthermore, there exist  constants $\beta\in (3/4, 1]$ and $\gamma>2(1-\beta)$ such that
		\begin{align}
			w_{t} = O(t^{-\beta}),\qquad \sup_{\mathbf{y}\in \Omega} \lVert P^{t}(\mathbf{y},\cdot) - \pi \rVert_{TV} = O(t^{-\gamma}),
		\end{align}
		where $P$ and $\pi$ denote the transition matrix and unique stationary distribution of the chain $Y_{t}$.  
	\end{customassumption}
	Note that the polynomial mixing condition in \ref{assumption:A1'} is automatically satisfied when $\Omega$ is finite due to \eqref{eq:finite_MC_convergence_thm}. Polynomial mixing rate is available in most MCMC algorithms used in practice. 
\end{remark}

{\color{black}
	\subsection{Markov chain Monte Carlo Sampling}
	\label{subsection:MCMC}
	
	Suppose we have a finite sample space $\Omega$ and probability distribution $\pi$ on it. We would like to sample a random element $\omega\in \Omega$ according to the distribution $\pi$. \textit{Markov chain Monte Carlo (MCMC)} is a sampling algorithm that leverages the properties of Markov chains we mentioned in Subsection \ref{subsection:MC}. Namely, suppose that we have found a Markov chain $(X_{t})_{t\ge 0}$ on state space $\Omega$ that is irreducible, aperiodic\footnote{Aperiodicity can be easily obtained by making a given Markov chain lazy, that is, adding a small probability $\eps$ of staying at the current state. Note that this is the same as replacing the transition matrix $P$ by $P_{\eps}:=(1-\eps)P+\eps I$ for some $\eps>0$. This `lazyfication' does not change stationary distributions, as $\pi P =\pi$ implies $\pi P_{\eps}=\pi$. }, and has $\pi$ as its unique stationary distribution. Denote its transition matrix as $P$. Then by \eqref{eq:finite_MC_convergence_thm}, for any $\eps>0$, one can find a constant $\tau=\tau(\eps)=O(\log \eps^{-1})$ such that the conditional distribution of
	$X_{t+\tau}$ given $X_{t}$ is within total variation distance $\eps$ from $\pi$ regardless of the distribution of $X_{t}$. Recall such $\tau=\tau(\eps)$ is called the \textit{mixing time} of the Markov chain  $(X_{t})_{t\ge 1}$. Then if one samples a long Markov chain trajectory $(X_{t})_{t\ge 1}$, the subsequence $(X_{k\tau})_{k\ge 1}$ gives approximate i.i.d. samples from $\pi$. 
	
	We can further compute how far the thinned sequence $(X_{k\tau })_{k\ge 1}$ is away from being independent. Namely, observe that  for any two nonempty subsets $A,B\subseteq \Omega$, 
	\begin{align}
		&\left|   \P(X_{k\tau}\in A,\, X_{\tau}\in B) - \P(X_{k\tau}\in A) \P(X_{\tau}\in B) \right| \\
		&\qquad =   \left|  \P(X_{k\tau}\in A) - \P(X_{k\tau}\in A \,|\,  X_{\tau}\in B)  \right| \, \left| \P(X_{\tau}\in B) \right| \\
		&\qquad \le \left| \P(X_{k\tau}\in A) - \P(X_{k\tau}\in A\,|\, X_{\tau}\in B) \right|  \\
		&\qquad \le \left| \P(X_{k\tau}\in A) - \pi(A) \right| + \left| \pi(A) -  \P(X_{k\tau}\in A\,|\, X_{\tau}\in B)  \right| \le \lambda^{k\tau} + \lambda^{(k-1)\tau}.
	\end{align}
	Hence the correlation between $X_{k\tau}$ and $X_{\tau}$ is $O(\lambda^{(k-1)\tau})$.
	
	For the lower bound, let us assume that $X_{t}$ is \textit{reversible} with respect to $\pi$, that is, $\pi(x)P(x,y)=\pi(y)P(y,x)$ for $x,y\in \Omega$ (e.g., random walk on graphs). Then $\tau(\eps) = \Theta(\log \eps^{-1})$ (see \cite[Thm. 12.5]{levin2017markov}), which yields $\sup_{x\in \Omega}\lVert P^{t}(x,\cdot) - \pi  \rVert_{TV} = \Theta(\lambda^{t})$. Also, $\P(X_{\tau}\in B)>\delta>0$ for some $\delta>0$ whenever $\tau$ is large enough under the hypothesis. Hence 
	\begin{align}
		&\left|   \P(X_{k\tau}\in A,\, X_{\tau}\in B) - \P(X_{k\tau}\in A) \P(X_{\tau}\in B) \right| \\
		&\qquad \ge \delta^{-1}\left| \P(X_{k\tau}\in A) - \P(X_{k\tau}\in A\,|\, X_{\tau}\in B) \right|  \\
		&\qquad \ge \big| \left| \P(X_{k\tau}\in A) - \pi(A) \right| - \left| \pi(A) -  \P(X_{k\tau}\in A\,|\, X_{\tau}\in B)  \right| \big| \ge c\lambda^{(k-1)\tau}
	\end{align}
	for some constant $c>0$. Hence the correlation between $X_{k\tau}$ and $X_{\tau}$ is $\Theta(\lambda^{(k-1)\tau})$. In particular, the correlation between two consecutive terms in $(X_{k\tau})_{k\ge 1}$ is of $\Theta(\lambda^{\tau})=\Theta(\eps)$. Thus, we can make the thinned sequence $(X_{k\tau})_{k\ge 1}$ arbitrarily close to being i.i.d. for $\pi$, but if $X_{t}$ is reversible with respect to $\pi$, the correlation within the thinned sequence is always nonzero. 
	
	In practice, one may not know how to estimate the mixing time $\tau=\tau(\eps)$. In order to empirically assess that the Markov chain has mixed to the stationary distribution, multiple chains are run for diverse mode exploration, and their empirical distribution is compared to the stationary distribution (a.k.a. multistart heuristic). See \cite{brooks2011handbook} for more details on MCMC sampling. 
	
}

\section{Auxiliary lemmas}
\label{sec:auxiliary_lemmas}

\begin{lemma}[Convex Surrogate for Functions with Lipschitz Gradient]
	\label{lem:surrogate_L_gradient}
	Let $f:\R^{p}\rightarrow \R$ be differentiable and $\nabla f$ be $L$-Lipschitz continuous. Then for each $\theta,\theta'\in \R^{p}$, 
	\begin{align}
		\left| f(\theta') - f(\theta) - \nabla f(\theta)^{T} (\theta'-\theta) \right|\le \frac{L}{2} \lVert \theta-\theta'\rVert_{F}^{2}.
	\end{align}
\end{lemma}

\begin{proof}
	This is a classical Lemma. See \cite[Lem 1.2.3]{nesterov1998introductory}.
\end{proof}

For each $\mathcal{X}\in \R^{I_{1}\times \dots \times I_{n}\times b}$ and $\mathcal{D}\in \R^{I_{1}\times R} \times \dots \R^{I_{n}\times R}$, denote
\begin{align}
	H^{\star}(\mathcal{X},\mathcal{D})\in \argmin_{H\in \mathcal{C}^{\textup{code}}}\,\,\ell(\mathcal{X},\mathcal{D},H).
\end{align}
Recall Assumption \ref{assumption:A1'}. For each subset $S$ of a Euclidean space, denote $\lVert S \rVert_{F}=\sup_{x\in S} \lVert x \rVert_{F}$. The following boundedness results for the codes $H_{t}$ and aggregate tensors $A_{t},\commHL{\B}_{t}$ are easy to derive. 

\begin{lemma}\label{lem:H_bdd}
	Assume \ref{assumption:A1'} and \ref{assumption:A2}. Then the following hold:
	\begin{description}
		\item[(i)] For all $\X \in \R^{I_{1}\times \dots \times I_{n}\times b}$ and $\mathcal{D}\in \mathcal{C}^{\textup{dict}}$,
		\begin{align}
			\lVert H^{\star}(\X, \mathcal{D}) \rVert_{F}^{2} \le \lambda^{-2} \lVert \varphi(\Omega) \rVert_{F}^{4} <\infty.
		\end{align} 
		\item[(ii)] For any sequences $(\X_{t})_{t\ge 1}$ in $\R^{I_{1}\times \dots \times I_{n}\times b}$ and $(\mathcal{D}_{t})_{t\ge 1}$ in $\mathcal{C}$, define $A_{t}$ and $\commHL{\B}_{t}$ recursively as in Algorithm \ref{alg:online NTF_highlevel}. Then for all $t\ge 1$, we have 
		\begin{align}
			\lVert A_{t} \rVert_{F} \le \lambda^{-2} \lVert \varphi(\Omega) \rVert_{F}^{4}, \qquad \lVert \commHL{\B}_{t} \rVert_{F} \le \lambda^{-1} \lVert \varphi(\Omega) \rVert_{F}^{3}.
		\end{align}
	\end{description}
\end{lemma}

\begin{proof}
	Omitted. See \cite[Prop. 7.2]{lyu2020online}.
\end{proof}

The following lemma shows Lipschitz continuity of the loss function $\ell(\varphi(\cdot),\cdot)$ defined in \eqref{eq:def_loss}. Since $\Omega$ and $\mathcal{C}^{\textup{code}}$ are both compact, this also implies that $\mathcal{D}\mapsto \hat{f}_{t}(\mathcal{D})$ and $\mathcal{D}\mapsto f_{t}(\mathcal{D})$ are $L$-Lipschitz for some $L>0$ uniformly for all $t\ge 0$.

\begin{lemma}\label{lem:loss_Lipschitz}
	Suppose \ref{assumption:A1'} and \ref{assumption:A2} hold, and let $M=2\lVert \varphi(\Omega) \Vert_{F} +2\lVert \mathcal{C} ^{\textup{dict}}\rVert_{F} \lVert \varphi(\Omega)\rVert_{F}^{2}/\lambda$.  Then for each $Y_{1},Y_{2}\in \Omega$  and $\mathcal{D}_{1},\mathcal{D}_{2}\in \mathcal{C}^{\textup{dict}}$, 
	\begin{align}
		| \ell(\varphi(Y_{1}), \mathcal{D}_{1}) - \ell(\varphi(Y_{2}) ,\mathcal{D}_{2}) | \le  M\left( \lVert Y_{1}-Y_{2} \rVert_{F} + \lambda^{-1} \lVert \varphi(\Omega) \rVert_{F} \lVert \mathcal{D}_{1} - \mathcal{D}_{2} \rVert_{F} \right).
	\end{align}
\end{lemma}

\begin{proof}
	Omitted. See \cite[Prop. 7.3]{lyu2020online}.
\end{proof}

The following deterministic statement on converging sequences is due to \citep{mairal2010online}.

\begin{lemma}\label{lem:positive_convergence_lemma}
	Let $(a_{n})_{n\ge 0}$ and $(b_{n})_{\ge 0}$ be non-negative real sequences such that 
	\begin{align}
		\sum_{n=0}^{\infty} a_{n} = \infty, \qquad \sum_{n=0}^{\infty} a_{n}b_{n} <\infty, \qquad |b_{n+1}-b_{n}|=O(a_{n}).
	\end{align}
	Then $\lim_{n\rightarrow \infty}b_{n} = 0$. 
\end{lemma}

\begin{proof}
	Omitted. See \cite[Lem. A.5]{mairal2013stochastic}.
\end{proof}

\begin{lemma}\label{lem:uniform_convergence_symmetric_weights}
	Under the assumptions \ref{assumption:A1'} and \ref{assumption:A2},
	\begin{align}
		\E\left[ \sup_{W\in \mathcal{C}^{\textup{dict}}} \sqrt{t}\left| f(\mathcal{D}) - \frac{1}{t} \sum_{s=1}^{t} \ell(\mathcal{X}_{s},\mathcal{D})\right| \right] = O(1).
	\end{align}
	Furthermore, $\sup_{W\in \mathcal{C}} \left| f(\mathcal{D}) - \frac{1}{t} \sum_{s=1}^{t} \ell(\mathcal{X}_{s},\mathcal{D})\right|\rightarrow 0$ almost surely as $t\rightarrow \infty$. 
\end{lemma}

\begin{proof}
	Omitted. See \cite[Lem. 7.8]{lyu2020online}.
\end{proof}

In the following lemma, we generalize the uniform convergence results in Lemma \ref{lem:uniform_convergence_symmetric_weights} for general weights $w_{t}\in (0,1)$ (not only for the `balanced weights' $w_{t}=1/t$). The original lemma is due to Mairal \cite[Lem B.7]{mairal2013stochastic}, which originally extended the uniform convergence result to weighted empirical loss functions with respect to i.i.d. input signals. A similar argument gives the corresponding result in our Markovian case \ref{assumption:A1'}, which was also used in \cite{lyu2020online}. \commHL{In Lemma \ref{lem:uniform_convergence_symmetric_weights} below, we also generalize the statement for the weights $w_{t}$ satisfy the required monotonicity property $w_{t+1}^{-1}-w_{t}^{-1}\le 1$ only asymptotically. See Remark \ref{rmk:weight} for more discussion. }

\begin{lemma}\label{lem:uniform_convergence_asymmetric_weights}
	Suppose  \ref{assumption:A1'}-\ref{assumption:A2} hold, \commHL{and assume that there exist an integer $T\ge 1$ such that  $w_{t+1}^{-1}-w_{t}^{-1}\le 1$ for all $t\ge T$. Also assume that there are some constants $c>0$ and $\gamma\in (0,1]$ such that $w_{t}\ge ct^{-\gamma}$ for all $t\ge 1$. Further assume that if $T\ge 1$ and $\gamma=1$, then $c\ge 1/2$. Then there exists a constant $C=C(T)>0$} such that 
	\begin{align}\label{eq:uniform_CLT_bd}
		\E\left[ \sup_{W\in \mathcal{C}^{\textup{dict}}} \left| f(\mathcal{D}) - f_{t}(\mathcal{D})\right| \right] \le C w_{t} \sqrt{t}.
	\end{align}
	Furthermore, if $\sum_{t=1}^{\infty} w_{t} =\infty$,  $\sum_{t=1}^{\infty}w_{t}^{2}\sqrt{t}<\infty$,   then $\sup_{\commHL{\mathcal{D}\in \mathcal{C}^{\textup{dict}}}} \left| f(\mathcal{D}) - f_{t}(\mathcal{D})\right|\rightarrow 0$ almost surely as $t\rightarrow \infty$. 
\end{lemma}

\begin{proof}
	Fix $t\in \mathbb{N}$. Recall the weighted empirical loss $f_{t}(\mathcal{D})$ defined recursively using the weights $(w_{s})_{s\ge 0}$ in \eqref{eq:def_OCPDL_ELM}. For each $0\le s \le t$, denote $w_{s}^{t} = w_{s}\prod_{j=s}^{t}(1-w_{j})$ and set $w^{t}_{0}=0$. Then for each $t\in \mathbb{N}$, we can write $f_{t}(\mathcal{D}) = \sum_{s=1}^{t} \ell(X_{s},W) w_{s}^{t}$. Moreover, note that $w_{1}^{t},\dots,w_{t}^{t}>0$ and $w_{1}^{t}+\dots+w_{t}^{t}=1$. Define $F_{i}(\mathcal{D}) = (t-i+1)^{-1}\sum_{j=1}^{t}\ell(X_{i}, W)$ for each $1\le i \le t$. By Lemma \ref{lem:uniform_convergence_symmetric_weights}, there exists a constant $c_{1}>0$ such that 
	\begin{align}\label{eq:pf_uniform_convergnece_weighted}
		\E\left[ \sup_{W\in \mathcal{C}} |F_{i}(\mathcal{D}) - f(\mathcal{D})| \right] \le \frac{c_{1}}{\sqrt{t-i+1}}
	\end{align} 
	for all $t \ge 1$ and $1\le i \le t$.  Noting that $[w^{t}_{1},\dots,w^{t}_{t}]$ is a probability distribution on $\{1,\dots,t\}$, a simple calculation shows the following important identity 
	\begin{align}
		f_{t} - f = \sum_{i=1}^{t} (w^{t}_{i} - w^{t}_{i-1}) (t-i+1) (F_{i} - f),
	\end{align} 
	with the convention of $w^{t}_{0}=0$. Also, suppose $T\ge 1$ is such that $w_{k}^{-1} - w_{k-1}^{-1} \le 1$ for $k\ge T$. Note that for $i\ge 2$, $w^{t}_{i-1} \le w^{t}_{i}$ if and only if $w_{i-1}(1-w_{i}) \le w_{i}$ if and only if $w_{i}^{-1}-w_{i-1}^{-1}\le 1$. Hence for each $n > T$ and $k\ge T$, we have $w^{t}_{k}\le w^{t}_{k+1}\le \dots \le w^{t}_{t}=w_{t}$. Then observe that 
	\begin{align}\label{eq:pf_uniform_convergnece_weighted2}
		\E\left[ \sup_{\param \in \Param} |f_{t}(\param) - \bar{\psi}(\param)| \right] &\le \E\left[ \sum_{i=1}^{t} |w^{t}_{i} - w^{t}_{i-1}| (t-i+1) \sup_{W\in \Param} \left|f_{i}(\param) - f(\param) \right|  \right] \\
		&= \sum_{i=1}^{t} |w^{t}_{i} - w^{t}_{i-1}|  (t-i+1) \,  \E\left[ \sup_{\param\in \Param} \left|f_{i}(\param) - f(\param) \right|  \right] \\
		&\le  \sum_{i=1}^{t} |w^{t}_{i} - w^{t}_{i-1}|  c_{1}\sqrt{t-i+1} \\
		&\le  c_{1}\sqrt{t} \left( \sum_{i=1}^{T} |w_{i}^{t} -w_{i-1}^{t} | +  \sum_{i=T}^{t} (\hat{w}^{t}_{i} - \hat{w}^{t}_{i-1}) \right)  \\
		&\le c_{1}\sqrt{t}  \left( w_{t} +  \sum_{i=1}^{T} w_{i}^{t}  \right).
	\end{align} 
	By using Lemma \ref{lem:weight_sum_bound}, we have $\sum_{i=1}^{T} w_{i}^{t} =O(1/n)$. Furthermore, since $w^{-1}_{t+1}- w^{-1}_{t}\le 1$ for all $t\ge T$, we deduce $w_{t}^{-1} - w_{T}^{-1} \le t-T$ for all $t\ge T$, so $w_{t} \ge \frac{1}{t-T+w_{T}^{-1}}$. Thus for some constant $c>0$, we have $w_{t} \ge \frac{1}{t+c}$ for all $t\ge T$. Thus the last displayed expression above is of $O(w_{t}\sqrt{t})$. This shows \eqref{eq:uniform_CLT_bd}. We can show the part by using Lemma \ref{lem:positive_convergence_lemma} \commHL{in Appendix \ref{sec:auxiliary_lemmas}}, following the argument in the proof of \citep[Lem. B7]{mairal2013stochastic}. See the reference for more details. 
\end{proof}

The following lemma was used in the proof of Lemma \ref{lem:uniform_convergence_asymmetric_weights}.

\begin{lemma}\label{lem:weight_sum_bound}
	Fix a sequence $(w_{n})_{n\ge 1}$ of numbers in $(0,1]$. Denote $w^{n}_{k}:=w_{k}\prod_{i=k+1}^{n} (1-w_{i})$ for $1\le k \le n$. Suppose $w_{n}^{-1} - w_{n-1}^{-1}\le 1$ for all sufficiently large $n\ge 1$.  Fix $T\ge 1$. Then for all $n\ge T$, 
	\begin{align}\label{eq:weight_sum_bound1}
		\sum_{i=1}^{T} w^{n}_{i}  = O(1/n).
	\end{align}
\end{lemma}

\begin{proof}
	Suppose $w_{n}^{-1} - w_{n-1}^{-1}\le 1$ for all $n\ge N$ for some $N\ge 1$. It follows that $w_{n}^{-1} - w_{N}^{-1} \le n-N$, so $w_{n} \ge \frac{1}{n-N+w_{N}^{-1}}$. Hence for some constant $c>0$, $w_{n}\ge \frac{1}{n+c}$ for all $n\ge N$. Denote $a\lor b = \max(a,b)$. Then note that 
	\begin{align}
		w^{n}_{k}=w_{k} \exp\left(  \sum_{i=k+1}^{n} \log (1-w_{i}) \right)& \le \exp\left( -\sum_{i=k+1}^{n} w_{i} \right) \\
		&\le \exp\left( -\int_{N\lor (k+1)}^{n} \frac{1}{x+c}\,dx \right) =  \frac{[N\lor (k+1)]+c }{n+c},
	\end{align}
	where the second inequality uses $w_{k}\le 1$ and the following inequality uses $\log(1-a)\le -a$ for $a<1$.  Hence for each fixed $1\le T\le n$, we have 
	\begin{align}
		\sum_{k=1}^{T} w^{n}_{k} \le T\left( (N\lor (T+1))  +c\right)\frac{1}{n+c} . 
	\end{align}
	This shows the assertion. 
\end{proof}

\begin{remark}\label{rmk:weight}
	\normalfont
	In the original statement of \cite[Lem B.7]{mairal2013stochastic}, the assumption that $w_{t+1}^{-1}-w_{t}^{-1}\le 1$ for sufficiently large $t$ was not used, and it seems that the argument in \cite{mairal2013stochastic} needs this assumption. To give more detail, the argument begins with writing the empirical loss $f_{t}(\cdot) = \sum_{k=1}^{t} w^{t}_{k} \, \ell(\mathcal{X}_{k},\cdot)$, where $w^{t}_{k}:=w_{k}(1-w_{k-1})\cdots (1-w_{t})$, and proceeds with assuming the monotonicity $w^{t}_{1}\le \dots \le w^{t}_{t}$, which is equivalent to $w_{k}\ge w_{k-1}(1-w_{k})$ for $2\le k \le t$. In turn, this is equivalent to $w_{k}^{-1}-w_{k-1}^{-1}\le 1$ for $2\le k \le t$. Note that this condition implies $w_{k}^{-1} \le (k-1) + w_{1}^{-1}$, or $w_{k} \ge \frac{1}{k-1 + w_{1}^{-1}}$, where $w_{1}\in [0,1]$ is a fixed constant. This means that, asymptotically, $w_{k}$ cannot decay faster than the balanced weight $1/k$, which gives $w^{t}_{k}\equiv 1/t$ for $k\in \{1,\dots,t\}$. Note that we proved Lemma \ref{lem:weight_sum_bound} with requiring $w_{t+1}^{-1}-w_{t}^{-1}\le 1$ for all sufficiently large $t$. 
	
	Next, we will argue that \ref{assumption:A3'} implies \ref{assumption:A3}. It is clear that if the sequence $w_{t}\in (0,1]$ satisfies \ref{assumption:A3'}, then $\sum_{t=1}^{\infty} w_{t}=\infty$ and $\sum_{t=1}^{\infty} w_{t}^{2}\sqrt{t}<\infty$. So it remains to verify $w_{t}^{-1}-w_{t-1}^{-1}\le 1$ for sufficiently large $t$. Suppose $w_{t}=\Theta(t^{-\beta} (\log t)^{-\delta} )$ for some $\beta\in [0,1]$ and $\delta\ge 0$. Let $c_{1},c_{2}>0$ be constants such that $w_{t}t^{\beta} (\log t)^{\delta} \in [c_{1},c_{2}] $ for all $t\ge 1$. Then by the mean value theorem, 
	\begin{align}
		w_{t+1}^{-1} - w_{t}^{-1} & \le c_{2} \left( (t+1)^{\beta}(\log (t+1))^{\delta} - t^{\beta}(\log t)^{\delta} \right)  \\
		&\le c_{2}\sup_{t\le s \le t+1} \left( \beta s^{\beta-1}(\log s)^{\delta} + \delta s^{\beta-1} (\log s)^{\delta-1}  \right)\\
		&\le c_{2}\sup_{t\le s \le t+1}  s^{\beta-1}(\log s)^{\delta-1} \left( (\log s)+\delta \right).
	\end{align}
	Since $t\ge 1$, the last expression is of $o(1)$ if $\beta<1$. Otherwise, $w_{t}=t^{-1}$ for $t\ge 1$ by \ref{assumption:A3'}. Then $w_{t+1}^{-1} - w_{t}^{-1} \equiv 1$ for all $t\ge 1$.
\end{remark}

\section{Bounded memory implementation of Algorithm \ref{alg:online NTF_highlevel}}
\label{section:statement_alg_bounded_memeory}

In this section, we introduce an alternative implementation of Algorithm \ref{alg:online NTF_highlevel} that uses bounded memory that is independent of the number $T$ of minibatches of data tensors being processed. This will be done by replacing the step for computing the surrogate loss function $\hat{f}_{t}$ with computing two `aggregate tensors' based on our deterministic analysis in Proposition \ref{prop:g_t_derivation}. The total amount of information fed in to the algorithm is $O(T\prod_{i=1}^{n}I_{n})$ and $T\rightarrow \infty$, whereas Algorithm \ref{algorithm:online NTF} stores only $O(R\prod_{i=1}^{n}I_{n})$ (recall that $R$ is the number of dictionary atoms to be learned and $T$ is the number of minibatches of data tensors that have arrived). This is an inherent memory efficiency of online algorithms against non-online algorithms (see, e.g., \cite{mairal2010online}).

\begin{algorithm}
	\caption{Online CP-Dictionary Learning (Bounded Memory Implementation)}
	\label{algorithm:online NTF}
	\begin{algorithmic}[1]
		\State \textbf{Input:} $(\mathcal{X}_{t})_{1\le t\le T}$ (minibatches of data tensors in $\R_{\ge 0}^{I_{1}\times \dots \times I_{n}\times b}$);  $[U_{0}^{(1)},\dots,U_{0}^{(n)}]\in \R_{\ge 0}^{I_{1}\times R} \times \dots \times \R_{\ge 0}^{I_{n}\times R}$ (initial loading matrices); $c'>0$ (search radius constant);
		\vspace{0.1cm} 
		\State \textbf{Constraints:}  $\mathcal{C}^{(i)}\subseteq \R^{I_{i}\times R}$, $1\le i\le n$, $\mathcal{C}^{\textup{code}}\subseteq \R^{R\times b}$ (e.g., nonnegativity constraints)

		\State \textbf{Parameters:} $R\in \mathbb{N}$ ($\#$ of dictionary atoms);\, $\lambda\ge 0$ ($\ell_{1}$-regularizer); \, $(w_{t})_{t\ge 1}$ (weights in $(0,1]$);\, 
		
		\vspace{0.1cm}
		\State \quad Initialize aggregate tensors  $A_{0}\in \R^{R\times R}$, $\commHL{\B}_{0} \in \R^{I_{1}\times \dots \times I_{n}\times R}$;
		
		\State \quad \textbf{For $t=1,\ldots,T$ do:}
		\State  \quad \quad \textit{Coding}: Compute the optimal code matrix
		\begin{align} \label{eq:coding_bdd_memeory_alg}
			\hspace{2cm} H_{t}\leftarrow \argmin_{H \in \mathcal{C}^{\textup{code}} \subseteq \R^{R \times b} } \,\, \ell(\mathcal{X}_{t}, U_{t-1}^{(1)},\dots,U_{t-1}^{(n)},H); \quad \text{(using Algorithm \ref{algorithm:spaser_coding})}
		\end{align}

		\State  \quad \quad \textit{Update aggregate tensors}:
		\vspace{-0.1cm}
		\begin{align}\label{eq:main_alg_aggregate_tensors}
			A_{t} &\leftarrow (1-w_{t}) A_{t-1} + w_{t} H_{t}H_{t}^{T} \in \R^{R\times R}; \\
			\commHL{\B}_{t} &\leftarrow (1-w_{t}) \commHL{\B}_{t-1} + w_{t} (\mathcal{X}_{t} \times_{n+1} H_{t}^{T}) \in \R^{I_{1}\times\dots \times I_{n} \times R};
		\end{align}

		\State  \quad \quad \textit{Update dictionary:} 
		\vspace{0.1cm}
		\State  \quad \quad \quad \textbf{For $i = 1, \ldots, n$ do:} 
		\vspace{0.1cm}
		\State  \quad \quad \qquad $\overline{A}_{t;i} \in \R^{R\times R}$, $\overline{B}_{t;i}\in \R^{I_{i}\times R}$ \\ \hspace{2cm} $\leftarrow$ Algorithm \ref{algorithm:inter_aggregation} with input $A_{t}, \commHL{\B}_{t}, U_{t}^{(1)},\dots,U_{t}^{(i-1)},U^{(i)}_{t},U_{t-1}^{(i+1)},\dots,U_{t-1}^{(n)}, i$;
		\vspace{0.1cm}
		\State \quad \quad \qquad $\mathcal{C}_{t}^{(i)}\leftarrow \left\{U \in \mathcal{C}^{(i)}\,\bigg|\,   \lVert U-U_{t-1}^{(i)}\rVert_{F}\le c'w_{t}\right\}$; \qquad (Restrict the search radius by $w_{t}$) \label{eq:Alg1_search_restriction_appendix}
		\vspace{0.1cm}
		\State  \quad \quad \qquad $U_{t}^{(i)} \leftarrow  \argmin_{U \in \mathcal{C}_{t}^{(i)}} \left[  \tr(U \overline{A}_{t;i} U^{T})  - 2\tr(U \, \overline{B}_{t;i}^{T}) \right]; \qquad \text{(Using Algorithm \ref{algorithm:dictionary_update})}$ \label{eq:dict_update_OCPDL}
		
		\State  \quad \quad \quad \textbf{End for} 
		\State \quad\textbf{End for} 
		\State \textbf{Return:}  $[U_{T}^{(1)},\dots, U_{T}^{(n)}] \in \mathcal{C}^{(1)}\times \dots \times \mathcal{C}^{(n)}$;
	\end{algorithmic}
\end{algorithm}

\begin{algorithm}
	\caption{Intermediate Aggregation}
	\label{algorithm:inter_aggregation}
	\begin{algorithmic}[1]
		\State \textbf{Input:} $A\in \R^{R\times R}$, $B \in \R^{I_{1}\times \dots \times  I_{n} \times R}$, $[U_{1},\ldots,U_{n}]\in \R^{I_{1}\times R} \times \ldots \times \R^{I_{n}\times R} $, $1\le j\le n$
		
		\State \textbf{Do:}  
		\begin{align}\label{eq:def_Atj}
			\overline{A}_{i} &= A \had U_1^{T} U_1 \had \ldots  \had U_{i - 1}^{T} U_{i - 1} \had U_{i + 1}^{T} U_{i + 1} \had \ldots \had U_n^{T} U_n   \in \R^{R\times R}
		\end{align}
		\State \quad  \textbf{For $r = 1, \ldots, R$ do:}  
		\begin{align}
			B(,r)&:= \text{mode-$(n+1)$ slice of $B$ at coordinate $r$}\\
			b_{i;r} &= B(,r) \times_1 U_1(:, r) \times_2 \dots \times_{i - 1} U_{i - 1}(:, r) \times_{i + 1}U_{i + 1}(:, r) \times_{i + 2} \dots \times_{n} U_n(:, r)  \in \R^{I_{i}} \label{eq:def_Btrj} \\
			\overline{B}_{t;\commHL{i}} &= \text{$I_{i} \times R$ matrix whose $r$th column is $b_{i;r}$}
		\end{align}
		\State \quad  \textbf{End for}
		
		\State \textbf{Return:} $$\overline{A}_{i}=\overline{A}_{i}(A, U_{1},\dots, U_{i - 1}, U_{i + 1}, \dots, U_{n})$$  $$\overline{B}_{i}=\overline{B}_{i}(B, U_{1},\dots, U_{i - 1}, U_{i + 1}, \dots, U_{n})$$
	\end{algorithmic}
\end{algorithm}

We describe how Algorithm \ref{algorithm:online NTF} is derived and why it is equivalent to Algorithm \ref{alg:online NTF_highlevel}. By the time that the new data tensor $\mathcal{X}_{t}$ arrives, the algorithm have computed previous loading matrices $U_{t-1}^{(1)},\dots,U_{t-1}^{(n)}$ and two aggregate tensors $A_{t-1}\in \R^{R\times R}$ and $\commHL{\B}_{t-1}\in \R^{I_{1}\times \cdots \times I_{n}\times \commHL{R}}$. Then one computes the code matrix $H_{t}\in \mathcal{C}^{\textup{code}}\subseteq \R^{R\times b}$ by solving the convex optimization problem in \eqref{eq:coding_bdd_memeory_alg}, and then updates the aggregate tensors $A_{t}\leftarrow A_{t-1}$ and $\commHL{\B}_{t}\leftarrow \commHL{\B}_{t-1}$. In order to perform the block coordinate descent to update the loading matrices $U_{i}^{(t)}$ in \cref{eq:Alg1_loading_update} of Algorithm \ref{alg:online NTF_highlevel}, we appropriately recompute intermediate aggregate matrices $\overline{A}_{i}$ and $\overline{B}_{i}$ using Algorithm \ref{algorithm:inter_aggregation} so that we are correctly minimizing the surrogate loss function $\hat{f}_{t}$ in \eqref{eq:scheme_online NTF_surrogate_full} marginally according to Proposition \ref{prop:g_t_derivation} \textbf{(ii)}.

\section{Auxiliary Algorithms}
\label{sec:sec:auxiliary_algorithms}

In this section, we give auxiliary algorithms that are used to solve convex sub-problems in coding and loading matrix updates for the main algorithm (Algorithm 1 for online CPDL). We denote by $\Pi_{S}$ the projection operator onto the given subset $S$ defined on the respective ambient space. For each matrix $A$, denote by $[A]_{\bullet i}$ (resp., $[A]_{i\bullet}$) the $i$th column (resp., row) of $A$.  

\begin{algorithm}
	\caption{Coding}
	\label{algorithm:spaser_coding}
	\begin{algorithmic}[1]
		\State \textbf{Input:} $X \in  \mathbb{R}^{M\times b}$: data matrix, $W\in \mathbb{R}^{M\times R}$: dictionary matrix 
		\State \qquad $\lambda\ge 0$: sparsity regularizer 
		\State \qquad $\mathcal{C}^{\textup{code}}\subseteq \mathbb{R}^{R\times b}$: constraint set of codes
		\State \textbf{Repeat until convergence:}
		\State \qquad \textbf{Do}
		\begin{align}\label{eq:algorithm_H}	
			H \leftarrow \Pi_{\mathcal{C}^{\textup{code}}}\left( H - \frac{1}{\tr(W^{T}W)}(W^{T}W H - W^{T}X + \lambda J)  \right),
		\end{align}			
		\qquad where $J\subseteq \mathbb{R}^{R\times b}$ is all ones matrix.
		
		\State \textbf{Return} $H\in \mathcal{C}^{\textup{code}}\subseteq \R^{R\times b}$
	\end{algorithmic}
\end{algorithm}

\begin{algorithm}
	\caption{Loading matrix update}
	\label{algorithm:dictionary_update}
	\begin{algorithmic}[1]
		\State \textbf{Variables:} 
		\State \qquad $U\in \mathcal{C}^{(i)} \subseteq \R^{I_{i}\times R}$: previous $j$th loading matrix
		\State  \qquad $(\overline{A}_{i}, \overline{B}_{t;j})\in \R^{R\times R}\times \R^{R\times (I_{1}\dots I_{n})}$: intermediate loading matrices computed previously

		\State \textbf{Repeat until convergence:} 
		\State \qquad \textbf{For $i=1$ to $R$:}
		\begin{align}\label{eq:dictioanry_column_update}
			[U]_{\bullet i} \leftarrow \Pi_{\mathcal{C}^{(i)}} \left( [U]_{\bullet i}  - \frac{1}{[\overline{A}_{t}]_{ii}+1} (U [\overline{A}_{i}]_{\bullet i} - [\overline{B}_{t;j}^{T}]_{\bullet i} ) \right)
		\end{align}
		\State \textbf{Return} $U\in \mathcal{C}^{(i)}\subseteq \R^{I_{i}\times R}$
	\end{algorithmic}
\end{algorithm}

\begin{algorithm}[H]
	\small
	\caption{Alternating Least Squares for NCPD}
	\label{algorithm:ALS}
	\begin{algorithmic}[1]
		\State \textbf{Input:} $\mathbf{X}\in \R^{I_{1}\times \cdots \times I_{m}}_{\ge 0}$ (data tensor);\,  $R\in \mathbb{N}$ (rank parameter);\,  $\param_{0}=(U_{0}^{(1)},\cdots,U_{0}^{(m)})\in  \R^{I_{1}\times R}_{\ge 0} \times \cdots \times  \R^{I_{m}\times R}_{\ge 0}$ (initial loading matrices); $N$ (number of iterations); 
		\vspace{0.2cm}
		\State \quad \textbf{for} $n=1,\dots,N$ \textbf{do}: \State \quad \quad Update loading matrices $\param_{n}=[U_{n}^{(1)},\cdots, U_{n}^{(m)}]$ by  
		\State \quad \quad \quad \textbf{For} $i=1,\cdots,m$ \textbf{do}:
		\quad 
		\begin{align}\label{eq:ALS_DR_block}
			&\qquad \mathbf{A} \leftarrow \Out(U^{(1)}_{n-1}, \dots, U^{(i-1)}_{n-1}, U^{(i+1)}_{n-1},\dots, U^{(m-1)}_{n-1} )^{(m)} \in \R^{(I_{1}\times \cdots \times   I_{i-1}\times  I_{i+1} \times \cdots\times  I_{m}) \times R } \\
			&\qquad  B \leftarrow \textup{unfold}(\mathbf{A}, m) \in \R^{(I_{1}\cdots I_{i-1} I_{i+1} \cdots I_{m}) \times R } \\
			&\qquad  U_{n}^{(i)}\in \argmin_{U \in \R^{I_{i}}_{\ge 0} } \,  \lVert \textup{unfold}(\mathbf{X}, i) -  B (U^{(i)})^{T} \rVert^{2} \\
			&\qquad \triangleright\textup{($\textup{unfold}(\cdot, i)$ denotes the mode-$i$ tensor unfolding (see \cite{kolda2009tensor}))}
		\end{align}
		\State \quad \qquad \textbf{end for}
		\State \quad \textbf{end for}
		\State \textbf{output:}  $\param_{N}$ 
	\end{algorithmic}
\end{algorithm}

\begin{algorithm}[H]
	\small
	\caption{Multiplicative Update for NCPD}
	\label{algorithm:MU}
	\begin{algorithmic}[1]
		\State \textbf{Input:} $\mathbf{X}\in \R^{I_{1}\times \cdots \times I_{m}}_{\ge 0}$ (data tensor);\,  $R\in \mathbb{N}$ (rank parameter);\,  $\param_{0}=(U_{0}^{(1)},\cdots,U_{0}^{(m)})\in  \R^{I_{1}\times R}_{\ge 0} \times \cdots \times  \R^{I_{m}\times R}_{\ge 0}$ (initial loading matrices); $N$ (number of iterations); 
		\vspace{0.2cm}
		\State \quad \textbf{for} $n=1,\dots,N$ \textbf{do}: \State \quad \quad Update loading matrices $\param_{n}=[U_{n}^{(1)},\cdots, U_{n}^{(m)}]$ by  
		\State \quad \quad \quad \textbf{For} $i=1,\cdots,m$ \textbf{do}:
		\quad 
		\vspace{-0.3cm}
		\begin{align}\label{eq:MU_block}
			&\qquad \mathbf{A} \leftarrow \Out(U^{(1)}_{n-1}, \dots, U^{(i-1)}_{n-1}, U^{(i+1)}_{n-1},\dots, U^{(m-1)}_{n-1} )^{(m)} \in \R^{(I_{1}\times \cdots \times   I_{i-1}\times  I_{i+1} \times \cdots\times  I_{m}) \times R } \\
			&\qquad  B \leftarrow \textup{unfold}(\mathbf{A}, m) \in \R^{(I_{1}\cdots I_{i-1} I_{i+1} \cdots I_{m}) \times R } \\
			&\qquad  U_{n}^{(i)} \leftarrow U_{n-1}^{(i)}\odot ( \textup{unfold}(\mathbf{X}, i))^{T} B \oslash (U B^{T} B) \\
			\vspace{0.2cm}
			&\qquad \triangleright\textup{ ($\textup{unfold}(\cdot, i)$ denotes the mode-$i$ tensor unfolding (see \cite{kolda2009tensor}))} \\
			&\qquad \triangleright \,\, \textup{($\odot$ and $\oslash$ denote entrywise product and division)}
		\end{align}
		\State \quad \qquad \textbf{end for}
		\State \quad \textbf{end for}
		\State \textbf{output:}  $\param_{N}$ 
	\end{algorithmic}
\end{algorithm}

\end{document}